\documentclass[twoside,11pt]{article}

\usepackage{jmlr2e}

\usepackage{lastpage}
\jmlrheading{25}{2024}{1-\pageref{LastPage}}{11/23; Revised
	4/24}{5/24}{23-1563}{Yuheng Ma and Hanfang Yang}
\ShortHeadings{Optimal Locally Private Nonparametric Classification with Public Data}{Ma and Yang}

\usepackage{float}

\usepackage{subfigure}
\usepackage{multirow}
\usepackage[all]{xy}
\usepackage{enumitem}
\usepackage{algorithmic} 
\setlist[enumerate]{itemsep=0mm}
\usepackage{color}
\usepackage{jmlr2e}
\usepackage{wrapfig}
\usepackage{graphics}
\usepackage{epsfig}
\usepackage[ruled]{algorithm2e}
\usepackage{booktabs}
\usepackage{url}
\usepackage{psfrag}
\usepackage{relsize}
\usepackage{amsmath,amsfonts,mathrsfs,mathtools,amssymb}
\usepackage{hyperref}
\hypersetup{
	colorlinks=true,
	linkcolor=blue,
	filecolor=magenta,      
	urlcolor=cyan,
	pdftitle={Overleaf Example},
	pdfpagemode=FullScreen,
	hidelinks
}
\usepackage{tikz,pgfplots}
\usepackage{tkz-tab}
\usepackage[format=hang,justification=justified,singlelinecheck=false,font=footnotesize]{caption}

\usepackage{tikz,pgfplots}
\pgfplotsset{compat=newest}
\newlength\figureheight
\newlength\figurewidth
\usetikzlibrary{shapes,arrows}
\usepackage[format=hang,justification=justified,singlelinecheck=true,font=footnotesize]{caption}
\usepackage{flowchart}

\newtheorem{assumption}{Assumption}
\newcommand{\eins}{\boldsymbol{1}}

\allowbreak

\begin{document}
	
	\title{Optimal Locally Private Nonparametric Classification \\ with Public Data}
	
	\author{\name Yuheng Ma \email yma@ruc.edu.cn \\
		\addr School of Statistics \\ 
		Renmin University of China \\
		100872 Beijing, China 
		\vspace{0.2cm}
		\\
		\name Hanfang Yang \email hyang@ruc.edu.cn \\
		\addr Center for Applied Statistics, School of Statistics \\ 
		Renmin University of China \\
		100872 Beijing, China 
		\vspace{0.2cm}
	}

	\editor{Po-Ling Loh}
	
	\maketitle

	\begin{abstract}
		
		In this work, we investigate the problem of public data assisted non-interactive Local Differentially Private (LDP) learning with a focus on non-parametric classification. 
		Under the posterior drift assumption, we for the first time derive the mini-max optimal convergence rate with LDP constraint.
		Then, we present a novel approach, the locally differentially private classification tree, which attains the mini-max optimal convergence rate. 
		Furthermore, we design a data-driven pruning procedure that avoids parameter tuning and provides a fast converging estimator.
		Comprehensive experiments conducted on synthetic and real data sets show the superior performance of our proposed methods.
		Both our theoretical and experimental findings demonstrate the effectiveness of public data compared to private data, which leads to practical suggestions for
		prioritizing non-private data collection.
		
	\end{abstract}
	

	\section{Introduction}\label{sec:introduction}

	Local differential privacy (LDP) \citep{kairouz2014extremal, duchi2018minimax}, which is a variant of differential privacy (DP) \citep{dwork2006calibrating}, has gained considerable attention in recent years, particularly among industrial developers \citep{erlingsson2014rappor, tang2017privacy}. 
	Unlike central DP, which relies on a trusted curator who has access to the raw data, LDP assumes that each sample is possessed by a data holder, and each holder privatizes their data before it is collected by the curator.
	Although this setting provides stronger protection, learning with perturbed data requires more samples \citep{duchi2018minimax} compared to the central setting. 
	Moreover, basic techniques such as principal component analysis \citep{wang2020principal, bie2022private}, data standardization \citep{bie2022private}, and tree partitioning \citep{wu2022ensemble} are troublesome or even prohibited. 
	Consequently, LDP introduces challenges for various machine learning tasks that are otherwise considered straightforward, including density estimation \citep{duchi2018minimax}, mean estimation \citep{duchi2018minimax}, Gaussian estimation \citep{joseph2019locally}, and change-point detection \citep{berrett2021locally}.

	Fortunately, in certain scenarios, private estimation performance can be enhanced with an additional public data set.
	From an empirical perspective, numerous studies demonstrated the effectiveness of using public data \citep{papernotsemi, papernot2018scalable, yu2021large, yu2022differentially, nasr2023effectively, gu2023choosing}.
	In most cases, the public data is out-of-distribution and collected from other related sources. 
	Yet, though public data from an identical distribution has been extensively studied \citep{bassily2018model, alon2019limits, kairouz2021nearly, amid2022public, ben2023private, lowy2023optimal, wang2023generalized}, only a few works systematically described the out-of-distribution relationship.
	Focusing on unsupervised problems, i.e. Gaussian estimation, \cite{bie2022private} discussed the gain of using public data with respect to the total variation distance between private and public data distributions.
	As for supervised learning,
	\cite{gu2023choosing} discussed the similarities between data sets to guide the selection of appropriate public data.
	\cite{ma2023decision} leveraged public data to create informative partitions for LDP regression trees. 
	Neither of them discussed the relationship between regression functions of private and public data distributions. 
	Aiming at the gap between theory and practice in supervised private learning with public data, we pose the following intriguing questions:
	\begin{align*}
		\textit{1. Theoretically,  when is labeled public data beneficial for LDP learning?}
	\end{align*}
	\begin{align*}
		\textit{2. Under such conditions, how to effectively leverage public data?}
	\end{align*}
	Hyper-parameter tuning is an essential yet challenging task in private learning \citep{papernot2021hyperparameter, mohapatra2022role} and remains an unsolved problem in local differential privacy. 
	Common strategies, such as cross-validation and information criteria, remain unavailable due to their requirement for multiple queries to the training data. 
	Tuning by omitting a validation set is restrictive in terms of the size of the potential hyperparameter space \citep{ma2023decision}.
	Moreover, leveraging public data necessitates more complex models, which often results in an increased number of hyperparameters to tune.
	Consequently, models with public data may face difficulties in parameter selection, which leads to the third consideration: 
	\begin{align*}
		\textit{3. Can data-driven hyperparameter tuning approaches be derived?}
	\end{align*}

	In this work, we answer the above questions from the perspective of non-parametric classification with the non-interactive LDP constraint. 
	Though parametric methods enjoy better privacy-utility trade-off \citep{duchi2018minimax} under certain assumptions, they are vulnerable to model misspecification. 
	Under LDP setting where we have limited prior knowledge of the data, non-parametric methods ensure the worst-case performance. 
	As for private classification, the local setting remains rarely explored compared to the central setting. 
	A notable reason is that most gradient-based methods \citep{song2013stochastic, abadi2016deep} are prohibited due to the high demand for memory and communication capacity on the terminal machine \citep{tramer2022considerations}. 
	We consider non-interactive LDP, where the communication between the curator and the data holders is limited to a single round. 
	This type of method is favored by practitioners \citep{smith2017interaction, zheng2017collect, daniely2019locally} since they avoid multi-round protocols that are prohibitively slow in practice due to network latency.

	Under such background, we summarize our contributions by answering the questions:
	\begin{itemize}
		\item For the first question, we propose to adopt the framework of \textit{posterior drift}, a setting in transfer learning \citep{cai2021transfer}, for our analysis. 
		Specifically, given the target distribution $\mathrm{P}$ and the source distribution $\mathrm{Q}$, we require $\mathrm{P}_X = \mathrm{Q}_X$, i.e. they have the same marginal distributions. 
		Moreover, their Bayes decision rule are identical, i.e. $\left(\mathrm{P}(Y = 1| X ) -{1}/{2}\right) \cdot \left(\mathrm{Q}(Y = 1| X) - {1}/{2}\right)\geq 0$.
		The assumption covers a wide range of distributions and includes $\mathrm{P} = \mathrm{Q}$ as a special case. 
		Theoretically, we for the first time establish a mini-max lower bound for nonparametric LDP learning with public data.
		
		\item To answer the second question, we propose the \textit{Locally differentially Private Classification Tree} (LPCT), a novel algorithm that leverages both public and private data.
		Specifically, we first create a tree partition on public data and then estimate the regression function through a  weighted average of public and private data. 
		Besides inheriting the merits of tree-based models such as efficiency, interpretability, and extensivity to multiple data types, LPCT is superior from both theoretical and empirical perspectives. 
		Theoretically, we show that LPCT attains the mini-max optimal convergence rate. 
		Empirically, we conduct experiments on both synthetic and real data sets to show the effectiveness of LPCT over state-of-the-art competitors.

		\item To answer the last question, we propose the \textit{Pruned Locally differentially Private Classification Tree} (LPCT-prune), a data-driven classifier that is free of parameter tuning. 
		Specifically, we query the privatized information with a sufficient tree depth and conduct an efficient pruning procedure.
		Theoretically, we show that LPCT-prune achieves the optimal convergence rate in most cases and maintains a reasonable convergence rate otherwise.
		Empirically, we illustrate that LPCT-prune with a default parameter performs comparable to LPCT with the best parameters. 
		
	\end{itemize}

	This article is organized as follows.
	In Section \ref{sec:relatedwork}, we provide a comprehensive literature review.
	Then we formally define our problem setting and introduce our methodology in Section \ref{sec::methodology}.
	The obtained theoretical results are presented in Section \ref{sec:theoreticalresults}.
	Then we develop the data-driven estimator in Section \ref{sec:datadriventreepruning}.
	Extensive experiments are conducted on synthetic data and real data in Section \ref{sec:simulation} and \ref{sec:realdata} respectively.
	The conclusion and discussions are in Section \ref{sec:comment}.

	\section{Related Work}\label{sec:relatedwork}
	
	\subsection{Private Learning with Public Data}
	
	Leveraging public data in private learning is a research topic that is both theoretically and practically meaningful to closing the gap between private and non-private methods.
	There is a long list of works addressing the issue of private learning with public data from a central differential privacy perspective.
	Through unlabeled public data, \cite{papernotsemi, papernot2018scalable} fed knowledge privately into student models, whose theoretical benefits are established by \cite{liu2021revisiting}.
	Empirical investigations have demonstrated the effectiveness of pretraining on public data and fine-tuning privately on sensitive data \citep{li2022large, yu2022differentially, ganesh2023public, yu2023selective}.
	Using the information obtained in public data, \cite{wang2020differentially, kairouz2021nearly, yu2021not, zhou2021bypassing, amid2022public, nasr2023effectively} investigated preconditioning or adaptively clipping the private gradients, which reduce the required amount of noise in DP-SGD and accelerate its convergence.
	Theoretical works such as \cite{bassily2018model, alon2019limits, bassily2020private} studied sample complexity bounds for PAC learning and query release based on the VC dimension of the function space.
	\cite{bie2022private} used public data to standardize private data and showed that the sample complexity of Gaussian mean estimation can be augmented in the sense of the range parameter. 
	More recently, by relating to sample compression schemes, \cite{ben2023private} presented sample complexities for several problems including Gaussian mixture estimation.
	\cite{ganesh2023public} explained the necessity of public data by investigating the loss landscape.

	In contrast, there is less attention focused on the local setting.
	\cite{su2023pac} considered half-space estimation with an auxiliary unlabeled public data. 
	They construct several weak LDP classifiers and label the public data by majority vote. 
	They established sample complexity that is linear in the dimension and polynomial in other terms for both private and public data.
	\cite{wang2023generalized} employed public data to estimate the leading eigenvalue of the covariance matrix, which will serve as a reference for data holders to clip their statistics. 
	Given the learned clipping threshold, the sample complexity for GLM with LDP for a general class of functions is improved to polynomial concerning error, dimension, and privacy budget, which is shown impossible without public data \citep{smith2017interaction,dagan2020interaction}. 
	Both \cite{wang2023generalized} and \cite{su2023pac} considered the unlabeled data, while our work further leverages the information contained in the labels if labeled public data is available. 
	More recently, \cite{ma2023decision} enhanced LDP regression by using public data to create an informative partition.
	However, the relationship between the source and target distributions is vaguely defined. 
	Also, their theoretical results are no better than using just private data. 
	\cite{lowy2023optimal} studied mean estimation, DP-SCO, and DP-ERM with public data.
	They assume that the public data and the private data are identically distributed. 
	Also, the theoretical results are established under sequentially-interactive LDP.
	Despite the distinctions in the research problem, their conclusions are analogous to ours: the asymptotic optimal error can be achieved by using either the private or public estimation solely, while weighted average estimation achieves better constant as well as empirical performance.

	\subsection{Locally Private Classification}
	LDP classification with parametric assumptions can be viewed as a special case of LDP-ERM \citep{smith2017interaction, zheng2017collect, wang2018empirical, daniely2019locally, dagan2020interaction, wang2023generalized}.
	Most works provide evidence of hardness on non-interactive LDP-ERM.
	Non-parametric methods possess a worse privacy-utility trade-off \citep{duchi2018minimax} and are considered to be harder than parametric problems. 
	\cite{zheng2017collect} investigated the kernel ridge regression and established a bound of estimation error of the order $n^{-1/4}$.
	Their results only apply to losses that are strongly convex and smooth. 
	The works most related to ours is that of \cite{berrett2019classification, berrett2021strongly}.
	\cite{berrett2019classification} studied non-parametric classification with the h\"{o}lder smoothness assumption. 
	They proposed a histogram-based method using the Laplace mechanism and showed that this approach is mini-max optimal under their assumption. 
	The strong consistency of this approach is established in \cite{berrett2021strongly} as a by-product.
	Our results include their conclusion as a special case and hold in a stronger sense, i.e. ``uniformly with high probability" instead of ``in expectation".

	\subsection{Private Hyperparameter Tuning}\label{sec:relatedparametertuning}
	
	Researchers highlighted the importance and difficulty of parameter tuning under DP constraint 
	\citep{papernot2021hyperparameter, wang2023dp} 
	focusing on central DP. 
	\cite{chaudhuri2013stability} explored this problem under strong stability assumptions.
	\cite{liu2019private} presented DP algorithms for selecting the best parameters in several candidates.
	Compared to the original algorithm, their method suffers a factor of 3 in the privacy parameter, meaning that 2/3 of the privacy budget is spent on the selection process. 
	\cite{papernot2021hyperparameter} improved the result with a much tighter privacy bound in Renyi-DP by running the algorithm random times and returning the best result. 
	Recently, \cite{wang2023dp} proposed a framework that reduces DP parameter selection to a non-DP counterpart.
	This approach leverages existing hyperparameter optimization methods and outperforms the uniform sampling approaches \citep{liu2019private, papernot2021hyperparameter}.
	\cite{ramaswamy2020training} proposed to tune hyperparameters on public data sets.
	\cite{mohapatra2022role} argued that, under a limited privacy budget, DP optimizers that require less tuning are preferable.  
	As far as we know, there is a lack of research focusing on the LDP parameter selection.
	\cite{butucea2020local} investigated adaptive parameter selection in non-parametric density estimation, and provided a data-driven wavelet estimator.

	\section{Locally Differentially Private Classification Tree}\label{sec::methodology}
	In this section, we introduce the \textit{Locally differentially Private Classification Tree} (LPCT). 
	After explaining necessary notations and problem setting in Section \ref{sec:problemsetting}, we propose our decision tree partition rule in Section \ref{sec:partition}.
	Finally, Section \ref{sec:ldpanypartition} introduces our privacy mechanism and proposes the final algorithm.

	\subsection{Problem Setting}\label{sec:problemsetting}

	We introduce necessary notations. 
	For any vector $x$, let $x^i$ denote the $i$-th element of $x$. 
	Recall that for $1 \leq p < \infty$, the $L_p$-norm of $x = (x^1, \ldots, x^d)$ is defined by $\|x\|_p := (|x^1|^p + \cdots + |x^d|^p)^{1/p}$. 
	Let $x_{i:j} = (x_i, \cdots, x_j)$ be the slicing view of $x$ in the $i,\cdots, j$ position.
	Throughout this paper, we use the notation $a_n \lesssim b_n$ and $a_n \gtrsim b_n$ to denote that there exist positive constant $c$ and $c'$ such that $a_n \leq c b_n$ and $a_n \geq c' b_n$, for all $n \in \mathbb{N}$.
	In addition, we denote $a_n\asymp b_n$ if $a_n\lesssim b_n$ and $b_n\lesssim a_n$.
	Let $a\vee b = \max (a,b)$ and $a\wedge b = \min (a,b)$. 
	Besides, for any set $A\subset \mathbb{R}^d$, the diameter of $A$ is defined by $\mathrm{diam}(A):=\sup_{x,x'\in A}\|x-x'\|_2$. 
	Let $f_1 \circ f_2$ represent the composition of functions $f_1$ and $f_2$.
	Let $A\times B$ be the Cartesian product of sets where $A\in\mathcal{X}_1$ and $B\in \mathcal{X}_2$. 
	For measure $\mathrm{P}$ on $\mathrm{X}_1$ and $\mathrm{Q}$ on $\mathcal{X}_2$, define the product measure $\mathrm{P} \otimes \mathrm{Q}$ on $\mathcal{X}_1\times \mathcal{X}_2$ as $\mathrm{P} \otimes \mathrm{Q} (A\times B)= \mathrm{P}(A)\cdot \mathrm{Q}(B)$. For an integer $k$, denote the $k$-fold product measure on $\mathcal{X}_1^k$ as $\mathrm{P}^k$.
	Let the standard Laplace random variable have probability density function $e^{-|x|} / 2$ for $x\in \mathbb{R}$.

	It is legitimate to consider binary classification while an extension of our results to multi-class classification is straightforward. 
	Suppose there are two unknown probability measures, the private measure $\mathrm{P}$ and the public measure $\mathrm{Q}$ on $\mathcal{X}\times \mathcal{Y} =  [0,1]^d \times\{0,1\}$. 
	We observe $n_{\mathrm{P}}$ i.i.d. private samples ${D}^{\mathrm{P}} = \{(X_1^{\mathrm{P}}, Y_1^{\mathrm{P}}), \dots,(X_{n_{\mathrm{P}}}^{\mathrm{P}}, Y_{n_{\mathrm{P}}}^{\mathrm{P}})\}$ drawn from the target distribution ${\mathrm{P}}$, and $n_{\mathrm{Q}}$ i.i.d. public samples ${D}^{\mathrm{Q}} = \{(X_1^{\mathrm{Q}}, Y_1^{\mathrm{Q}}), \dots,(X_{n_{\mathrm{Q}}}^{\mathrm{Q}}, Y_{n_{\mathrm{Q}}}^{\mathrm{Q}})\}$ drawn from the source distribution ${\mathrm{Q}}$. 
	The data points from the distributions ${\mathrm{P}}$ and ${\mathrm{Q}}$ are also mutually independent. 
	Our goal is to classify under the target distribution $\mathrm{P}$. Given the observed data, we construct a classifier $\widehat{f}:\mathcal{X} \rightarrow\mathcal{Y}$ which minimizes the classification risk under the target distribution $\mathrm{P}$ :
	\begin{align*}
		\mathcal{R}_{\mathrm{P}}(\widehat{f}) \triangleq \mathbb{P}_{(X, Y) \sim \mathrm{P}}(Y \neq \widehat{f}(X)).
	\end{align*}
	Here $\mathbb{P}_{(X, Y) \sim \mathrm{P}}(\cdot)$ means the probability when $(X, Y)$ are drawn from distribution $\mathrm{P}$. The Bayes risk, which is the smallest possible risk with respect to $\mathrm{P}$, is given by $\mathcal{R}_{\mathrm{P}}^{*}:=\inf \{\mathcal{R}_{\mathrm{P}}(f) | f: \mathcal{X} \to \mathcal{Y} \text{ is measurable} \}$. 
	In such binary classification problems, the regression functions are defined as
	\begin{align}\label{equ:defofgroundtruth}
		\eta_{\mathrm{P}}^*(x) \triangleq \mathrm{P}(Y=1 \mid X=x) \quad \text { and } \quad \eta_{\mathrm{Q}}^*(x) \triangleq \mathrm{Q}(Y=1 \mid X=x),
	\end{align}
	which represent the conditional distributions $\mathrm{P}_{Y \mid X}$ and $\mathrm{Q}_{Y \mid X}$. 
	The function that achieves Bayes risk with respect to $\mathrm{P}$ is called {Bayes function}, namely, $f^*_{\mathrm{P}}(x):=  \eins\left(\eta_{\mathrm{P}}^*\left(x\right) > 1/2\right)$.

	We consider the following setting. 
	The estimator $\widehat{f}$ is considered as a random function with respect to both $D^{\mathrm{P}}$ and $D^{\mathrm{Q}}$, while its construction process with respect to $D^{\mathrm{P}}$ is \textit{locally differentially private} (LDP).
	The rigorous definition of LDP is as follows.

	\begin{definition}[\textbf{Local Differential Privacy}]\label{def:ldp}
		Given data $\{(X_i^{\mathrm{P}},Y_i^{\mathrm{P}})\}_{i=1}^{n_{\mathrm{P}}}$, a privatized information $Z_i$, which is a random variable on $\mathcal{S}$, is released based on $(X_i^{\mathrm{P}}, Y_i^{\mathrm{P}})$ and $Z_1,\cdots, Z_{i - 1}$. 
		Let $\sigma(\mathcal{S})$ be the $\sigma$-field on $\mathcal{S}$. 
		$Z_i$ is drawn conditional on $(X_i^{\mathrm{P}}, Y_i^{\mathrm{P}})$ and $Z_{1 : i -1}$ via the distribution $\mathrm{R}_i \left(S \mid X_i^{\mathrm{P}}=x, Y_i^{\mathrm{P}} = y, Z_{1 : i -1} = z_{1 : i -1}\right) $ for $S\in \sigma(\mathcal{S})$. 
		Then the mechanism $\mathrm{R} = \{\mathrm{R}_i\}_{i = 1}^{n_{\mathrm{P}}}$ is \textit{sequentially-interactive $\varepsilon$-locally differentially private} ($\varepsilon$-LDP) if
		\begin{align*}
			\frac{\mathrm{R}_i\left(S \mid X_i^{\mathrm{P}}=x, Y_i^{\mathrm{P}} = y,  Z_{1 : i -1} = z_{1 : i -1} \right)}{\mathrm{R}_i\left(S \mid X_i^{\mathrm{P}}=x^{\prime}, Y_i^{\mathrm{P}} = y^{\prime},  Z_{1 : i -1} = z_{1 : i -1} \right)}  \leq e^{\varepsilon}
		\end{align*}
		for all $1 \leq i\leq  n_{\mathrm{P}}, S \in \sigma(\mathcal{S}), x, x^{\prime} \in \mathcal{X}, y, y^{\prime} \in \mathcal{Y}$, and $z_{1:i-1} \in \mathcal{S}^{i-1}$. 
		Moreover, if \begin{align*}
			\frac{\mathrm{R}_i\left(S \mid X_i^{\mathrm{P}}=x, Y_i^{\mathrm{P}} = y \right)}{\mathrm{R}_i\left(S \mid X_i^{\mathrm{P}}=x^{\prime}, Y_i^{\mathrm{P}} = y^{\prime}\right)}  \leq e^{\varepsilon}
		\end{align*}
		for all $1 \leq i\leq  n_{\mathrm{P}}, S \in \sigma(\mathcal{S}), x, x^{\prime} \in \mathcal{X}, y, y^{\prime} \in \mathcal{Y}$, then $\mathrm{R}$ is non-interactive $\varepsilon$-LDP. 
	\end{definition}
	
	This formulation is widely adopted \citep{duchi2018minimax, berrett2019classification}. 
	In contrast to central DP where the likelihood ratio is taken concerning some statistics of all data, LDP requires individuals to guarantee their own privacy by considering the likelihood ratio of each $(X_i^{\mathrm{P}}, Y_i^{\mathrm{P}})$. Once the view $z$ is provided, no further processing can reduce the deniability about taking a value $(x, y)$ since any outcome $z$ is nearly as likely to have come from some other initial value $(x^{\prime}, y^{\prime})$. 
	Sequentially-interactive mechanisms fulfill the requirements for a wide range of methods including gradient-based approaches.
	However, they can prove to be prohibitively slow in practice.
	Conversely, non-interactive mechanisms are quite efficient but suffer from poorer performance \citep{smith2017interaction}.
	Our analysis does not encompass the more general fully-interactive mechanisms due to their complex nature \citep{duchi2018minimax}.

	Besides the privacy guarantee of $\mathrm{P}$ data, we consider the existence of additional public data from distribution $\mathrm{Q}$.
	To depict the relationship between $\mathrm{P}$ and $\mathrm{Q}$, we consider the \textit{posterior drift} setting in transfer learning literature \citep{cai2021transfer}. 
	Under the posterior drift model, let $\mathrm{P}_X = \mathrm{Q}_X$ be the identical marginal distribution. The main difference between $\mathrm{P}$ and $\mathrm{Q}$ lies in the regression functions $\eta_{\mathrm{P}}^*(x)$ and $\eta_{\mathrm{Q}}^*(x)$.
	Specifically, let $\eta_{\mathrm{Q}}^*(x)=\phi \left(\eta_{\mathrm{P}}^*(x) \right)$ for some strictly increasing link function $\phi(\cdot)$ with $\phi\left({1}/{2}\right)={1}/{2}$. 
	The condition that $\phi$ is strictly increasing leads to the situation where those $X$ that are more likely to be labeled $Y=1$ under $\mathrm{P}$ are more likely to be labeled $Y=1$ under $\mathrm{Q}$. 
	The assumption $\phi\left({1}/{2}\right)={1}/{2}$ guarantees that those $X$ that are non-informative under $\mathrm{P}$ are the same under $\mathrm{Q}$, and more importantly, the sign of $\eta_{\mathrm{P}}^* - {1}/{2}$ and $\eta_{\mathrm{Q}}^* - {1}/{2}$ are identical.

	\subsection{Decision Tree Partition}
	\label{sec:partition}

	In this section, we formalize the decision tree partition process.
	A binary tree partition $\pi$ is a disjoint cover of $\mathcal{X}$ obtained by recursively split grids into subgrids. 
	While our methodology applies to any tree partition, it can be challenging to use general partitions such as the original CART \citep{breiman1984classification} for theoretical analysis. 
	Following \cite{cai2023extrapolated, ma2023decision}, we propose a new splitting rule called the \textit{max-edge partition rule}. 
	This rule is amenable to theoretical analysis and can also achieve satisfactory practical performance. 
	In the non-interactive setting, the private data set is unavailable during the partition process and the partition is created solely on the public data. 
	Given public data set $\{(X_i^{\mathrm{Q}}, Y_i^{\mathrm{Q}})\}_{i=1}^{n_{\mathrm{Q}}}$, the partition rule is stated as follows:
	\begin{itemize}
		\item Let $A_{(0)}^1 := [0,1]^d$ be the initial rectangular cell and $\pi_0 := \{ A_{(0)}^j\}_{j\in \mathcal{I}_{0}}$ be the initialized cell partition. 
		$\mathcal{I}_{0} = \{1\}$ stands for the initialized index set.
		In addition, let $p\in\mathbb{N}$ represent the maximum depth of the tree. 
		The parameter is fixed beforehand by the user and possibly depends on $n$. 
		\item Suppose we have obtained a partition $\pi_{i-1}$ of $\mathcal{X}$ after $i-1$ steps of the recursion. Let $\pi_i = \emptyset$. 
		In the $i$-th step, let $A_{(i-1)}^j\in \pi_{i-1}$ be $\times_{\ell=1}^d [a_{\ell}, b_{\ell}]$ for $j\in\mathcal{I}_{i-1}$. 
		We choose the edge to be split among the longest edges. The index set of longest edges is defined as
		\begin{align*}
			\mathcal{M}_{(i-1)}^{j} = \left\{k\mid |b_k - a_k| = \max_{\ell = 1,\cdots,d} |b_{\ell} - a_{\ell}|, \; k = 1,\cdots, d \right\}.
		\end{align*}
		\item Assume we split along the $\ell$-th dimension for $\ell \in \mathcal{M}_{(i-1)}^{j}$, $A_{(i-1)}^j$ is then partitioned into a left sub-cell ${A}_{(i-1)}^{j,0}(\ell)$ and a right sub-cell ${A}_{(i-1)}^{j,1}(\ell)$ along the midpoint of the chosen dimension, where ${A}_{(i-1)}^{j,0}(\ell) = \left\{x\mid x\in A_{(i-1)}^j, x^{\ell} < {(a_{\ell}+b_{\ell})}/{2} \right\}$ and ${A}_{(i-1)}^{j,1}(\ell) = A_{(i-1)}^j/ {A}_{(i-1)}^{j,0}(\ell)$.
		Then the dimension to be split is chosen by
		\begin{align}\label{equ:varreduction}
			\min\; \underset{\ell \in \mathcal{M}_{(i-1)}^{j}}{\arg\min} \;\; \text{score}\left( {A}_{(i-1)}^{j,0}(\ell), {A}_{(i-1)}^{j,1}(\ell), D^{\mathrm{Q}}\right) ,
		\end{align} 
		where the $\mathrm{score}(\cdot )$ function can be criteria for decision tree classifiers such as the Gini index or information gain.  
		The $\min$ operator implies that if the $\arg\min$ function returns multiple indices, we select the smallest index.
		Thus, the partition process is feasible even if there is no public data, where all axes have the same score. 
		\item Once $\ell$ is selected, let $\pi_i = \pi_i \cup \{{A}_{(i-1)}^{j,0}(\ell), {A}_{(i-1)}^{j,1}(\ell))\}$. 
	\end{itemize}
	The complete process is presented in Algorithm \ref{alg:partition} and illustrated in Figure \ref{fig:partitionillustration}. 
	For each grid, the partition rule selects the midpoint of the longest edges that achieves the largest score reduction.
	This procedure continues until either node has no samples or the depth of the tree reaches its limit. 
	We also define relationships between nodes which are necessary in Section \ref{sec:datadriventreepruning}.
	For node $A^j_{(i)}$, we define its parent node as
	\begin{align*}
		\mathrm{parent}(A^j_{(i)}) := A^{j'}_{(i-1)} \ \ \ \text{s.t.} \ \ \  
		A_{(i-1)}^{j'} \in \pi_{i-1 }, \ \ A_{(i)}^{j} \subset A_{(i-1)}^{j'}. 
	\end{align*}
	For $i'< i$, the ancestor of a node $A_{(i)}^j$ with depth $i'$ is then defined as 
	\begin{align*}
		\mathrm{ancestor}(A_{(i)}^j, i') = \mathrm{parent}^{i - i' } (A_{(i)}^j) = \underbrace{ \mathrm{parent} \circ \cdots \circ \mathrm{parent}}_{i - i'}(A_{(i)}^j),
	\end{align*}
	which is the grid in $\pi_{i'}$ that contains $A_{(i)}^j$.

	\begin{algorithm}[!t]
		\caption{Max-edge Partition}
		\label{alg:partition}
		{\bfseries Input: }{ Public data ${D}^{\mathrm{Q}}$, depth $p$.  }\\
		{\bfseries Initialization: } $\pi_0 = [0,1]^d$. 
		
		\For{$i = 1$ {\bfseries to} $p$}{
			$\pi_i = \emptyset$\\
			\For{$j$ in $\mathcal{I}_{i-1}$}{
				Select $\ell$ as in \eqref{equ:varreduction}.\\
				$\pi_i = \pi_i \cup \{{A}_{(i-1)}^{j,0}(\ell), {A}_{(i-1)}^{j,1}(\ell)\}$.}
		}
		{\bfseries Output: }{ Parition $\pi_p$}
		
	\end{algorithm}

	\begin{figure}[!t]
		\centering
		\includegraphics[width = .6\textwidth]{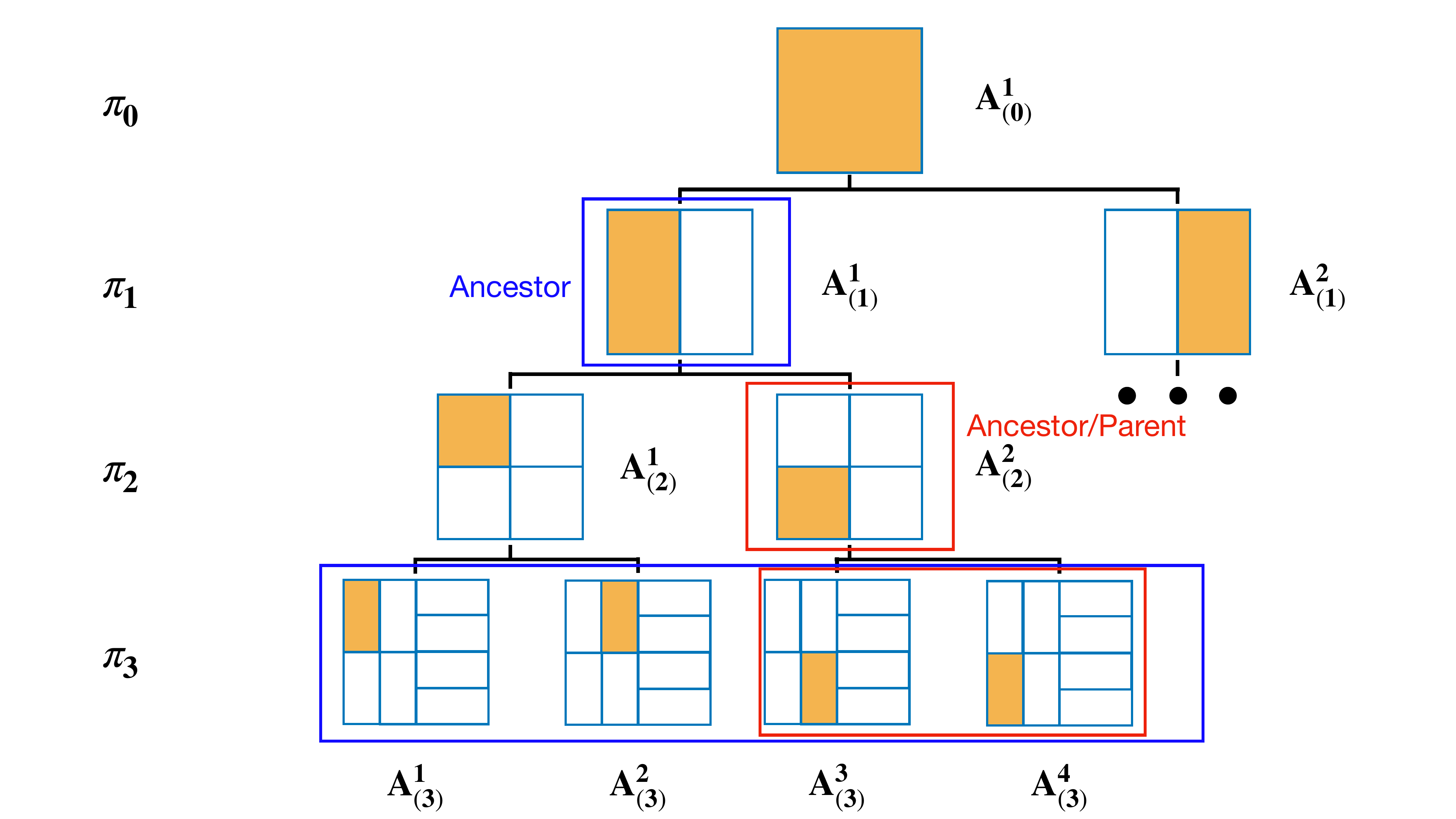}
		\caption{Partition created by the max-edge rule. The areas filled with orange represent the corresponding $A_{(i)}^j$ and the blue lines represent the partition boundaries after each level of partitioning. Red boxes contain an ancestor with depth 2, which is also a parent. Blue boxes contain an ancestor with depth 1.}
		\label{fig:partitionillustration}
	\end{figure}

	\subsection{Privacy Mechanism for Any Partition}\label{sec:ldpanypartition}

	This section focuses on the privacy mechanism based on general tree partitions.
	We first introduce the necessary definitions and then present our private estimation based on the Laplacian mechanism in \eqref{equ:twosampleeta}. 
	To avoid confusion, we refer the readers to a clear summarization of the defined estimators in Appendix \ref{sec:estimatorsummarization}.

	For index set $\mathcal{I}$, let $\pi = \pi_p= \{ A^j_{(p)} \}_{j \in \mathcal{I}_p}$ be the tree partition created by Algorithm  \ref{alg:partition}.
	A \textit{population classification tree estimation} is defined as 
	\begin{align}\label{equ:etapi}
		\overline{\eta}_{\mathrm{P}}^{\pi}(x) =  \sum_{j\in\mathcal{I}_p} \eins\{x\in A^j_{(p)}\}\frac{\int_{A^j_{(p)}}\eta_{\mathrm{P}}^*(x') \,d\mathrm{P}_X(x')}{\int_{A^j_{(p)}}\, d\mathrm{P}_X(x')}
	\end{align}
	The label is inferred using the \textit{population tree classifier} which is defined as $\overline{f}_{\mathrm{P}}^{\pi}(x)
	= \eins\left( \overline{\eta}_{\mathrm{P}}^{\pi}(x) > {1}/{2}\right)$.
	Here, we let $0/0 = 0$ by definition.
	
	To get an empirical estimator given the data set $D=\{(X_1^{\mathrm{P}},Y_1^{\mathrm{P}}),\ldots,(X_{n_{\mathrm{P}}}^{\mathrm{P}},Y_{n_{\mathrm{P}}}^{\mathrm{P}})\}$, we estimate the numerator and the denominator of \eqref{equ:etapi} separately. 
	To estimate the denominator, each sample 
	$(X_i^{\mathrm{P}}, Y_i^{\mathrm{P}})$ contributes a one-hot vector $U_i^{\mathrm{P}} \in \{0,1\}^{|\mathcal{I}_p|}$ where the $j$-th element of $U_i^{\mathrm{P}}$ is $\eins\{X_i^{\mathrm{P}} \in A^j_{(p)}\}$. 
	Then an estimation of $\int_{A^j_{(p)}}d\mathrm{P}_X(x')$ is $\frac{1}{n_{\mathrm{P}}}\sum_{i=1}^{n_{\mathrm{P}}} U_i^{\mathrm{P}j}$, which is the number of samples in $A^j_{(p)}$ divided by $n_{\mathrm{P}}$.  
	Analogously, let $V_i^{\mathrm{P}} = Y_i \cdot U_i^{\mathrm{P}} \in \{0,1\}^{|\mathcal{I}_p|}$. 
	An estimation of $\int_{A^j_{(p)}}\eta^*_{\mathrm{P}}(x')d\mathrm{P}_X(x')$ is $\frac{1}{n_{\mathrm{P}}}\sum_{i=1}^{n_{\mathrm{P}}} V_i^{\mathrm{P}j}$, which is the sum of the labels in $A^j_{(p)}$ divided by $n_{\mathrm{P}}$. 
	Combining the pieces, a \textit{classification tree estimation} is defined as 
	\begin{align}\label{equ:reformetapi}
		\widehat{\eta}^{\pi}_{\mathrm{P}}(x) = \sum_{j\in\mathcal{I}_p}\eins\{x\in A^j_{(p)}\}\frac{ \sum_{i=1}^{n_{\mathrm{P}}}V_i^{\mathrm{P}j}}{ \sum_{i=1}^{n_{\mathrm{P}}}U_i^{\mathrm{P}j}}.
	\end{align}
	The corresponding classifer is defined as $\widehat{f}^{\pi}_{\mathrm{P}}(x)
	= \eins\left(\widehat{\eta}^{\pi}_{\mathrm{P}}(x) > {1}/{2}\right)$.
	In other words,  $\widehat{\eta}^{\pi}_{\mathrm{P}}(x)$ estimates $\eta^*_{\mathrm{P}}(x)$ by the average of the responses in the cell. 
	In the non-private setting, each data holder prepares $U_i^{\mathrm{P}}$ and $V_i^{\mathrm{P}}$ according to the partition $\pi$ and sends it to the curator. 
	Then the curator aggregates the transmission following \eqref{equ:reformetapi}.

	To protect the privacy of each data, we propose to estimate the numerator and denominator of the population regression tree using a privatized method. 
	Specifically, $U_i^{\mathrm{P}}$ and $V_i^{\mathrm{P}}$ are combined with the Laplace mechanism \citep{dwork2006calibrating} before being sent to the curator. 
	For i.i.d. standard Laplace random variables $\zeta_i^j$ and $\xi_i^j$, let
	\begin{align}\label{equ:privatizeprocedureU}
		\Tilde{U}_i^{\mathrm{P}} := U_i^{\mathrm{P}} + \frac{4}{\varepsilon}\zeta_i = U_i^{\mathrm{P}} + (\frac{4}{\varepsilon}\zeta_i^1, \cdots ,  \frac{4}{\varepsilon}\zeta_i^{|\mathcal{I}_p|}).
	\end{align}
	Then a privatized estimation of $\int_{A^j_{(p)}}d\mathrm{P}_X(x)$ is $\frac{1}{n_{\mathrm{P}}}\sum_{i=1}^{n_{\mathrm{P}}} \tilde{U}_i^{\mathrm{P}j}$.
	Analogously, let
	\begin{align}\label{equ:privatizeprocedureV}
		\Tilde{V}_i^{\mathrm{P}} := V_i^{\mathrm{P}} + \frac{4 }{\varepsilon}\xi_i = V_i^{\mathrm{P}} + (\frac{4}{\varepsilon}\xi_i^1, \cdots ,  \frac{4}{\varepsilon}\xi_i^{|\mathcal{I}_p|}).
	\end{align}
	An estimation of $\int_{A^j_{(p)}}\eta^*_{\mathrm{P}}(x)d\mathrm{P}_X(x)$ is $\frac{1}{n_{\mathrm{P}}}\sum_{i=1}^{n_{\mathrm{P}}} \tilde{V}_i^{\mathrm{P}j}$.
	Then using the privatized information $(\tilde{U}_i^{\mathrm{P}}, \tilde{V}_i^{\mathrm{P}}), i= 1 , \cdots, n_{\mathrm{P}}$, we can estimate the regression function as 
	\begin{align}\label{equ:etadppi}
		\tilde{\eta}^{\pi}_{\mathrm{P}}(x) = 
		\sum_{j\in\mathcal{I}_p}\eins\{x\in A^j_{(p)}\}\frac{\sum_{i=1}^{n_{\mathrm{P}}}\tilde{V}^{\mathrm{P}j}_i}{ \sum_{i=1}^{n_{\mathrm{P}}}\tilde{U}^{\mathrm{P}j}_i}. 
	\end{align}
	The procedure is also derived by  \cite{berrett2019classification}.
	As an alternative, one can use random response \citep{warner1965randomized} since both $U$ and $V$ are binary vectors. 
	The following proposition shows that the estimation procedures satisfy LDP. 
	\begin{proposition}\label{thm:dppartition}
		Let $\pi = \{ A^j \}_{j \in \mathcal{I}}$ be any partition of $\mathcal{X}$ with $\cup_{j\in\mathcal{I}} A^j = \mathcal{X}$ and  $A^i\cap A^j = \emptyset$, $i\neq j$. 
		Then the privacy mechanism 
		defined in \eqref{equ:privatizeprocedureU} and \eqref{equ:privatizeprocedureV} is non-interactive $\varepsilon$-LDP. 
	\end{proposition}

	In the case where one also has access to the $\mathrm{Q}$-data in addition to the $\mathrm{P}$-data,  the $\mathrm{Q}$ data can be used to help the classification task under the target distribution $\mathrm{P}$ and should be taken into consideration.
	The utilization of $\mathrm{P}$ data must satisfy the local DP constraint while the $\mathrm{Q}$ data can be arbitrarily adopted. 
	We accommodate the existing private estimation \eqref{equ:etadppi} and non-private estimation \eqref{equ:reformetapi} by taking a weighted average based on both estimators.
	Denote the encoded information from $D^{\mathrm{P}}$ and $D^{\mathrm{Q}}$ as $(U_i^{\mathrm{P}}, V_i^{\mathrm{P}})$ and $(U_i^{\mathrm{Q}}, V_i^{\mathrm{Q}})$, respectively. 
	We use the same partition for $\mathrm{P}$ estimator and $\mathrm{Q}$ estimator. 
	Namely, \eqref{equ:reformetapi} and \eqref{equ:etadppi} are constructed via an identical $\pi$.
	When taking the average, data from the different distributions should have different weights since the signal strengths are different between $\mathrm{P}$ and $\mathrm{Q}$. 
	To make the classification at $x \in \mathcal{X}$, the \textit{Locally differentially Private Classification Tree estimation} (LPCT) is defined as follows:
	\begin{align}\label{equ:twosampleeta}
		\tilde{\eta}_{\mathrm{P},\mathrm{Q}}^{\pi}(x) = 
		\sum_{j\in\mathcal{I}_p}\eins\{x\in A^j_{(p)}\}\frac{
			\sum_{i=1}^{n_{\mathrm{P}}} \tilde{V}^{\mathrm{P} j }_i + \lambda \sum_{i=1}^{n_{\mathrm{Q}}}{V}^{\mathrm{Q} j}_i
		}
		{ 
			\sum_{i=1}^{n_{\mathrm{P}}} \tilde{U}^{\mathrm{P} j }_i +  \lambda\sum_{i=1}^{n_{\mathrm{Q}}}{U}^{\mathrm{Q} j}_i
		}.
	\end{align}
	Here, we let $\lambda \geq  0$. 
	The \textit{locally private tree classifier} is $\tilde{f}_{\mathrm{P},\mathrm{Q}}^{\pi}(x) = \eins( \tilde{\eta}_{\mathrm{P},\mathrm{Q}}^{\pi}(x) > {1}/{2})$.
	In each leaf node $A^j_{(p)}$, the classifier is assigned a label based on the weighted average of estimations from $\mathrm{P}$ and $\mathrm{Q}$ data. 
	The parameter $\lambda$ serves to balance the relative contributions of both $\mathrm{P}$ and $\mathrm{Q}$.
	A higher value of $\lambda$ indicates that the final predictions are predominantly influenced by the public data, and vice versa.

	The estimator defined in \eqref{equ:twosampleeta} easily adapts to solve the non-private transfer learning problem in \cite{cai2021transfer} if we define 
	\begin{align}\label{equ:twosamplenonprivateeta}
		\widehat{\eta}_{\mathrm{P},\mathrm{Q}}^{\pi}(x) = 
		\sum_{j\in\mathcal{I}_p}\eins\{x\in A^j_{(p)}\}\frac{
			\sum_{i=1}^{n_{\mathrm{P}}} {V}^{\mathrm{P} j }_i + \lambda \sum_{i=1}^{n_{\mathrm{Q}}}{V}^{\mathrm{Q} j}_i
		}
		{ 
			\sum_{i=1}^{n_{\mathrm{P}}} {U}^{\mathrm{P} j }_i +  \lambda\sum_{i=1}^{n_{\mathrm{Q}}}{U}^{\mathrm{Q} j}_i
		}.
	\end{align}
	Similar to the nearest neighbor estimator in \cite{cai2021transfer}, \eqref{equ:twosamplenonprivateeta} estimates by taking weighted averages of labels from nearby samples.
	Our theoretical results in Section \ref{sec:theoreticalresultsforlpct} imply that \eqref{equ:twosamplenonprivateeta} is also rate optimal in terms of rate for non-private transfer learning.
	A key advantage for incorporating \eqref{equ:twosamplenonprivateeta} is that, given the partition,  it stores only the prediction values at each node, which makes it convenient to inject noise. 
	Conversely, the nearest neighbor estimator (similarly, the kernel estimator) requires a query to the entire training data set for each single testing sample, which prohibits any modification towards a private estimator. 
	See \cite{kroll2021density} for an example of such limitations.

	\section{Theoretical Results}\label{sec:theoreticalresults}
	In this section, we present the obtained theoretical results. 
	We present the matching lower and upper bounds of excess risk in Section \ref{sec:convergencerateminimax} and \ref{sec:theoreticalresultsforlpct}. 
	Finally, we demonstrate that LPCT is free from the unknown range parameter in Section \ref{sec:rangeparameter}.
	All technical proofs can be found in Appendix \ref{app:proofs}.

	\subsection{Mini-max Convergence Rate}\label{sec:convergencerateminimax}
	
	We first give the necessary assumptions on the distribution $\mathrm{P}$ and $\mathrm{Q}$.
	Then we provide the mini-max rate under the assumptions.

	\begin{assumption}\label{asp:alphaholder}
		\label{asp:boundedmarginal}
		Let $\alpha \in (0, 1]$. Assume the regression function $\eta^{*}_{\mathrm{P}} : \mathcal{X} \to \mathbb{R}$ is $\alpha$-H\"{o}lder continuous, i.e. there exists a constant $c_L > 0$ such that for all $x_1, x_2 \in \mathcal{X}$,  $|\eta^*_{\mathrm{P}}(x_1) - \eta^*_{\mathrm{P}}(x_2)| \leq c_L \|x_1 - x_2\|^{\alpha}$.
		Also, assume that the marginal density functions of $\mathrm{P}$ and $\mathrm{Q}$ are upper and lower bounded, i.e. $\underline{c}\leq d\mathrm{P}_X(x)\leq \overline{c}$ for some $\underline{c}, \overline{c}>0$. 
	\end{assumption}
	
	\begin{assumption}[Margin Assumption] \label{asp:margin}
		Let $\beta \geq 0, C_{\beta}>0$. Assume
		\begin{align*}
			\mathrm{P}_X\left(\left|\eta^*_{\mathrm{P}}(X)-\frac{1}{2}\right|<t\right) \leq C_{\beta} t^{\beta}.
		\end{align*}
	\end{assumption}
	
	\begin{assumption}[Relative Signal Exponent]\label{asp:rse}
		Let the relative signal exponent $\gamma \in(0, \infty)$ and $C_\gamma \in(0, \infty)$. Assume that 
		\begin{align*}
			\left(\eta^*_{\mathrm{Q}}(x)-\frac{1}{2}\right)\left(\eta^*_{\mathrm{P}}(x)-\frac{1}{2}\right) \geq 0, \; \text{and} \;\;\;\;
			\left|\eta^*_{\mathrm{Q}}(x)-\frac{1}{2}\right| \geq C_{\gamma} \left| \eta^*_{\mathrm{P}}(x)-\frac{1}{2} \right|^\gamma. 
		\end{align*}
	\end{assumption}

	Assumption \ref{asp:alphaholder} and \ref{asp:margin} are standard conditions that have been widely used for non-parametric classification problems \citep{audibert2007fast, samworth2012optimal, chaudhuri2014rates}. 
	Assumption \ref{asp:rse} depicts the similarity between the conditional distribution of public data and private data. 
	When $\gamma$ is small, the signal strength of $\mathrm{Q}$ is strong compared to $\mathrm{P}$. 
	In the extreme instance of $\gamma = 0$, $\eta^*_{\mathrm{Q}}(x)$ consistently remains distant from ${1}/{2}$ by a constant.
	In contrast, when $\gamma$ is large, there is not much information contained in $\mathrm{Q}$. 
	The special case of $\gamma = \infty$ allows $\eta^*_{\mathrm{Q}}(x)$ to always be ${1}/{2}$ and is thus non-informative.
	We present the following theorem that specifies the mini-max lower bound with LDP constraints on $\mathrm{P}$ data under the above assumptions.
	The proof is based on first constructing two function classes and then applying the information inequalities for local privacy \citep{duchi2018minimax} and Assouad's Lemma \citep{tsybakov2008introduction}.

	\begin{theorem}\label{thm:lowerbound}
		Denote the function class of $(\mathrm{P}, \mathrm{Q})$ satisfying Assumption \ref{asp:alphaholder}, \ref{asp:margin}, and \ref{asp:rse} by $\mathcal{F}$. Then for any estimator $\widehat{f}$ that is sequentially-interactive $\varepsilon$-LDP, there holds
		\begin{align}\label{equ:minimaxconvergencerate}
			\inf_{\widehat{f}}\sup_{\mathcal{F}}\mathbb{E}_{X, Y}\left[ \mathcal{R}_{\mathrm{P}} (\widehat{f}) - \mathcal{R}_{\mathrm{P}}^*  \right] \gtrsim \left(n_{\mathrm{P}} \left(e^{\varepsilon} - 1\right)^2 + n_{\mathrm{Q}}^{\frac{2 \alpha + 2 d}{2 \gamma\alpha + d}}\right)^{-\frac{\alpha(\beta + 1)}{2\alpha + 2d}}. 
		\end{align}
	\end{theorem}

	If $n_{\mathrm{Q}}= 0$, our results reduce to the special case of LDP non-parametric classification with only private data.
	In this case, the mini-max lower bound is of the order $({n_{\mathrm{P}} (e^{\varepsilon} - 1)^2})^{-\frac{\alpha(\beta + 1)}{2\alpha + 2 d}}$. 
	Previously, \cite{berrett2019classification} provided an estimator that reaches the mini-max optimal convergence rate for small $\varepsilon$. 
	If $n_{\mathrm{P}}  = 0$, the result recovers the learning rate $n_{\mathrm{Q}}^{-\frac{\alpha(\beta + 1)}{2\gamma\alpha +  d}}$ established for transfer learning under posterior drift \citep{cai2021transfer}. 
	In this case, we have only source data from $\mathrm{Q}$ and want to generalize the classifier on the target distribution $\mathrm{P}$. 
	
	The mini-max convergence rate \eqref{equ:minimaxconvergencerate} is the same as if one fits an estimator using only $\mathrm{P}$ data with sample size $n_{\mathrm{P}}  + (e^{\varepsilon} - 1)^{-2}\cdot n_{\mathrm{Q}}^{\frac{2\alpha + 2d}{2 \gamma\alpha + d}}$.
	The contribution of $\mathrm{Q}$ data is substantial when $n_{\mathrm{P}} \ll (e^{\varepsilon} - 1)^{-2}\cdot n_{\mathrm{Q}}^{\frac{2\alpha + 2d}{2\gamma \alpha + d}}$. 
	Otherwise, the convergence rate is not significantly better than using $\mathrm{P}$ data only. 
	The quantity $(e^{\varepsilon} - 1)^{-2}\cdot n_{\mathrm{Q}}^{\frac{2\alpha + 2d}{2\gamma \alpha + d}}$, which is referred to as {effective sample size}, represent the theoretical gain of using $\mathrm{Q}$ data with sample size $n_{\mathrm{Q}}$ under given $\varepsilon, \alpha, d$, and $\gamma$.  
	Notably, the smaller $\gamma$ is, the more effective the $\mathrm{Q}$ data becomes, and the greater gain is acquired by using $\mathrm{Q}$ data.
	Regarding different values of $\gamma$, our work covers the setting in previous work when a small amount of effective public samples are required \citep{bie2022private, ben2023private}, and when a large amount of noisy/irrelevant public samples are required \citep{ganesh2023public, yu2023selective}. 
	Also, the public data is more effective for small $\varepsilon$, i.e. under stronger privacy constraints. 
	Intuitively, in the presence of stringent privacy constraints, the utilization of $\mathrm{Q}$ data is as effective as incorporating a substantial amount of $\mathrm{P}$ data.
	Furthermore, even for a constant $\varepsilon$, our effective sample size significantly exceeds $n_{\mathrm{Q}}^{\frac{2\alpha + d}{2\gamma \alpha + d}}$ as in the non-private scenario \citep{cai2021transfer}. This suggests that, in comparison to transfer learning in the non-private setting, public data assumes greater value in the private setting.

	\subsection{Convergence Rate of LPCT}\label{sec:theoreticalresultsforlpct}
	
	In the following, we show the convergence rate of LPCT reaches the lower bound with proper parameters. 
	The proof is based on a carefully designed decomposition in Appendix \ref{sec:proofoftheoremutility}.
	Define the quantity 
	\begin{align}\label{equ:defofdelta}
		\delta := \left( \frac{n_{\mathrm{P}} \varepsilon^2}{\log (n_{\mathrm{P}} + n_{\mathrm{Q}})} + \left(\frac{n_{\mathrm{Q}}}{\log (n_{\mathrm{P}} + n_{\mathrm{Q}})}\right)^{\frac{2\alpha + 2d}{2 \gamma\alpha + d}}\right)^{-\frac{d}{2\alpha + 2 d}}.
	\end{align}

	\begin{theorem}\label{thm:utility}
		Let $\pi$ be constructed through the max-edge partition in Section \ref{sec:partition} with depth $p$.
		Let $\tilde{f}^{\pi}_{\mathrm{P}, \mathrm{Q}}$ be the locally private two-sample tree classifier. 
		Suppose Assumption \ref{asp:alphaholder}, \ref{asp:margin}, and \ref{asp:rse} hold.
		Then, for $n_{\mathrm{P}}^{-\frac{\alpha}{2\alpha + 2d}}\lesssim \varepsilon \lesssim n_{\mathrm{P}}^{\frac{d}{4\alpha + 2d}}$, if $\lambda  \asymp \delta^{ \frac{(\gamma \alpha-\alpha) \wedge(2 \gamma \alpha-d)}{d}}$ and $2^p \asymp \delta^{-1}$, there holds
		\begin{align}\label{equ:convergencerate}
			\mathcal{R}_{\mathrm{P}} (\tilde{f}^{\pi}_{\mathrm{P}, \mathrm{Q}}) - \mathcal{R}_{\mathrm{P}}^*   \lesssim \delta^{\frac{\alpha(\beta + 1)}{d}} = 
			\left(\frac{n_{\mathrm{P}} \varepsilon^2}{\log (n_{\mathrm{P}} + n_{\mathrm{Q}})} + \left(\frac{n_{\mathrm{Q}}}{\log (n_{\mathrm{P}} + n_{\mathrm{Q}})}\right)^{\frac{2\alpha + 2d}{2 \gamma\alpha + d}}\right)^{-\frac{\alpha(\beta + 1)}{2\alpha + 2 d}}
		\end{align}
		with probability $1 -3/(n_{\mathrm{P}} + n_{\mathrm{Q}})^2 $ with respect to $\mathrm{P}^{n_{\mathrm{P}}}\otimes\mathrm{Q}^{n_{\mathrm{Q}}}\otimes\mathrm{R} $ where $\mathrm{R}$ is the joint distribution of privacy mechanisms in \eqref{equ:privatizeprocedureU} and \eqref{equ:privatizeprocedureV}.
	\end{theorem}

	Compared to Theorem \ref{thm:lowerbound}, the LPCT with the best parameter choice reaches the mini-max lower bound up to a logarithm factor. 
	Both parameters $\lambda$ and $p$ should be assigned a value of the correct order to achieve the optimal trade-off. 
	The depth $p$ increases when either $n_{\mathrm{P}}$ or $n_{\mathrm{Q}}$ grows. 
	This is intuitively correct since the depth of a decision tree should increase with a larger number of samples. 
	The strength ratio $\lambda$, on the other hand, is closely related to the value of $\gamma$. 
	As explained earlier, a small $\gamma$ guarantees a stronger signal in $\eta_{\mathrm{Q}}^*$. 
	Thus, the value of $\lambda$ is decreasing for $\gamma$. 
	In this case, when the signal strength of $\eta_{\mathrm{Q}}^*$ is large, we rely more on the estimation by $\mathrm{Q}$ data by assigning a larger ratio to it.
	
	We note that when $\varepsilon$ is large, there is a gap between $(e^{\varepsilon} - 1)^2$ in the lower bound and $\varepsilon^2$ in the upper bound. 
	The problem can be mitigated by employing the random response mechanism \citep{warner1965randomized} instead of the Laplace mechanism, as mentioned in Section \ref{sec:ldpanypartition}.
	The theoretical analysis generalizes analogously. 
	Here, we do not address this issue as we want to focus on public data. 
	The theorem restricts $n_{\mathrm{P}}^{-\frac{\alpha}{2\alpha + 2d}}\lesssim \varepsilon \lesssim n_{\mathrm{P}}^{\frac{d}{4\alpha + 2d}}$. 
	When $\varepsilon \lesssim n_{\mathrm{P}}^{-\frac{\alpha}{2\alpha + 2d}}$, the estimator is drastically perturbed by the random noise and the convergence rate is no longer optimal.
	As $\varepsilon$ gets larger, the upper bound of excess risk decreases until $\varepsilon \asymp n_{\mathrm{P}}^{\frac{d}{4\alpha + 2d}}$. 
	When $\varepsilon \gtrsim n_{\mathrm{P}}^{\frac{d}{2\alpha + 2d}}$, the random noise is negligible and the convergence rate remains to be 
	\begin{align*}
		\left(\frac{n_{\mathrm{P}} }{\log (n_{\mathrm{P}} + n_{\mathrm{Q}})} + \left(\frac{n_{\mathrm{Q}}}{\log (n_{\mathrm{P}} + n_{\mathrm{Q}})}\right)^{\frac{2\alpha + d}{2 \gamma\alpha + d}}\right)^{-\frac{\alpha(\beta + 1)}{2\alpha +  d}},
	\end{align*}
	which recovers the rate of non-private learning as established in \cite{cai2021transfer}.

	We discuss the order of ratio $\lambda$. 
	When $\gamma \geq \frac{d}{\alpha } - 1$, we have $\lambda \asymp \delta^{\frac{(\gamma -1)\alpha}{d}}$.
	This rate is exactly the same as in the conventional transfer learning case \citep{cai2021transfer} and is enough to achieve the optimal trade-off. 
	Yet, under privacy constraints, this choice of $\lambda$ will fail when $\gamma <\frac{d}{\alpha } - 1$. 
	Roughly speaking, if we let $\lambda \asymp \delta^{\frac{(\gamma -1)\alpha}{d}}$ in this case, there is a probability such that the noise $\frac{4}{\varepsilon}\sum_{i=1}^{n_{\mathrm{P}}}\zeta_i^j$ in \eqref{equ:twosampleeta}
	will dominate the non-private density estimation
	$\sum_{i=1}^{n_{\mathrm{P}}} {U}^{\mathrm{P} j }_i +  \lambda\sum_{i=1}^{n_{\mathrm{Q}}}{U}^{\mathrm{Q} j}_i$.
	When the domination happens, the point-wise estimation is decided by random noises, and thus the overall excess risk fails to converge. 
	To alleviate this issue, we assign a larger value $\lambda \asymp\delta^{\frac{2\gamma\alpha - d}{d}} > \delta^{\frac{(\gamma -1)\alpha}{d}}$
	to ensure 
	$\frac{4}{\varepsilon}\sum_{i=1}^{n_{\mathrm{P}}}\zeta_i^j \lesssim \sum_{i=1}^{n_{\mathrm{P}}} {U}^{\mathrm{P} j }_i +  \lambda\sum_{i=1}^{n_{\mathrm{Q}}}{U}^{\mathrm{Q} j}_i$.
	As an interpretation, in the private setting,  we place greater reliance on the estimation by public data compared to the non-private setting.

	Last, we discuss the computational complexity of LPCT. We first consider the average computation complexity of LPCT.
	The training stage consists of two parts.
	Similar to the standard decision tree algorithm \citep{breiman1984classification}, the partition procedure takes $\mathcal{O}(p n_{\mathrm{Q}} d )$ time.
	The computation of \eqref{equ:twosampleeta} takes $\mathcal{O}(p( n_{\mathrm{P}} + n_{\mathrm{Q}}) )$ time.
	From the proof of Theorem \ref{thm:utility}, we know that $2^{p} \asymp \delta^{-1}$, where $\delta^{-1} \lesssim n_{\mathrm{P}} + n_{\mathrm{Q}}$.
	Thus the training stage complexity is around $\mathcal{O}\left((n_{\mathrm{P}} + n_{\mathrm{Q}}d )\log (n_{\mathrm{P}} + n_{\mathrm{Q}})\right)$ which is linear.
	Since each prediction of the decision tree takes $\mathcal{O}(p)$ time, the test time for each test instance is no more than $\mathcal{O} (\log (n_{\mathrm{P}} + n_{\mathrm{Q}}) )$.
	As for storage complexity, since LPCT only requires the storage of the tree structure and the prediction value at each node, the space complexity of LPCT is $\mathcal{O}(\delta^{-1} )$ which is also sub-linear.
	In short, LPCT is an efficient method in time and memory.

	\subsection{Eliminating the Range Parameter}\label{sec:rangeparameter}
	
	We discuss how public data can eliminate the range parameter contained in the convergence rate. 
	When $\mathcal{X}$ is unknown, one must select a predefined range $[-r, r]^d$ for creating partitions.
	Only then do the data holders have a reference for encoding their information. 
	In this case, the convergence rate of \cite{berrett2019classification} becomes $({r^{2d}}/({n_{\mathrm{P}} \varepsilon^2)})^{\frac{(1+\beta)\alpha}{2\alpha + 2d}} \vee (1 - \int_{[-r, r]^d} d\mathrm{P(x)})$, which is slower than the known $\mathcal{X} = [0,1]^d$ for any $r$. 
	A small $r$ means that a large part of the informative $\mathcal{X}$ will be ignored (the second term is large), whereas a large $r$ may create too many redundant cells (the first term is large).
	See \cite{gyorfi2022rate} for detailed derivation. 
	Without much prior knowledge, this parameter can be hard to tune. 
	The removal of the range parameter is first discussed in \cite{bie2022private}, where $d+1$ public data is required. 
	Our result shows that we do not need to tune this range parameter with the help of public data.

	\begin{theorem}\label{thm:rangeparameter}
		Suppose the domain $\mathcal{X} = \times_{j=1}^d[a^j,b^j]$ is unknown. 
		Let $\widehat{a^j} = \min_{i}X_i^j$ and $\widehat{b^j} = \max_{i}X_i^j$.
		Define the min-max scaling map $\mathcal{M}: \mathbb{R}^d \to \mathbb{R}^d$ where 
		$\mathcal{M}(x)^j = (x^j - \widehat{a^j}) / (\widehat{b^j} - \widehat{a^j})$. 
		The the min-max scaled $D^{\mathrm{P}}$ and $D^{\mathrm{Q}}$ is written as $\mathcal{M}(D^{\mathrm{P}})$ and $\mathcal{M}(D^{\mathrm{Q}})$. 
		Then for $\tilde{f}^{\pi}_{\mathrm{P}, \mathrm{Q}}$ trained within $[0,1]^d$ on $\mathcal{M}(D^{\mathrm{P}})$ and $\mathcal{M}(D^{\mathrm{Q}})$ as in Theorem \ref{thm:utility}, we have 
		\begin{align*}
			\mathcal{R}_{\mathrm{P}} \left((\tilde{f}^{\pi}_{\mathrm{P}, \mathrm{Q}} \cdot \eins_{[0,1]^d}) \circ \mathcal{M}\right) - \mathcal{R}_{\mathrm{P}}^*   \lesssim  
			\delta^{\frac{\alpha(\beta + 1)}{d}} + \frac{\log n_{\mathrm{Q}}}{n_{\mathrm{Q}}}
		\end{align*}
		with probability $1 -3/(n_{\mathrm{P}} + n_{\mathrm{Q}})^2  - d / n_{\mathrm{Q}}^2$ with respect to $\mathrm{P}^{n_{\mathrm{P}}}\otimes\mathrm{Q}^{n_{\mathrm{Q}}}\otimes\mathrm{R}$.
	\end{theorem}
	
	Compared to Theorem \ref{thm:utility}, the upper bound of Theorem \ref{thm:rangeparameter} increases by $\log n_{\mathrm{Q}} / n_{\mathrm{Q}}$, which is a minor term in most cases. 
	Moreover, Theorem \ref{thm:rangeparameter} fails with additional probability $d / n_{\mathrm{Q}}^2$ since we are estimating $[a^j, b^j]$ with $[\widehat{a^j} , \widehat{b^j}]$. 
	Both changes indicate that sufficient public data is required to resolve the unknown range issue.

	\section{Data-driven Tree Structure Pruning}\label{sec:datadriventreepruning}

	In this section, we develop a data-driven approach for parameter tuning and investigate its theoretical properties.
	In Section \ref{sec:naivestrategy}, we derive a selection scheme in the spirit of Lepski's method \citep{lepskii1992asymptotically}, a classical method for adaptively choosing the hyperparameters.
	In Section \ref{sec:plpct}, we make slight modifications to the scheme and illustrate the effectiveness of the amended approach.

	\subsection{A Naive Strategy}\label{sec:naivestrategy}

	We first develop a private analog of Lepski's method \citep{lepskii1992asymptotically} for LPCT.
	There are two parameters: the depth $p$ and the relative ratio $\lambda$, both of which have an optimal value depending on the unknown $\alpha$ and $\gamma$. 
	We do not set the parameters to specific values. 
	Instead, we first query the privatized information with a sufficient tree depth $p_0$.
	Then, we introduce a greedy procedure to locally prune the tree to an appropriate depth and select $\lambda$ based on local information.

	We provide a preliminary result that will guide the derivation of the pruning procedure.
	For max-edge partition $\pi_{p_0}$ with depth $p_0$ and node $A_{(p_0)}^j$, the \textit{depth $k$ ancestor private tree estimation} is defined as 
	\begin{align}\label{equ:defofetapik}
		\tilde{\eta}^{\pi_k}_{\mathrm{P}, \mathrm{Q}}(x) = 
		\sum_{j\in\mathcal{I}_{p_0}}\eins\{x\in A^j_{(p_0)}\}
		\frac{\sum_{\mathrm{ancestor}(A^{j'}_{(p_0)}, k) = \mathrm{ancestor}(A^{j}_{(p_0)}, k) }\left(
			\sum_{i=1}^{n_{\mathrm{P}}} \tilde{V}^{\mathrm{P} j' }_i + \lambda\sum_{i=1}^{n_{\mathrm{Q}}}{V}^{\mathrm{Q} j'}_i\right)
		}
		{ 
			\sum_{\mathrm{ancestor}(A^{j'}_{(p_0)}, k) = \mathrm{ancestor}(A^{j}_{(p_0)}, k) }\left(
			\sum_{i=1}^{n_{\mathrm{P}}} \tilde{U}^{\mathrm{P} j' }_i + \lambda\sum_{i=1}^{n_{\mathrm{Q}}}{U}^{\mathrm{Q} j'}_i\right)
		}
	\end{align}
	and consequently the classifier is defined as $\tilde{f}^{\pi_k}_{\mathrm{P}, \mathrm{Q}}(x)  = \eins (\tilde{\eta}^{\pi_k}_{\mathrm{P}, \mathrm{Q}}(x) >{1}/{2})$.
	See Figure \ref{fig:pruning1} for an illustration of estimation at ancestor nodes.

	\begin{figure}[]
		\centering
		\subfigure[Initial estimation]{
			\begin{minipage}{0.3\linewidth}
				\centering
				\includegraphics[width=\textwidth]{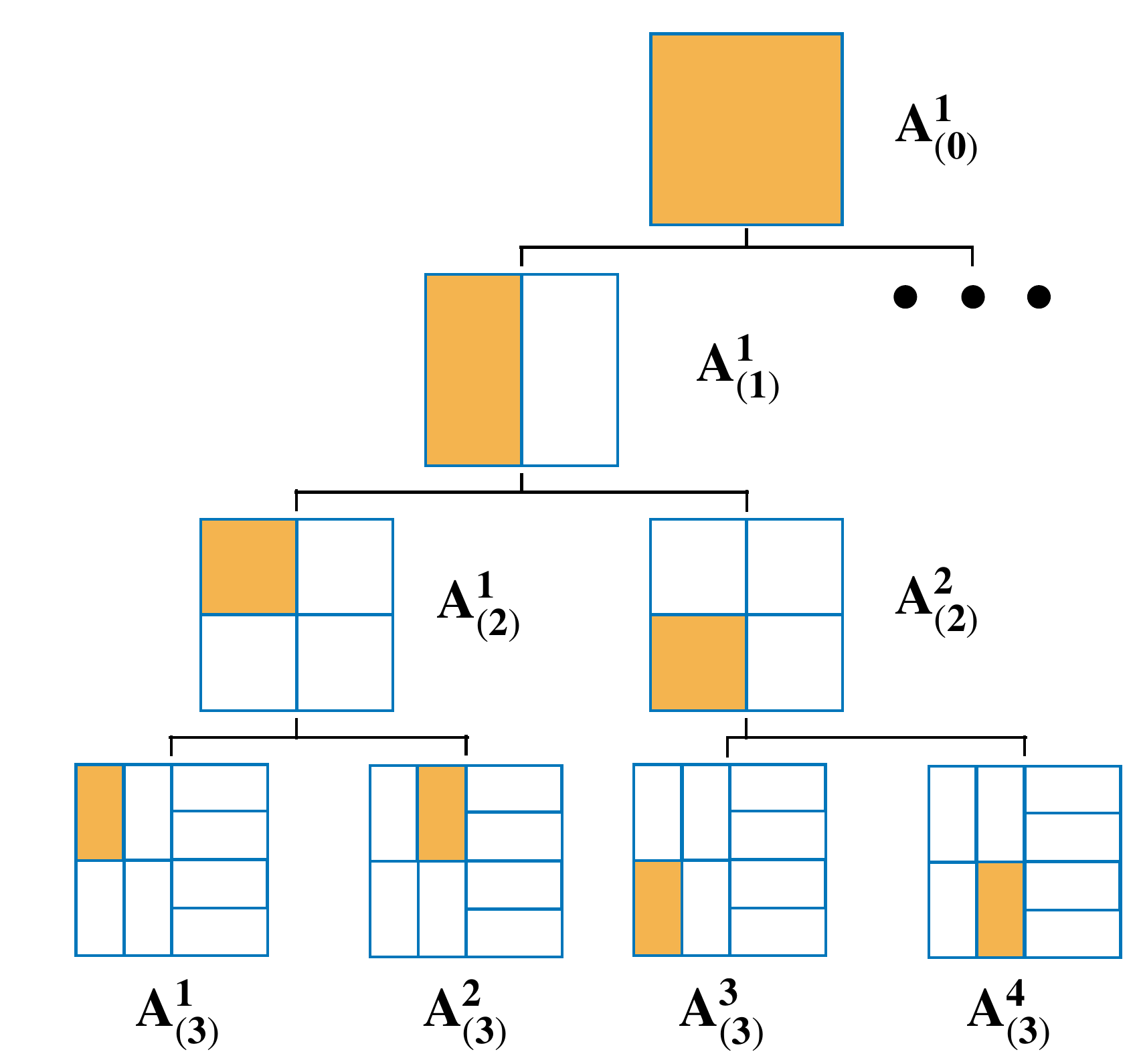}
			\end{minipage}
			\label{fig:pruning1}
		}
		\subfigure[During pruning]{
			\begin{minipage}{0.3\linewidth}
				\centering
				\includegraphics[width=\textwidth]{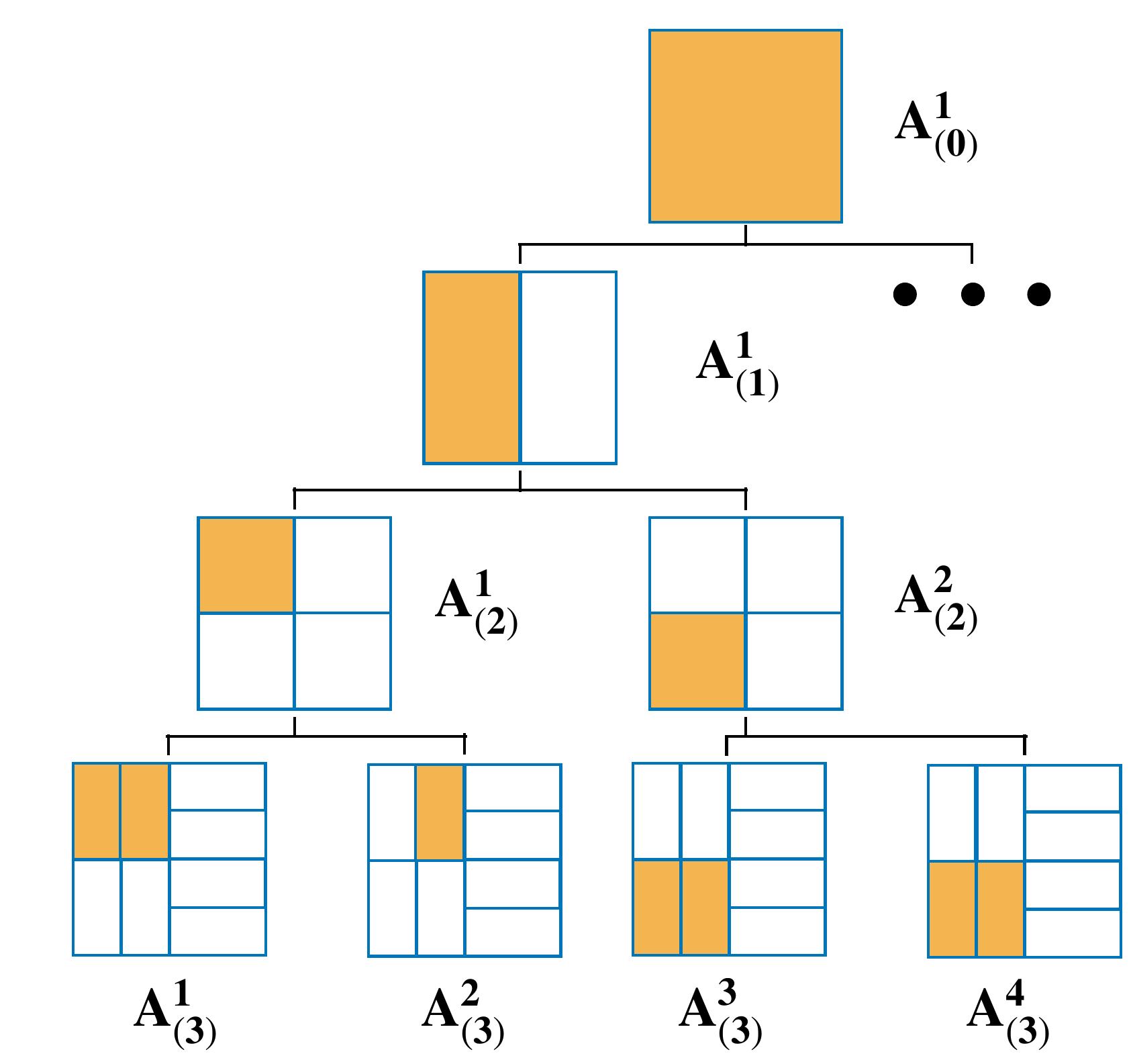}
			\end{minipage}
			\label{fig:pruning2}
		}
		\subfigure[Pruned estimation]{
			\begin{minipage}{0.3\linewidth}
				\centering
				\includegraphics[width=\textwidth]{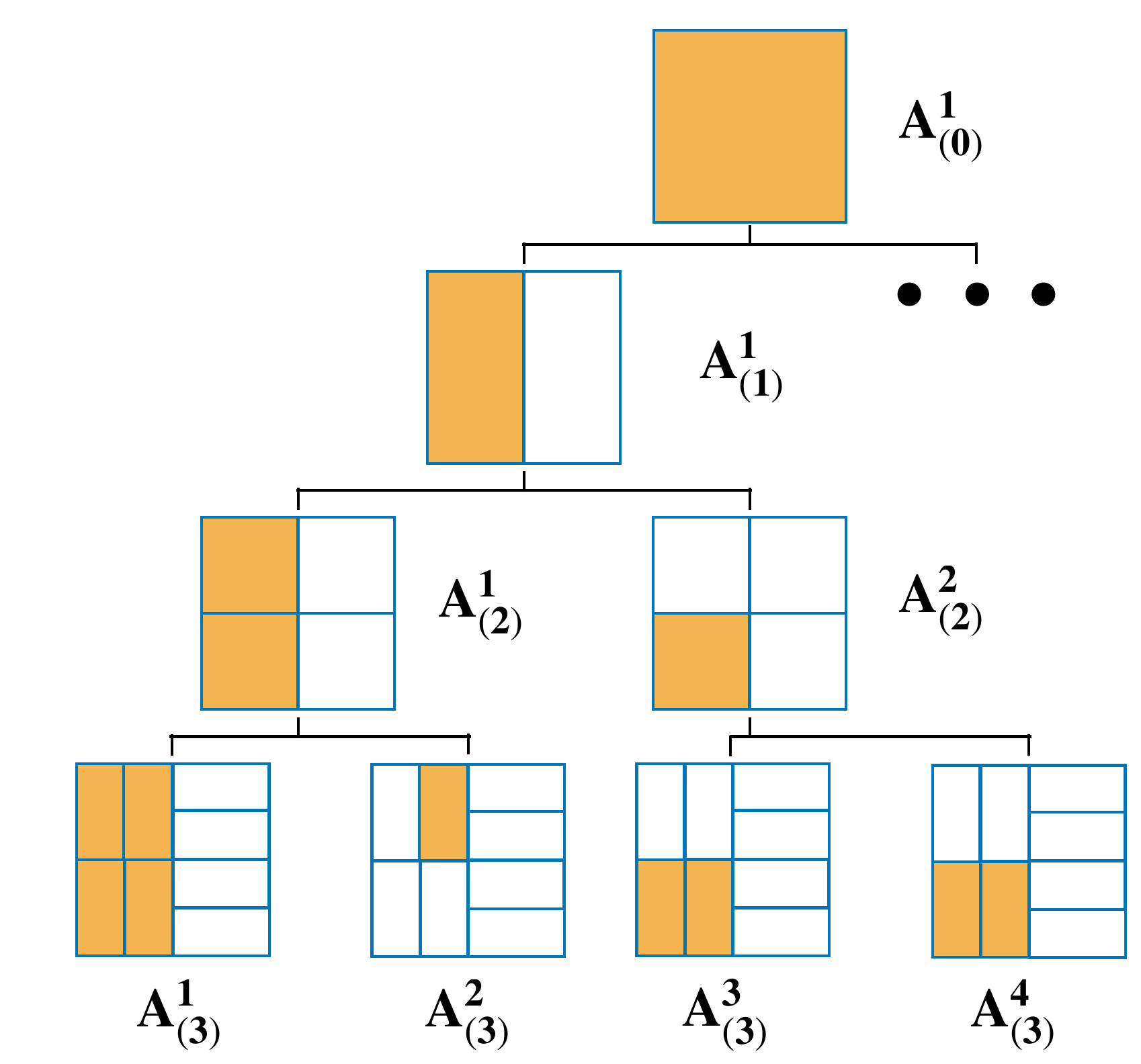}
			\end{minipage}
			\label{fig:pruning3}
		}
		\caption{Illustration of pruning process. The yellow area filled at node $A_{(i)}^j$ means that for $x\in A_{(i)}^j$, we use the average of labels in the yellow area instead of $A_{(i)}^j$ itself. 
			The prediction in $A_{(3)}^1$ is pruned to its depth - 1 ancestor; the prediction in $A_{(3)}^3$ is pruned to its depth - 2 ancestor; the prediction in $A_{(3)}^2$ is pruned to its depth - 3 ancestor, i.e. not pruned. }
		\label{fig:pruning}
		\vskip -0.1in
	\end{figure}

	\begin{proposition}\label{prop:lepskibound}
		Let $\pi$ be constructed through the max-edge partition in Section \ref{sec:partition} with depth $p_0$.
		Let $\tilde{\eta}^{\pi_k}_{\mathrm{P}, \mathrm{Q}}$ be the tree estimator with depth $k$ ancestor node  in \eqref{equ:defofetapik}. 
		Let $\widehat{\eta}^{\pi_k}_{\mathrm{P}, \mathrm{Q}}$ be the non-private tree estimator with partition $\pi_k$ in \eqref{equ:twosamplenonprivateeta}. Define the quantity 
		\begin{align}\label{equ:rkj}
			r_{k}^j :=\sqrt{\frac{4C_1\left(2^{p_0 - k + 3}\varepsilon^{-2} n_{\mathrm{P}}\right)\vee \left( \sum_{j'}\sum_{i=1}^{n_{\mathrm{P}}}\tilde{U}_i^{\mathrm{P}j'} \right) + C_1 \lambda^2\sum_{j'}\sum_{i=1}^{n_{\mathrm{P}}}U_i^{\mathrm{Q}j'}}{\left( \sum_{j'} \sum_{i=1}^{n_{\mathrm{P}}} \tilde{U}^{\mathrm{P} j}_i + \lambda \sum_{j'} \sum_{i=1}^{n_{\mathrm{Q}}}{U}^{\mathrm{Q} j}_i\right)^2}}
		\end{align}
		where the summation over $j'$ is with respect to all descendent nodes of the depth-$k$ ancestor, i.e. $\mathrm{ancestor}(A^{j'}_{(p_0)}, k) = \mathrm{ancestor}(A^{j}_{(p_0)}, k)$.
		The constant $C_1$ is specified in the proof. 
		Then, for all $\lambda$ and ${p_0}$, if $x\in A^{j}_{(p_0)}$, there holds
		\begin{align*}
			\left|\tilde{\eta}^{\pi_k}_{\mathrm{P}, \mathrm{Q}}(x)  - \mathbb{E}_{Y|X}\left[\widehat{\eta}^{\pi_k}_{\mathrm{P}, \mathrm{Q}}(x) \right]\right| \leq r_k^j
		\end{align*}
		with probability $1-3/(n_{\mathrm{P}} + n_{\mathrm{Q}})^2 $ with respect to $\mathrm{P}_{Y^{\mathrm{P}} | X^{\mathrm{P}}}^{n_{\mathrm{P}}} \otimes \mathrm{Q}_{Y^{\mathrm{Q}} | X^{\mathrm{Q}}}^{n_{\mathrm{Q}}}\otimes \mathrm{R}$ where $\mathrm{R}$ is the joint distribution of privacy mechanisms in \eqref{equ:privatizeprocedureU} and \eqref{equ:privatizeprocedureV}.
		$\mathbb{E}_{Y|X}$ is taken with respect to $\mathrm{P}_{Y|X}^{n_\mathrm{P}}\otimes \mathrm{Q}_{Y|X}^{n_\mathrm{Q}}$.
	\end{proposition}

	The above proposition indicates that if $\tilde{\eta}^{\pi_k}_{\mathrm{P}, \mathrm{Q}}(x) - {1}/{2} \geq r_k^j$,
	then there holds
	\begin{align*}
		\mathbb{E}_{Y|X}\left[\widehat{\eta}^{\pi_k}_{\mathrm{P}, \mathrm{Q}}(x)\right] - \frac{1}{2} =  \mathbb{E}_{Y|X}\left[\widehat{\eta}^{\pi_k}_{\mathrm{P}, \mathrm{Q}}(x)\right] - \tilde{\eta}^{\pi_k}_{\mathrm{P}, \mathrm{Q}}(x) + \tilde{\eta}^{\pi_k}_{\mathrm{P}, \mathrm{Q}}(x) - \frac{1}{2} \geq 0
	\end{align*}
	with probability $1-3/(n_{\mathrm{P}} + n_{\mathrm{Q}})^2 $. 
	Similar conclusion holds when $\tilde{\eta}^{\pi_k}_{\mathrm{P}, \mathrm{Q}}(x)  - {1}/{2} < r_k^j$.
	In other words, when we have
	\begin{align}\label{equ:pruningcondition}
		\frac{|\tilde{\eta}^{\pi_k}_{\mathrm{P}, \mathrm{Q}}(x)  - \frac{1}{2}|}{r_k^j } \geq 1,
	\end{align}
	the population version of ancestor estimation is of the same sign as the sample version estimation.
	As long as \eqref{equ:pruningcondition} holds, it is enough to use the sample version estimation, and our goal is to find the best $k$ for $\mathbb{E}_{Y|X}[\widehat{\eta}^{\pi_k}_{\mathrm{P}, \mathrm{Q}}(x)]$ to approximate $\eta_{\mathrm{P}}^*(x)$. 
	Under continuity assumption \ref{asp:alphaholder}, the approximation error is bounded by the radius of the largest cell (see Lemma \ref{lem::treeproperty}). 
	Consequently, one can show that $\mathcal{R}_{\mathrm{P}}(\eins(\mathbb{E}_{Y|X}[\widehat{\eta}^{\pi_k}_{\mathrm{P}, \mathrm{Q}}(x)] > {1}/{2}))$ is monotonically decreasing with respect to $k$. 
	Thus, our goal is to find the largest $k$ such that \eqref{equ:pruningcondition} holds.
	On the other hand, \eqref{equ:pruningcondition} is dependent on $\lambda$.
	We select the best possible $\lambda$ in order to let \eqref{equ:pruningcondition} hold for each $j$ and $k$, i.e. 
	\begin{align}\label{equ:bestlamda}
		\lambda^j_k = {\arg\max}_{\lambda} \frac{|\tilde{\eta}^{\pi_k}_{\mathrm{P}, \mathrm{Q}}(x)  - {1}/{2}|}{r_k^j}
	\end{align}
	for $x\in A^j_{(p_0)}$. 
	The optimization problem in \eqref{equ:bestlamda} has a closed-form solution that can be computed efficiently and explicitly. The derivation of the closed-form solution is postponed to Section \ref{sec:derivationpruning}.

	Based on the above analysis, we can perform the following pruning procedure.
	We first query with a sufficient depth, i.e. we select $p_0$ large enough, create a partition on $D^{\mathrm{Q}}$, and receive the information from data holders. 
	Then we prune back by the following procedure:
	\begin{itemize}
		\item For $j \in \mathcal{I}_{p_0}$, we do the following operation: for $k = p_0, \cdots ,2$ and $r_k^j$ defined in \eqref{equ:rkj}, we assign $\lambda_k^j$ as in \eqref{equ:bestlamda}. 
		Then if \eqref{equ:pruningcondition} holds, let $k_j = k$; else, $k \leftarrow k - 1$. 
		\item If $k_j$ is not assigned, we select $k_j = \arg\max_k  {|\tilde{\eta}^{\pi_k}_{\mathrm{P}, \mathrm{Q}}(x)  - {1}/{2}|}/{r_k^j}$. 
		\item Assign $\tilde{\eta}^{{prune}}_{\mathrm{P}, \mathrm{Q}}(x) = \tilde{\eta}^{\pi_{k_j}}_{\mathrm{P}, \mathrm{Q}}(x)$ for all $x\in A_{(p_0)}^j$. 
	\end{itemize}

	This process can be done efficiently. The exact process is illustrated in Figure \ref{fig:pruning}. 
	We first calculate estimations as well as the relative signal strength at all nodes. 
	Then we trace back from each leaf node to the root and find the node that maximizes the statistic \eqref{equ:pruningcondition}. 
	The prediction value at the leaf node is assigned as the prediction at the ancestor node with the maximum statistic value.
	In total, the pruning causes additional time complexity at most $\mathcal{O}(2^{p_0})$, which is ignorable compared to the original complexity as long as $2^{p_0}\lesssim n_{\mathrm{P}} + d n_{\mathrm{Q}}$.

	\subsection{Pruned LDP Classification Tree}\label{sec:plpct}
	
	In the non-private case, i.e. when $\varepsilon = \infty$, $\tilde{\eta}^{{prune}}_{\mathrm{P}, \mathrm{Q}}$ is essentially the Lepski's method \citep{lepskii1992asymptotically, cai2021transfer}. 
	However, in the presence of privacy concerns, a tricky term $2^{p_0 - k + 3}\varepsilon^{-2} n_{\mathrm{P}}$ arises in \eqref{equ:rkj} due to the Laplace noises.
	Specifically, if $2^{p_0 - k + 3}\varepsilon^{-2} n_{\mathrm{P}}\geq \sum_{j'}\sum_{i=1}^{n_{\mathrm{P}}}\tilde{U}_i^{\mathrm{P}j'} $,  the first term in the numerator of \eqref{equ:rkj} contains no information other than the level of noise. 
	As a result, the theoretical analysis of Lepski's method fails to adapt.

	To deal with the issue, we simultaneously perform two pruning procedures using information of solely $\mathrm{P}$ or $\mathrm{Q}$ data.
	We first query with a large depth $p_0 = \lfloor\frac{d}{2 + 2d} \log_2 (n_{\mathrm{P}} \varepsilon^2 + n_{\mathrm{Q}}^{\frac{2 + 2d}{d}}) \rfloor$ and half of the budget $\varepsilon / 2$.
	When $2^{p_0 - k + 3}\varepsilon^{-2} n_{\mathrm{P}}\geq \sum_{j'}\sum_{i=1}^{n_{\mathrm{P}}}\tilde{U}_i^{\mathrm{P}j'} $, we compare the values of $|\tilde{\eta}_{\mathrm{P}, \mathrm{Q}}^{\pi_k}(x) - 1 / 2| / r_k^j$ with $\lambda = 0$ and $\lambda  = \infty$, associating to $\mathrm{P}$ and $\mathrm{Q}$ data, respectively. 
	If either of the statistics is larger than 1, we terminate the pruning.
	Else, we compare the values of the statistics.
	If the statistic associated with $\mathrm{Q}$ is larger, we return the estimation $\widehat{f}^{\pi_k}_{\mathrm{Q}}$.
	In this case, we guarantee that the pruned classifier is at least as good as the estimation using only $\mathrm{Q}$ data. 
	If the statistic associated with $\mathrm{P}$ is larger, we proceed the pruning to $k - 1$.
	If we reach $k \leq  \lfloor \frac{d}{2 + 2d} \log_2 (n_{\mathrm{P}} \varepsilon^2 ) \rfloor$ for any $j$, we terminate the whole algorithm and return \eqref{equ:etadppi} with depth $ \lfloor \frac{d}{2 + 2d} \log_2 (n_{\mathrm{P}} \varepsilon^2 ) \rfloor$, which costs another $\varepsilon / 2$.
	In this case, we ensure $\mathrm{Q}$ data is unimportant and adopt a terminating condition that is close to the optimal value $\lfloor \frac{d}{2\alpha + 2d} \log_2 (n_{\mathrm{P}} \varepsilon^2 ) \rfloor$.
	The detailed algorithm is presented in Algorithm \ref{alg:prunedeta}.
	Moreover, we show that the adjusted pruning procedure will result in a near-optimal excess risk.

	\begin{algorithm}[!p]
		\caption{Pruned LPCT}
		\label{alg:prunedeta}
		\label{alg:ldpdt}
		
		{\bfseries Input: }{ Private data ${D}^{\mathrm{P}} = \{ ( X_i^{\mathrm{P}}, Y_i^{\mathrm{P}} )\}_{i=1}^{n_{\mathrm{P}}}$, public data ${D}^{\mathrm{Q}} = \{ ( X_i^{\mathrm{Q}}, Y_i^{\mathrm{Q}} )\}_{i=1}^{n_{\mathrm{Q}}}$.  }\\
		{\bfseries Initialization: }{ $p_0= \lfloor \frac{d}{2 + 2d} \log_2 (n_{\mathrm{P}} \varepsilon^2 + n_{\mathrm{Q}}^{\frac{2 + 2d}{d}}) \rfloor$, \texttt{flag} $=0$. }\\
		Create $\pi$ with depth $p_0$ on $D^{\mathrm{Q}}$. Query \eqref{equ:privatizeprocedureU} and \eqref{equ:privatizeprocedureV} with $\varepsilon / 2$.
		\\
		\For{$A_{(p_0)}^j$ in $\pi$}{
			\For{$k$ = $p_0$, $\cdots$, 1}{
				\eIf{$2^{p_0 - k + 3}\varepsilon^{-2} n_{\mathrm{P}}\geq \sum_{j'}\sum_{i=1}^{n_{\mathrm{P}}}\tilde{U}_i^{\mathrm{P}j'} $}{
					$r_k^{\mathrm{Q}j} = \sqrt{\frac{4 \log (n_{\mathrm{P}} + n_{\mathrm{Q}} )}{\sum_{j'}\sum_{i=1}^{n_{\mathrm{P}}}U_i^{\mathrm{Q}j'}}}$,
					$v_k^{\mathrm{Q}j} = |\widehat{\eta}_{\mathrm{Q}}^{\pi_k} - {1}/{2}| /r_k^{\mathrm{Q}j} $
					\\
					$r_k^{\mathrm{P}j} = \sqrt{\frac{2^{p_0 - k + 5}\varepsilon^{-2} n_{\mathrm{P}} \log (n_{\mathrm{P}} + n_{\mathrm{Q}} )}{\left(\sum_{j'}\sum_{i=1}^{n_{\mathrm{P}}}\tilde{U}_i^{\mathrm{P}j'}\right)^2}}$,
					$v_k^{\mathrm{P}j} = |\tilde{\eta}_{\mathrm{P}}^{\pi_k} - {1}/{2}| /r_k^{\mathrm{P}j} $\\
					\eIf{$v_k^{\mathrm{Q}j}\leq v_k^{\mathrm{P}j}$}{
						$s^j_k = \tilde{\eta}_{\mathrm{P}}^{\pi_k}$, $r_k^j = r_k^{\mathrm{P}j}$, $v_k^j = v_k^{\mathrm{P}j}$\\
						\If{ $k \leq \lfloor \frac{d}{2 + 2d} \log_2 (n_{\mathrm{P}} \varepsilon^2 ) \rfloor$ }{ $k_j = k$,  \texttt{flag} $ = 1$, \textbf{break}}
					}{
						$s^j_k = \widehat{\eta}_{\mathrm{Q}}^{\pi_k}$, $r_k^j = r_k^{\mathrm{Q}j}$, $v_k^j = v_k^{\mathrm{Q}j}$
					}
				}{
					$r_k^j = \sqrt{\frac{\left(32\sum_{j'}\sum_{i=1}^{n_{\mathrm{P}}}\tilde{U}_i^{\mathrm{P}j'} + 4 \lambda^2\sum_{j'}\sum_{i=1}^{n_{\mathrm{P}}}U_i^{\mathrm{Q}j'} \right) \log (n_{\mathrm{P}} + n_{\mathrm{Q}} )}{\left(\sum_{j'}\sum_{i=1}^{n_{\mathrm{P}}} \tilde{U}^{\mathrm{P} j}_i + \lambda\sum_{j'}\sum_{i=1}^{n_{\mathrm{Q}}}{U}^{\mathrm{Q} j'}_i\right)^2}}$\\
					$v_k^j = \max_{\lambda}|\tilde{\eta}_{\mathrm{P}, \mathrm{Q}}^{\pi_k} - {1}/{2}| /r_k^j $, $s_k^j = \tilde{\eta}_{\mathrm{P}, \mathrm{Q}}^{\pi_k}$  with $\lambda  = \lambda_k^j$ in \eqref{equ:bestlamda}
				}
				\If{$v_k^j\geq 1$}{
					$k_j = k$, 
					\textbf{break} 
				}
			}
		}
		\eIf{$\mathtt{flag} $}{
			$p = \lfloor \frac{d}{2 + 2d} \log_2 (n_{\mathrm{P}} \varepsilon^2 ) \rfloor$. 
			Create $\pi_p$ on $D^{\mathrm{Q}}$.  Query \eqref{equ:privatizeprocedureU} and \eqref{equ:privatizeprocedureV} with $\varepsilon / 2$.\\
			{\bfseries Output:} $\tilde{\eta}_{\mathrm{P}, \mathrm{Q}}^{prune} = \tilde{\eta}_{\mathrm{P}}^{\pi_p}$.	
		}{{\bfseries Output: }{$\tilde{\eta}_{\mathrm{P}, \mathrm{Q}}^{prune} = \sum_j s_k^j \cdot \eins\{A_{(p_0)}^j\}$. }}
	\end{algorithm} 
	\vskip -0.3in

	\begin{theorem}\label{thm:utility3}
		Let $\tilde{f}^{prune}_{\mathrm{P}, \mathrm{Q}} = \eins (\tilde{\eta}_{\mathrm{P}, \mathrm{Q}}^{prune} \geq 1 / 2)$ be the pruned LPCT defined in Algorithm \ref{alg:ldpdt}. 
		Suppose Assumption \ref{asp:alphaholder}, \ref{asp:margin}, and \ref{asp:rse} hold.
		Let $\delta$ be defined in \eqref{equ:defofdelta}. 
		Then, for $n_{\mathrm{P}}^{-\frac{\alpha}{2\alpha + 2d}}\lesssim \varepsilon \lesssim n_{\mathrm{P}}^{\frac{d}{4 + 2d}}$, with probability $1 -4/(n_{\mathrm{P}} + n_{\mathrm{Q}})^2 $ with respect to $\mathrm{P}^{n_{\mathrm{P}}}\otimes\mathrm{Q}^{n_{\mathrm{Q}}}\otimes\mathrm{R}$, we have
		\begin{enumerate}
			\item[(i)] if $\frac{n_{\mathrm{P}} \varepsilon^2}{\log (n_{\mathrm{P}} + n_{\mathrm{Q}})} \lesssim \left(\frac{n_{\mathrm{Q}}}{\log (n_{\mathrm{P}} + n_{\mathrm{Q}})}\right)^{\frac{2\alpha + 2d}{2 \gamma\alpha + d}}$, $\mathcal{R}_{\mathrm{P}} (\tilde{f}^{prune}_{\mathrm{P}, \mathrm{Q}}) - \mathcal{R}_{\mathrm{P}}^*   \lesssim \delta^{\frac{\alpha(\beta+1)}{d}}$;
			\item[(ii)]  if $\frac{n_{\mathrm{P}} \varepsilon^2}{\log (n_{\mathrm{P}} + n_{\mathrm{Q}})} \gtrsim \left(\frac{n_{\mathrm{Q}}}{\log (n_{\mathrm{P}} + n_{\mathrm{Q}})}\right)^{\frac{2\alpha + 2d}{2 \gamma\alpha + d}}$, $\mathcal{R}_{\mathrm{P}} (\tilde{f}^{prune}_{\mathrm{P}, \mathrm{Q}}) - \mathcal{R}_{\mathrm{P}}^*   \lesssim    \delta^{\frac{\alpha(\alpha + d) ( \beta + 1)}{d(1+d)}}$.
		\end{enumerate}
	\end{theorem}

	The conclusion of Theorem \ref{thm:utility3} is twofold. 
	On one hand, when $\mathrm{Q}$ data dominates $\mathrm{P}$ data, i.e. $\frac{n_{\mathrm{P}} \varepsilon^2}{\log (n_{\mathrm{P}} + n_{\mathrm{Q}})} \lesssim (\frac{n_{\mathrm{Q}}}{\log (n_{\mathrm{P}} + n_{\mathrm{Q}})})^{\frac{2\alpha + 2d}{2 \gamma\alpha + d}}$, the pruned estimator remains rate-optimal. 
	In this case, the estimator is appropriately pruned to the correct depth.
	On the other hand, when $\mathrm{P}$ data dominates $\mathrm{Q}$ data, the convergence rate suffers degradation.
	The degradation factor $(\alpha + d) / ( 1 + d)$ is mild, and even diminishes when $\alpha = 1$, i.e. $\eta^*_{\mathrm{P}}$ is Lipschitz continuous. 
	In this case, there are two possible regions in which our conclusion is derived.
	If $(\frac{n_{\mathrm{Q}}}{\log (n_{\mathrm{P}} + n_{\mathrm{Q}})})^{\frac{2\alpha + 2d}{2 \gamma\alpha + d}} \lesssim (\frac{n_{\mathrm{P}} \varepsilon^2}{\log (n_{\mathrm{P}} + n_{\mathrm{Q}})} )^{\frac{2\alpha + 2d}{2  + d}}$, the estimator triggers termination and is computed with depth $ \lfloor \frac{d}{2 + 2d} \log_2 (n_{\mathrm{P}} \varepsilon^2 ) \rfloor$. 
	If the opposite, i.e. $\frac{n_{\mathrm{P}} \varepsilon^2}{\log (n_{\mathrm{P}} + n_{\mathrm{Q}})} \gtrsim (\frac{n_{\mathrm{Q}}}{\log (n_{\mathrm{P}} + n_{\mathrm{Q}})})^{\frac{2\alpha + 2d}{2 \gamma\alpha + d}} \gtrsim (\frac{n_{\mathrm{P}} \varepsilon^2}{\log (n_{\mathrm{P}} + n_{\mathrm{Q}})} )^{\frac{2\alpha + 2d}{2  + d}}$, pruning fails to identify the optimal depth $ \lfloor \frac{d}{2\alpha + 2d} \log_2 (n_{\mathrm{P}} \varepsilon^2 ) \rfloor$.
	Yet, it also does not trigger termination at $ \lfloor \frac{d}{2 + 2d} \log_2 (n_{\mathrm{P}} \varepsilon^2 ) \rfloor$.
	In this region, we take a shortcut to avoid intricate analysis, while the true convergence rate is more complex yet faster than that stated in Theorem \ref{thm:utility3}.
	Notably, one drawback of this shortcut is the discontinuity of the upper bound at the boundary $\frac{n_{\mathrm{P}} \varepsilon^2}{\log (n_{\mathrm{P}} + n_{\mathrm{Q}})} \asymp (\frac{n_{\mathrm{Q}}}{\log (n_{\mathrm{P}} + n_{\mathrm{Q}})})^{\frac{2\alpha + 2d}{2 \gamma\alpha + d}}$.

	Algorithm \ref{alg:prunedeta} lacks true adaptiveness, given that its upper bound suffers a significant loss factor of ${(\alpha + d)}/{(1+d)}$.
	Conventional adaptive methods typically incur a loss factor of logarithmic order \citep{butucea2020local, cai2021transfer}. 
	Following the approach outlined in \cite{butucea2020local}, one potential avenue for improvement involves fitting multiple trees for $p = \lfloor \frac{d}{2+2d}\log_2 \left(n_{\mathrm{P}} \varepsilon^2\right)\rfloor,\cdots, \lfloor\frac{d}{2 + 2d} \log_2 (n_{\mathrm{P}} \varepsilon^2 + n_{\mathrm{Q}}^{\frac{2 + 2d}{d}}) \rfloor$ and designing a test statistic to zero out the suboptimal candidates.

	In summary, this proposed scheme offers benefits from two perspectives:
	(1) It avoids multiple queries to the data. 
	All data holders send their privatized messages at most twice. 
	(2) It selects sensitive parameters $p$ and $\lambda$ in a data-driven manner. 
	In sacrifice, the bound of excess risk is reasonably loosened compared to the model with the best parameters, and one more query to the private data is required.

	\section{Experiments on Synthetic Data}\label{sec:simulation}
	
	We validate our theoretical findings by comparing LPCT with several of its variants on synthetic data in this section. All experiments are conducted on a machine with 72-core Intel Xeon 2.60GHz and 128GB memory.
	Reproducible codes are available on GitHub\footnote{https://github.com/Karlmyh/LPCT.
	}.

	\subsection{Simulation Design}

	We motivate the design of simulation settings.
	While conducting experiments across various combinations of $\gamma$, $n_{\mathrm{P}}$, and $n_{\mathrm{Q}}$, our primary focus is on two categories of data: (1) both $n_{\mathrm{Q}}$ and $\gamma$ are small, i.e. we have a small amount of high-quality public data (2) both $n_{\mathrm{Q}}$ and $\gamma$ are large, i.e. we have a large amount of low-quality public data.
	When $\gamma$ is small but $n_{\mathrm{Q}}$ is large, using solely public data yields sufficient performance, rendering private estimation less meaningful.
	Conversely, when $\gamma$ is large but $n_{\mathrm{Q}}$ is small, incorporating public data into the estimation offers limited assistance.
	Moreover, we require $n_{\mathrm{P}}$ to be much larger than $n_{\mathrm{Q}}$ since practical data collection becomes much easier with a privacy guarantee.

	We consider the following pair of distributions $\mathrm{P}$ and $\mathrm{Q}$.
	The marginal distributions $\mathrm{P}_X = \mathrm{Q}_X$ are both uniform distributions on $\mathcal{X} = [0,1]^2$.
	The regression function of $\mathrm{P}$ is 
	\begin{align*}
		\eta_{\mathrm{P}}(x) = \frac{1}{2} + \mathrm{sign}(x^1 - \frac{1}{3}) \cdot \mathrm{sign}(x^2 - \frac{1}{3})
		\cdot
		\left(\left|\frac{(3x^1-1)(3x^2-1)}{10}\right|
		\cdot \max\left(0, 1 - \left|2x^2 - 1\right|\right)\right)^{\frac{1}{10}}
	\end{align*}
	while the regression function of $\mathrm{Q}$ with parameter $\gamma$ is
	$$\eta_{\mathrm{Q}}(x) = \frac{1}{2} + \frac{2}{5}\cdot \left(\frac{5}{2}\cdot \left(\eta_{\mathrm{P}}(x) - \frac{1}{2}\right)\right)^{\gamma}.$$
	It can be easily verified that the constructed distributions above satisfy Assumption \ref{asp:alphaholder}, \ref{asp:margin}, and \ref{asp:rse}.
	Throughout the following experiments, we take $\gamma$ in 0.5 and 5.
	For better illustration, their regression functions are plotted in Figure \ref{fig:distribution}.
	\begin{figure}[h]
		\centering
		\subfigure[$\eta_{\mathrm{P}}$]{
			\begin{minipage}{0.31\linewidth}
				\centering
				\includegraphics[width=\textwidth]{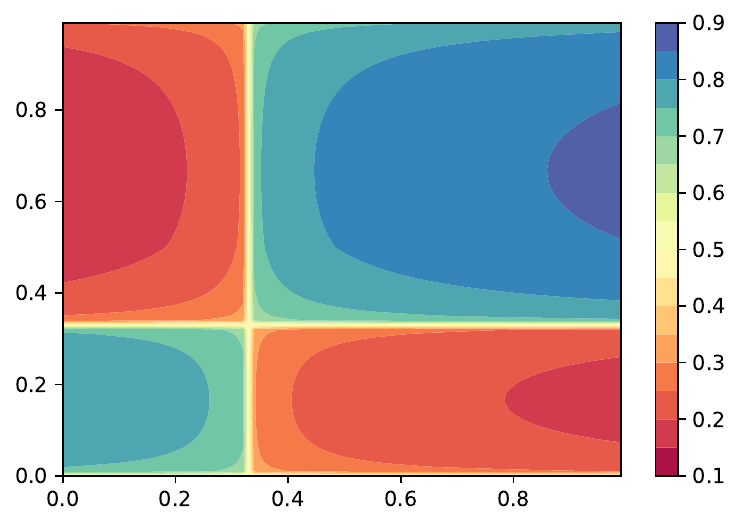}
			\end{minipage}
			\label{fig:distribution1}
		}
		\subfigure[$\eta_{\mathrm{Q}}$ with $\gamma = 0.5$]{
			\begin{minipage}{0.31\linewidth}
				\centering
				\includegraphics[width=\textwidth]{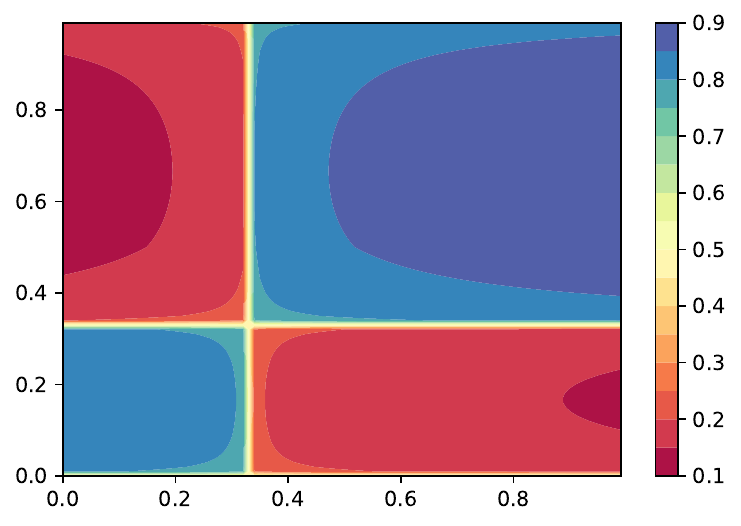}
			\end{minipage}
			\label{fig:distribution2}
		}
		\subfigure[$\eta_{\mathrm{Q}}$ with $\gamma = 5$]{
			\begin{minipage}{0.31\linewidth}
				\centering
				\includegraphics[width=\textwidth]{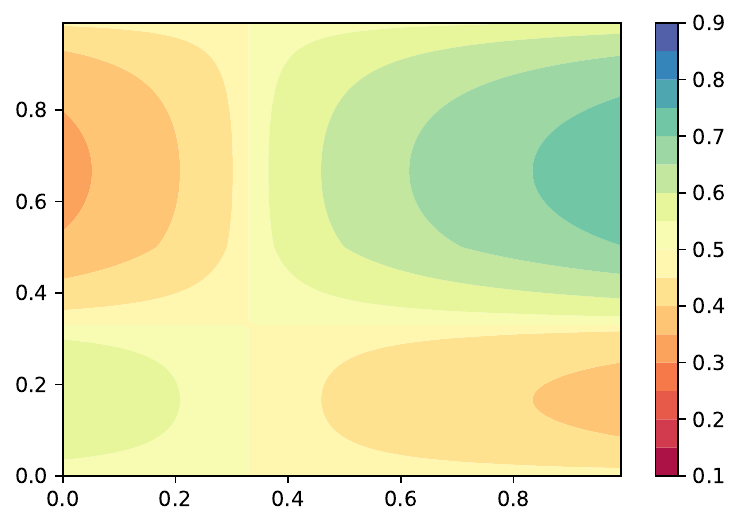}
			\end{minipage}
			\label{fig:distribution3}
		}
		\caption{Contour plots of the simulation distributions. }
		\label{fig:distribution}
		\vskip -0.1in
	\end{figure}

	We conduct experiments with privacy budgets $\varepsilon \in \{0.5, 1,2,4,8,1000\}$.
	Here, $\varepsilon = 1000$ can be regarded as the non-private case which marks the performance limitation of our methods.
	The other budget values cover all commonly seen magnitudes of privacy budgets.
	We take classification accuracy as the evaluation metric.
	In the simulation studies, the comparison methods and their abbreviation are as follows:
	\begin{itemize}
		\item \textbf{LPCT} is the proposed algorithm with the max-edge partition rule. We select the parameters among $p\in\{1,\cdots,8\}$ and $\lambda\in\{0.1$ $, 0.5, 1, 2, 5, 10, 50, 100, 200,$ $ 300, 400, $ $ 500, 750, 1000, 1250, $ $ 1500, 2000\}$.
		We choose the Gini index as the reduction criterion.
		\item \textbf{LPCT-P} and \textbf{LPCT-Q} are the estimation associated to $\mathrm{P}$ and $\mathrm{Q}$ data, respectively.
		Their parameter grids are the same as LPCT.
		\item \textbf{LPCT-R}
		is the proposed algorithm using the random max-edge partition \citep{cai2023extrapolated}.
		Specifically, we randomly select an edge among the longest edges to cut at the midpoint.
		Its parameter grids are the same as LPCT.
		\item \textbf{LPCT-prune} is the pruned locally private classification tree in Algorithm \ref{alg:prunedeta}.
		We use constants that yield a tighter bound than \eqref{equ:rkj}.
	\end{itemize}
	For each model, we report the best result over its parameter grids, with the best result determined based on the average result of at least 100 replications.

	\subsection{Simulation Results}

	In Section \ref{sec:simulationunderlyingparameter} and \ref{sec:simulationsamplesize}, we analyze the influence of underlying parameters and sample sizes to validate the theoretical findings.
	In Section \ref{sec:simulationmodelparameter}, we analyze the parameters of LPCT to understand their behaviors. 
	We illustrate various ways that public data benefits the estimation in Section \ref{sec:simulationgainofpublicdata}.

	\subsubsection{Influence of underlying parameters}
	\label{sec:simulationunderlyingparameter}
	
	\paragraph{Privacy-utility trade-off}
	We analyze how the privacy budget $\varepsilon$ influences the quality of prediction in terms of accuracy.
	We generate 10,000 private samples, with 50 and 1,000 public samples for $\gamma = 0.5$ and $\gamma = 5$, respectively.
	The results are presented in Figure \ref{fig:epsilon-acc-1} and \ref{fig:epsilon-acc-2}.
	In both cases, the performances of all models increase as $\varepsilon$ increases.
	Furthermore, the performance of LPCT is consistently better than both LPCT-P and LPCT-Q, demonstrating the effectiveness of our proposed method in combining both data.
	LPCT-prune performs reasonably worse than LPCT.
	In the high privacy regime ($\varepsilon \leq 1$), LPCT-P is almost non-informative and our estimator performs similarly to LPCT-Q.
	In the low privacy regime ($\varepsilon \geq 8$), the performance of our estimator improves rapidly along with LPCT-P.
	In the medium regime, neither LPCT-P nor LPCT-Q obviously outperforms each other, and there is a clear performance gain from using LPCT.

	\begin{figure}[!t]
		\centering
		\subfigure[Privacy-utility trade-off with $\gamma = 0.5$ and $n_{\mathrm{Q}} = 50$]{
			\begin{minipage}{0.3\linewidth}
				\centering
				\includegraphics[width=\textwidth]{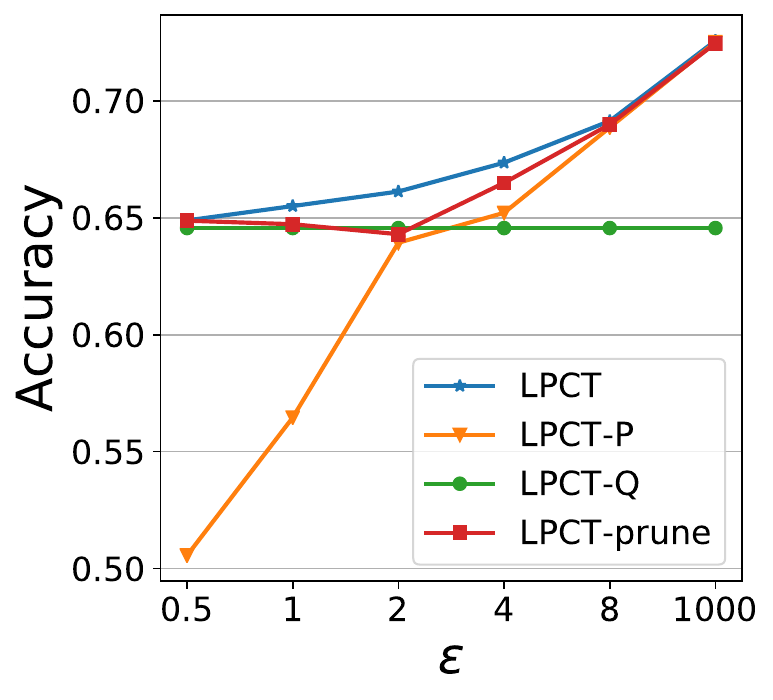}
			\end{minipage}
			\label{fig:epsilon-acc-1}
		}
		\subfigure[Privacy-utility trade-off with $\gamma = 5$ and $n_{\mathrm{Q}} = 1,000$]{
			\begin{minipage}{0.3\linewidth}
				\centering
				\includegraphics[width=\textwidth]{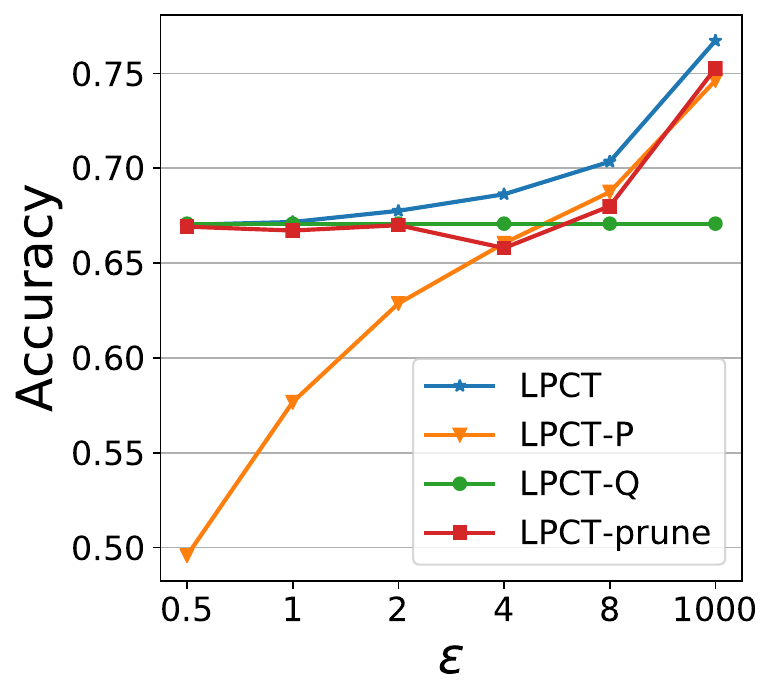}
			\end{minipage}
			\label{fig:epsilon-acc-2}
		}
		\subfigure[Accuracy - $\gamma$]{
			\begin{minipage}{0.3\linewidth}
				\centering
				\includegraphics[width=\textwidth]{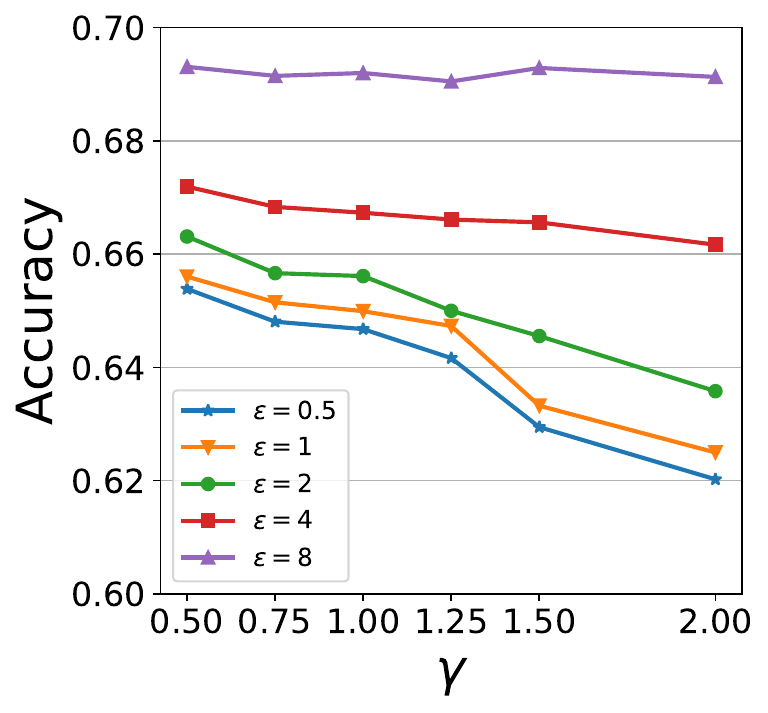}
			\end{minipage}
			\label{fig:gamma-acc-1}
		}
		\caption{Illustration of relationship between accuracy and underlying parameters, i.e. $\varepsilon$ and $\gamma$.  }
		\label{fig:epsilon-gamma-acc}
		\vskip -0.1in
	\end{figure}

	\paragraph{Influence of $\gamma$}
	We analyze how the relative signal exponent $\gamma$ influences the performance.
	We set $\gamma = 0.5$, $n_{\mathrm{P}} = 10,000$ and $n_{\mathrm{Q}} = 50$, while varying $\gamma$ in $\{0.5, 0.75, 1, 1.25, 1.5, 2\}$.
	In Figure \ref{fig:gamma-acc-1}, we can see that the accuracy of LPCT is decreasing with respect to $\gamma$ for all $\varepsilon$, which aligns with Theorem \ref{thm:utility}.
	Since $\gamma$ only influences the performance through $\mathrm{Q}$ data, its effect is less apparent when the ${\mathrm{P}}$ data takes the dominance, for instance $\varepsilon = 8$.

	\subsubsection{Influence of sample sizes}
	\label{sec:simulationsamplesize}

	\begin{figure}[!b]
		\centering
		\subfigure[Accuracy - $n_{\mathrm{P}}$]{
			\begin{minipage}{0.3\linewidth}
				\centering
				\includegraphics[width=\textwidth]{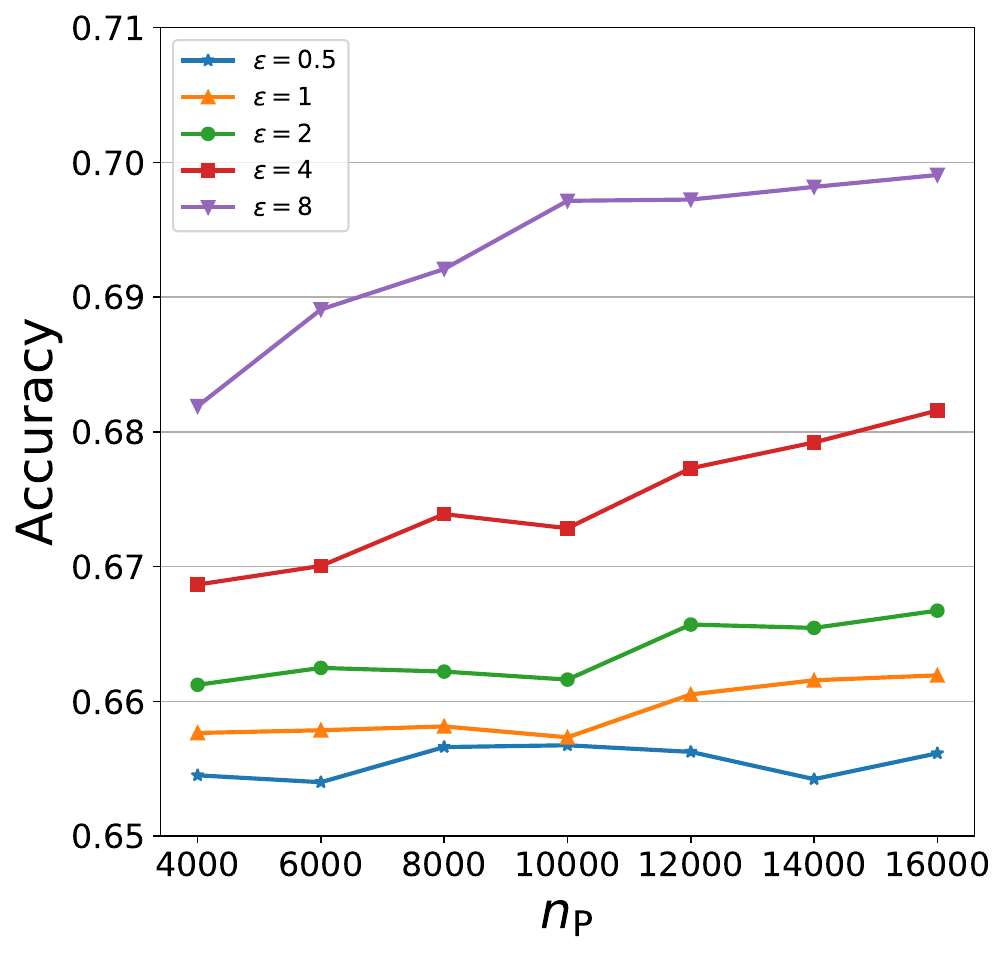}
			\end{minipage}
			\label{fig:np-acc}
		}
		\subfigure[With LPCT-P and LPCT-Q]{
			\begin{minipage}{0.3\linewidth}
				\centering
				\includegraphics[width=\textwidth]{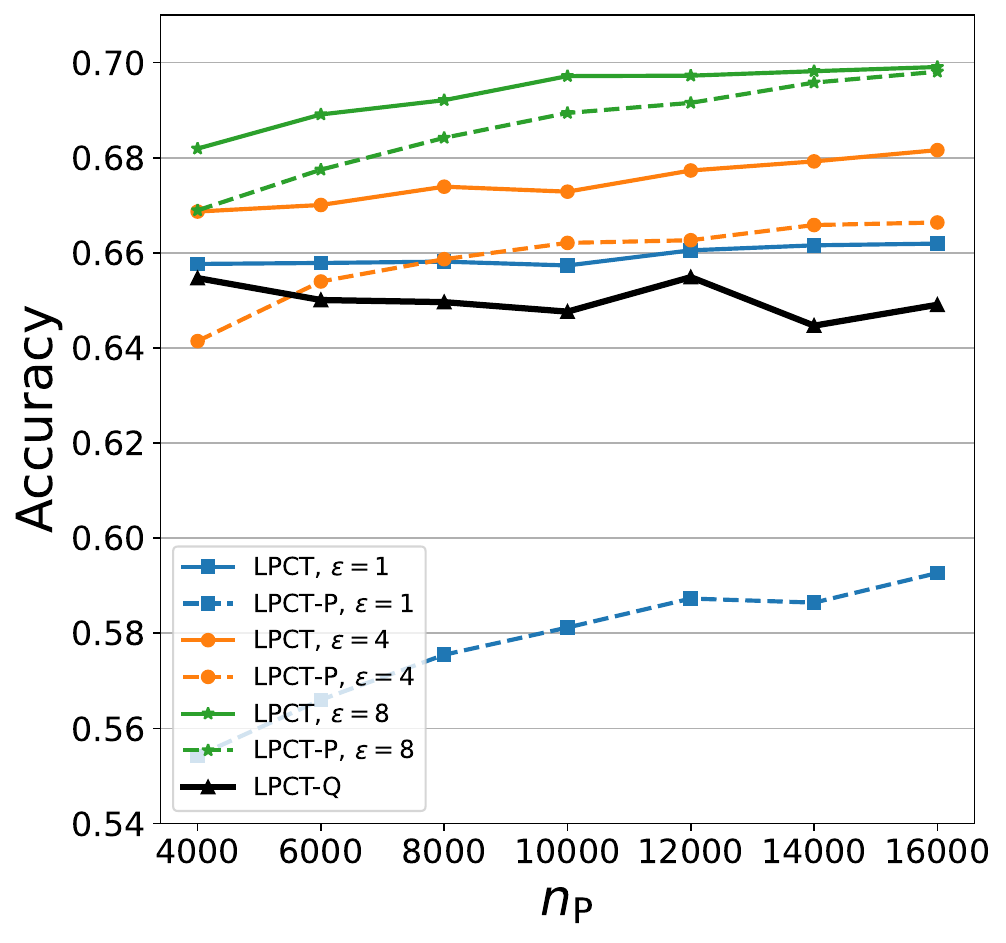}
			\end{minipage}
			\label{fig:np-acc-comparison}
		}
		\subfigure[With LPCT-prune]{
			\begin{minipage}{0.3\linewidth}
				\centering
				\includegraphics[width=\textwidth]{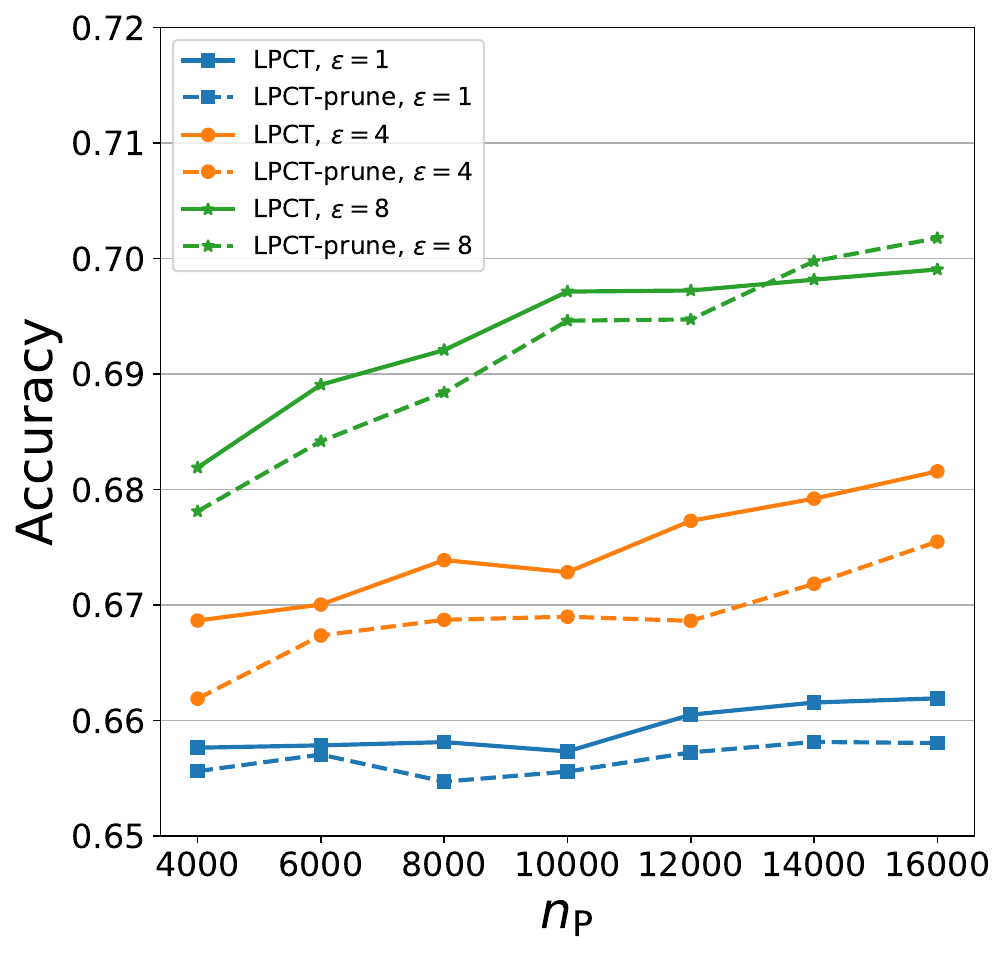}
			\end{minipage}
			\label{fig:np-acc-prune}
		}
		\caption{Analysis of accuracy with respect to $n_{\mathrm{P}}$. }
		\label{fig:np-analysis}
		\vskip -0.1in
	\end{figure}
	
	\paragraph{Accuracy - $\mathbf{n_{\mathrm{P}}}$}
	We conduct experiments to investigate the influence of the private sample size $n_{\mathrm{P}}$.
	We fix $n_{\mathrm{Q}} = 50$ and $\gamma = 0.5$ while varying $n_{\mathrm{P}}$ from $4,000$ to $16,000$.
	The results are presented in Figure \ref{fig:np-acc}.
	Additionally, we compare the performance of LPCT-P and LPCT-Q to analyze the contribution of both data sets, as shown in Figure \ref{fig:np-acc-comparison}.
	Moreover, we compare LPCT and LPCT-prune in Figure \ref{fig:np-acc-prune}.
	We refrain from invoking $\tilde{\eta}_{\mathrm{P}}^{\pi_p}$ even if the $\texttt{flag} = 1$ condition is triggered, as the termination induces sudden changes in the continuous curves, thereby rendering the findings less evident.
	In Figure \ref{fig:np-acc}, except for $\varepsilon = 0.5$, the accuracy improves as $n_{\mathrm{P}}$ increases for all $\varepsilon$ values.
	Moreover, the rate of improvement is more pronounced for higher values of $\varepsilon$.
	This observation aligns with the convergence rate established in Theorem \ref{thm:utility}.
	In \ref{fig:np-acc-comparison}, LPCT always outperforms methods using single data set, i.e. LPCT-P and LPCT-Q.
	The improvement of LPCT over LPCT-Q is more significant for large $\varepsilon$.
	In Figure \ref{fig:np-acc-prune}, we justify that LPCT-prune performs analogously to LPCT and its performance improves as $\varepsilon$ and $n_{\mathrm{P}}$ increase.

	\paragraph{Accuracy - $\mathbf{n_{\mathrm{Q}}}$}
	We conduct experiments to investigate the influence of the private sample size $n_{\mathrm{Q}}$.
	We keep $n_{\mathrm{P}} = 10,000$ and $\gamma = 0.5$ while varying $n_{\mathrm{Q}}$ from 10 to 200.
	The classification accuracy monotonically increases with respect to $n_{\mathrm{Q}}$ for all methods.
	In Figure \ref{fig:nq-acc}, the performance of LPCT-prune is closer to LPCT when $n_{\mathrm{Q}}$ gets larger.
	As explained in Section \ref{sec:plpct}, the pruning is appropriately conducted when $\mathrm{Q}$ data takes the dominance, and is disturbed when $\mathrm{P}$ data takes the dominance.
	In Figure \ref{fig:nq-acc-comparison}, it is worth noting that the performance improvement is more significant for small values of $n_{\mathrm{Q}}$, specifically when $n_{\mathrm{Q}}\leq 50$.
	Moreover, the performance of LPCT-P also improves in this region and remains approximately unchanged for larger $n_{\mathrm{Q}}$.
	This can be attributed to the information gained during partitioning, which is further validated in Figure \ref{fig:randompartition}.
	
	\begin{figure}[!t]
		\centering
		\subfigure[With LPCT-prune]{
			\begin{minipage}{0.41\linewidth}
				\centering
				\includegraphics[width=\textwidth]{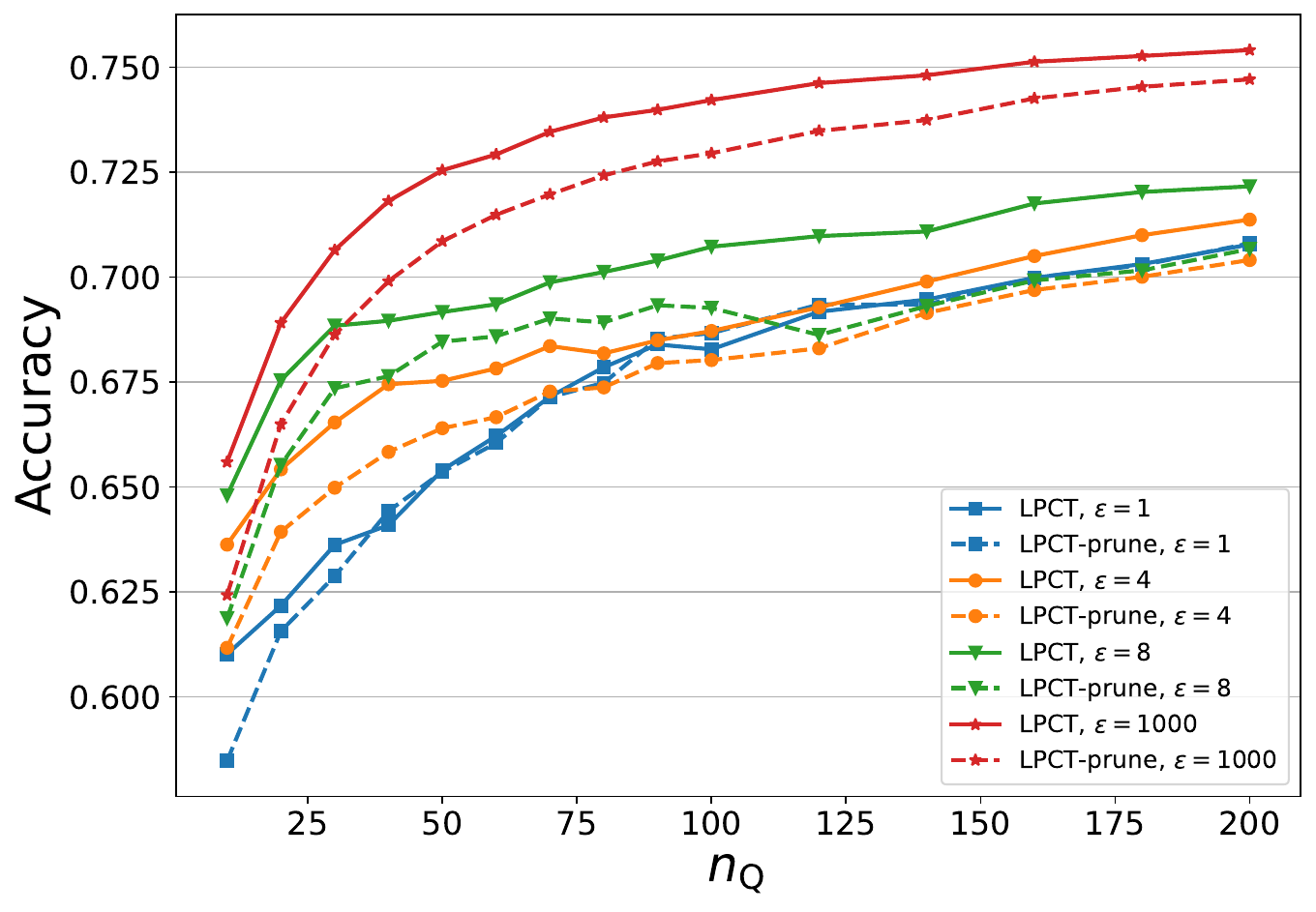}
			\end{minipage}
			\label{fig:nq-acc}
		}
		\hskip 0.2in
		\subfigure[With LPCT-P and LPCT-Q]{
			\begin{minipage}{0.41\linewidth}
				\centering
				\includegraphics[width=\textwidth]{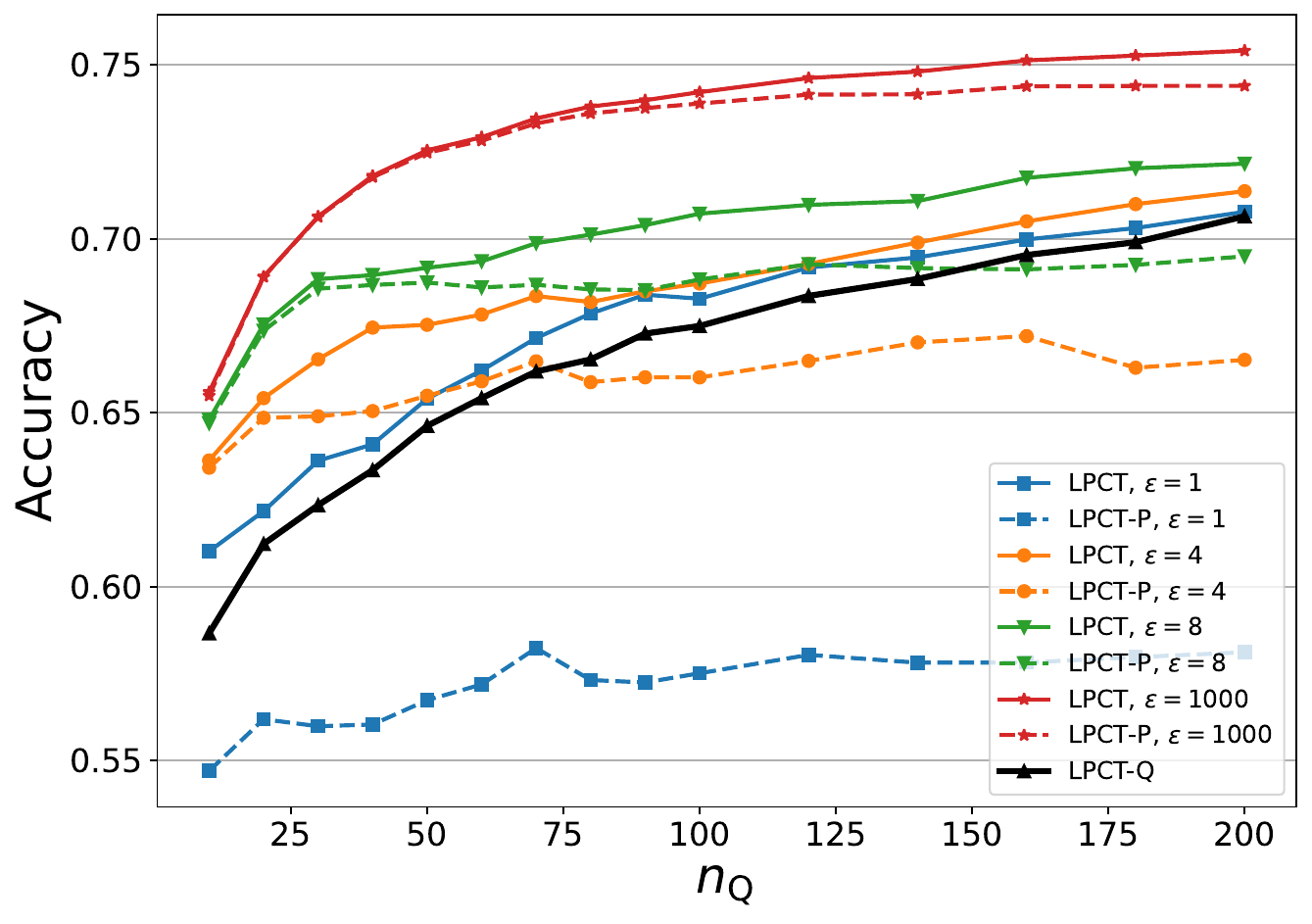}
			\end{minipage}
			\label{fig:nq-acc-comparison}
		}
		\caption{Analysis of accuracy with respect to $n_{\mathrm{Q}}$. }
		\label{fig:nq-analysis}
		\vskip -0.1in
	\end{figure}

	\subsubsection{Analysis of model parameters}
	\label{sec:simulationmodelparameter}

	\paragraph{Parameter analysis of depth}
	We conduct experiments to investigate the choice of depth.
	Under $(n_{\mathrm{P}}, n_{\mathrm{Q}}) = (10000, 50)$ and $\gamma = 0.5$, we plot the best performance of LPCT and LPCT-prune when fixing $p$ and $p_0$ to certain values, respectively in Figure \ref{fig:p-acc} and \ref{fig:p-acc-prune}.
	To avoid obscuring empirical findings, we refrain from invoking $\tilde{\eta}_{\mathrm{P}}^{\pi_p}$ even if $\texttt{flag} = 1$ is triggered, since the experiment is conducted with respect to $p_0$.
	In Figure \ref{fig:p-acc}, the accuracy first increases as $p$ increases until it reaches a certain value.
	Then the accuracy begins to decrease as $p$ increases.
	This confirms the trade-off observed in Theorem \ref{thm:utility}.
	Moreover, the $p$ that maximizes the accuracy is increasing with respect to $\varepsilon$.
	This is compatible with the theory since the optimal choice $p \asymp \log (n_{\mathrm{P}} \varepsilon^2+n_{\mathrm{Q}}^{\frac{2 \alpha+2 d}{2 \gamma \alpha+d}})$ is monotonically increasing with respect to $\varepsilon$.
	As in Figure \ref{fig:p-acc-prune}, the accuracy increases with $p_0$ at first, whereas the performance remains unchanged or decreases slightly for $p_0$ larger than the optimum.
	Note that the axes of \ref{fig:p-acc} and \ref{fig:p-acc-prune} are aligned.
	LPCT-prune behaves analogously to LPCT in terms of accuracy and optimal depth choice.
	
	\begin{figure}[!t]
		\centering
		\subfigure[$p$ of LPCT]{
			\begin{minipage}{0.3\linewidth}
				\centering
				\includegraphics[width=\textwidth]{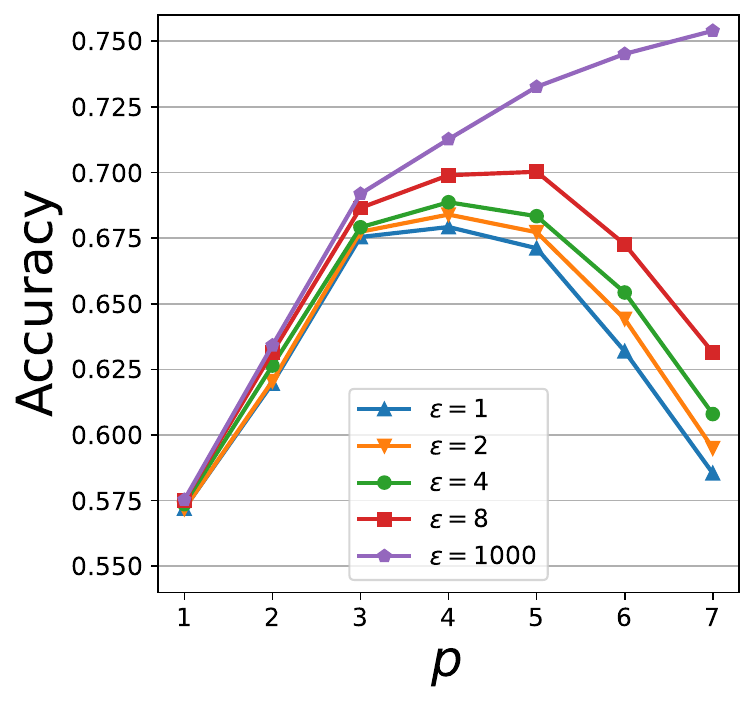}
			\end{minipage}
			\label{fig:p-acc}
		}
		\subfigure[$p_0$ of LPCT-prune]{
			\begin{minipage}{0.3\linewidth}
				\centering
				\includegraphics[width=\textwidth]{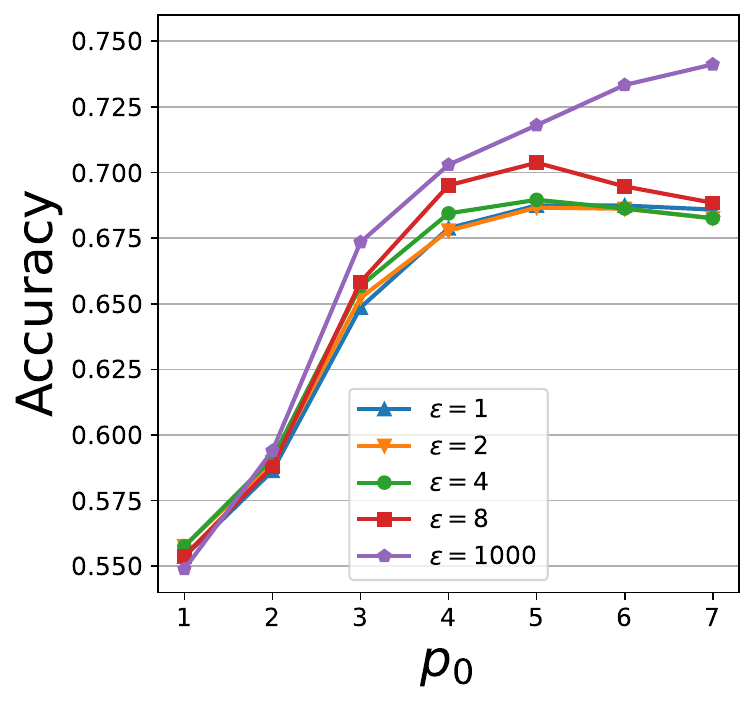}
			\end{minipage}
			\label{fig:p-acc-prune}
		}
		\subfigure[$\lambda$ of LPCT]{
			\begin{minipage}{0.3\linewidth}
				\centering
				\includegraphics[width=\textwidth]{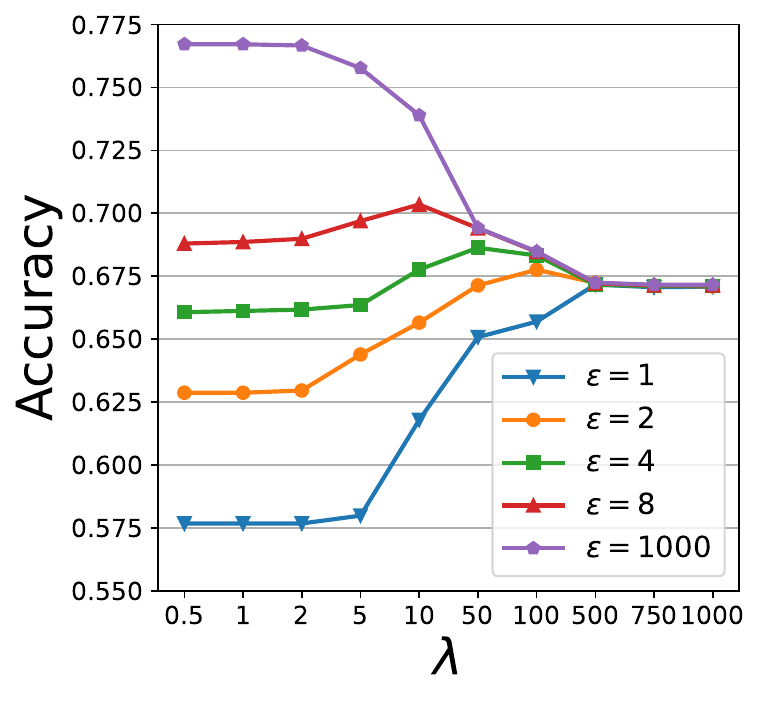}
			\end{minipage}
			\label{fig:lamda-acc}
		}
		\caption{Parameter analysis. }
		\label{fig:param-analysis-p}
		\vskip -0.1in
	\end{figure}

	\paragraph{Parameter analysis of $\mathbf{\lambda}$}
	We conduct experiments to investigate the choice of $\lambda$ in terms of accuracy.
	Under $(n_{\mathrm{P}}, n_{\mathrm{Q}}) = (10000, 1000)$ and $\gamma = 5$, we plot in Figure \ref{fig:lamda-acc} the best performance of LPCT when fixing $\lambda$ to certain values.
	In Figure \ref{fig:lamda-acc}, the relation between accuracy and $\lambda$ is inverted U-shaped for $\varepsilon =2, 4, 8$, while is monotonically increasing and decreasing for $\varepsilon = 1$ and $1000$, respectively.
	The results indicate that a properly chosen $\lambda$ is necessary as stated in Theorem \ref{thm:utility}.
	Moreover, the $\lambda$ that maximizes the accuracy is decreasing with respect to $\varepsilon$, which justify the fact that we should rely more on public data with higher privacy constraint.

	\subsubsection{Benefits of public data}
	\label{sec:simulationgainofpublicdata}

	\begin{figure}[!b]
		\centering
		\subfigure[Illustration of range estimation]{
			\begin{minipage}{0.32\linewidth}
				\vskip -0.07in
				\centering
				\includegraphics[width=\textwidth]{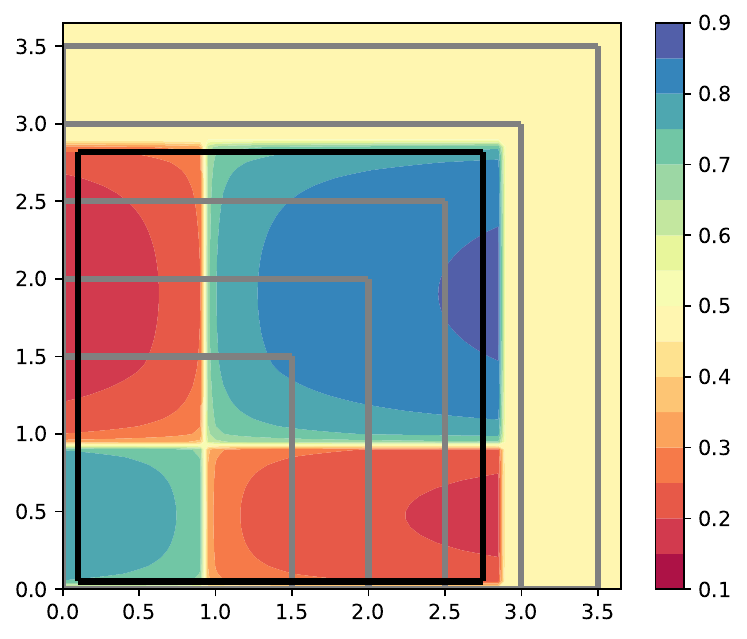}
				\vskip 0.1in
			\end{minipage}
			\label{fig:range_illustration}
		}
		\subfigure[$\varepsilon = 0.5$]{
			\begin{minipage}{0.3\linewidth}
				\centering
				\includegraphics[width=\textwidth]{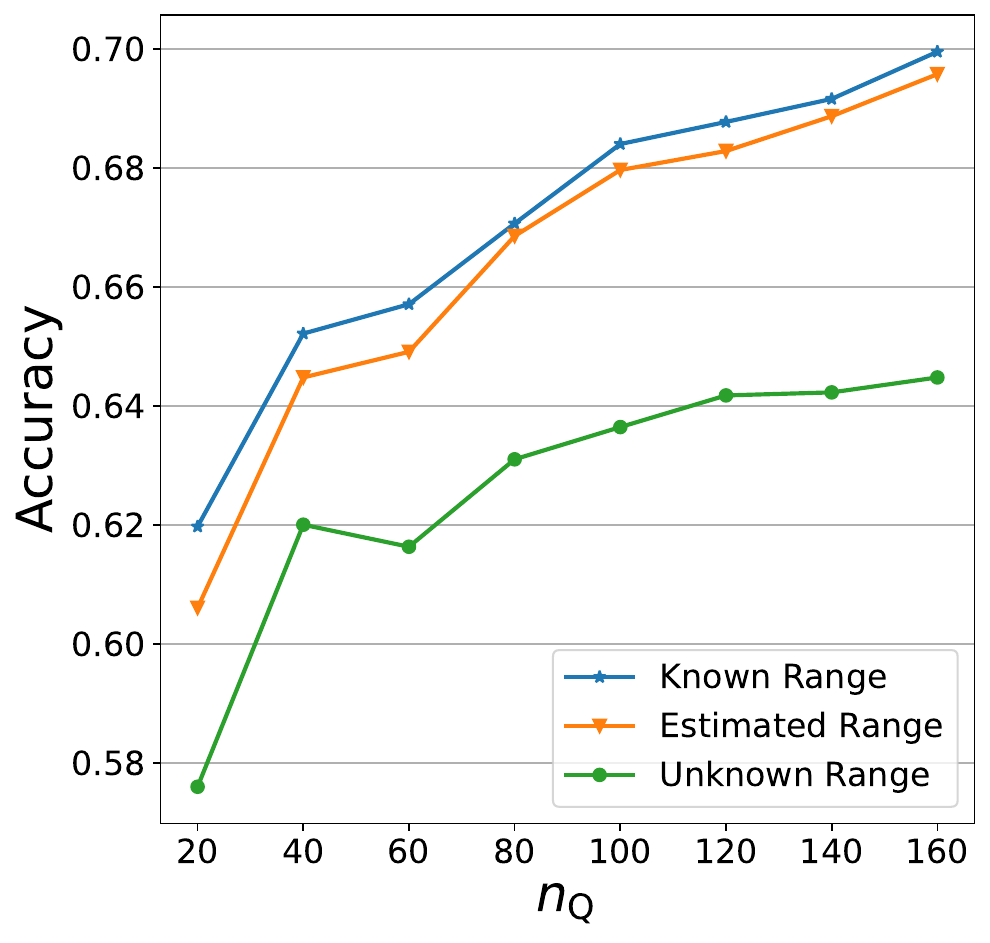}
			\end{minipage}
			\label{fig:range1}
		}
		\subfigure[$\varepsilon = 8$]{
			\begin{minipage}{0.3\linewidth}
				\centering
				\includegraphics[width=\textwidth]{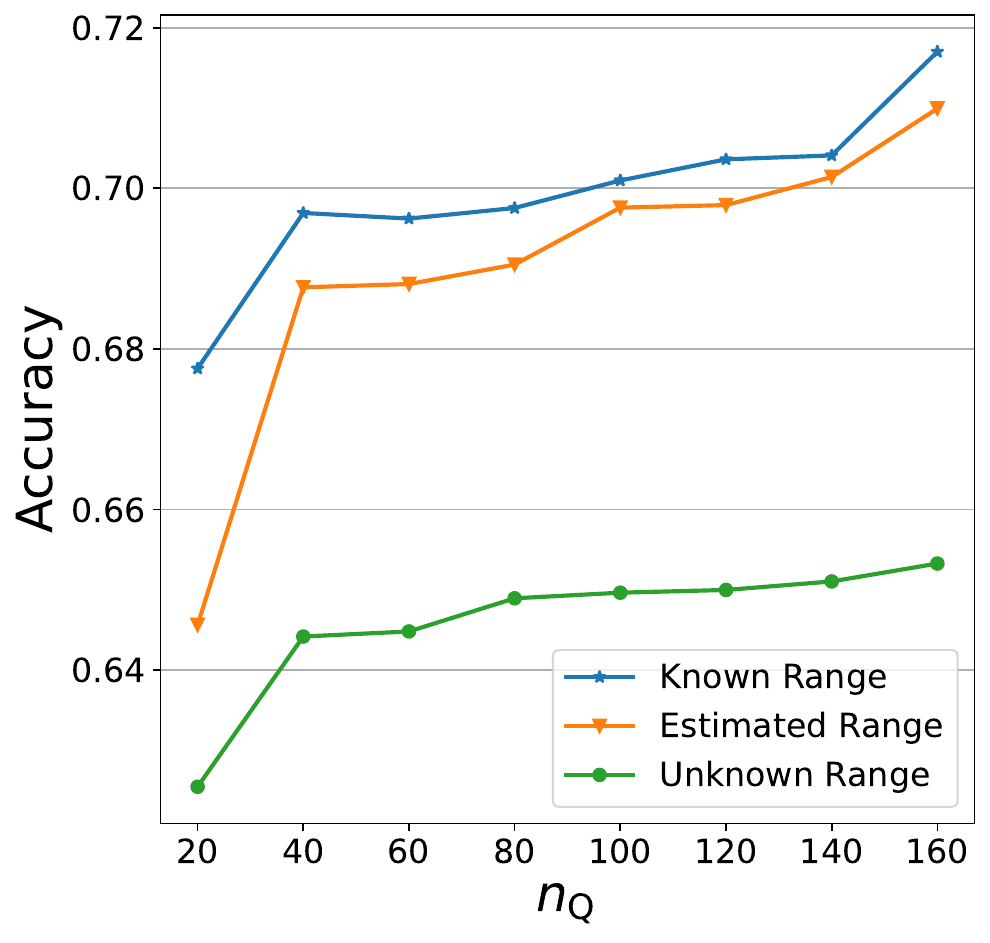}
			\end{minipage}
			\label{fig:range2}
		}
		\caption{Analysis of range parameter. In \ref{fig:range_illustration}, we have $a = 2.9$. The black square is the range estimated by public data in case (II). The grey squares represent the grid of $\overline{a}$ in case (III).  }
		\label{fig:rangeparameter}
	\end{figure}
	
	\paragraph{Range parameter}
	We demonstrate the effectiveness of range estimation using public data when the domain is unknown.
	Consider the case where $\mathcal{X} = [0,a]^d$ and $a$ is a random variable uniformly distributed in $[1,5]$.
	We compare the following scenarios:
	(I) $a$ is known. (II) $a$ is estimated using public data as shown in Theorem \ref{thm:rangeparameter}. (III) $a$ is unknown.
	Case (I) is equivalent to $\mathcal{X} = [0,1]^d$.
	In case (II), we min-max scale the samples from the estimated range to $[0,1]^d$.
	In case (III), we scale the samples from $[0,\overline{a}]^d$ to $[0,1]^d$ and select the best $\overline{a}$ among $\overline{a}\in \{ 1, 1.5,\cdots, 5\}$.
	In Figure \ref{fig:range_illustration}, we showcase the scaling schemes in (II) and (III).
	The estimated range from public data is much more concise than the best $\overline{a}$.
	Given no prior knowledge, the $\overline{a}$ chosen in practice can perform significantly worse.
	In Figure \ref{fig:range1} and \ref{fig:range2}, we observe that the performance of the estimated range matches that of (I) as $n_{\mathrm{Q}}$ increases, whereas the unknown range significantly degrades performance.
	The performance gap between the known and estimated range is more substantial when $n_{\mathrm{Q}}$ is small.
	This observation is compatible with the conclusion drawn in Theorem \ref{thm:rangeparameter} emphasizing the need for sufficient public data to resolve the unknown range issue.

	\paragraph{Partition}
	We demonstrate the effectiveness of partitioning with information of public data by comparing LPCT with LPCT-R.
	In Figure \ref{fig:randompartion1}, the showcase of $p = 3$ is presented.
	The random partition is possibly created as the top figure, where the upper left and the lower right cells are hard to distinguish as 0 or 1.
	Using sufficient public data, the max-edge partition is bound to create an informative partition like the lower figure.
	Moreover, Figure \ref{fig:randompartion2} shows that LPCT outperforms LPCT-R.

	\begin{figure}[!t]
		\centering
		\subfigure[Illustration of random partition and max-edge partition]{
			\begin{minipage}{0.2\linewidth}
				\centering
				\includegraphics[width=0.95\textwidth]{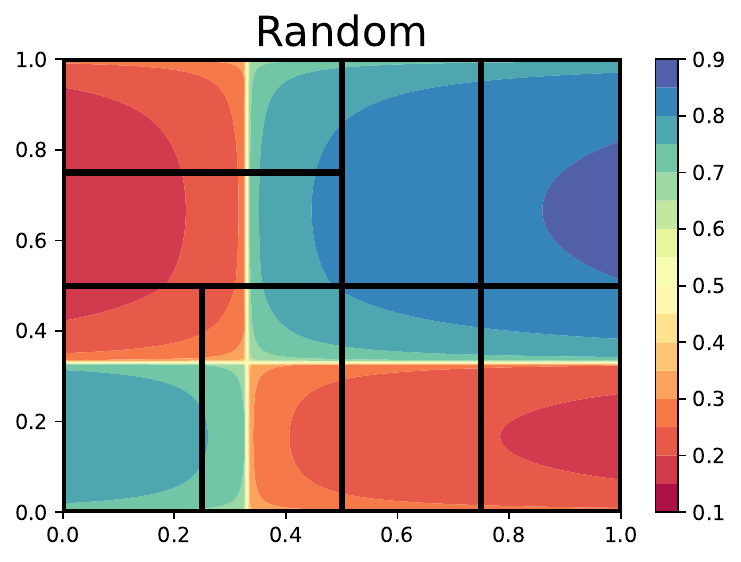}
				\includegraphics[width=0.95\textwidth]{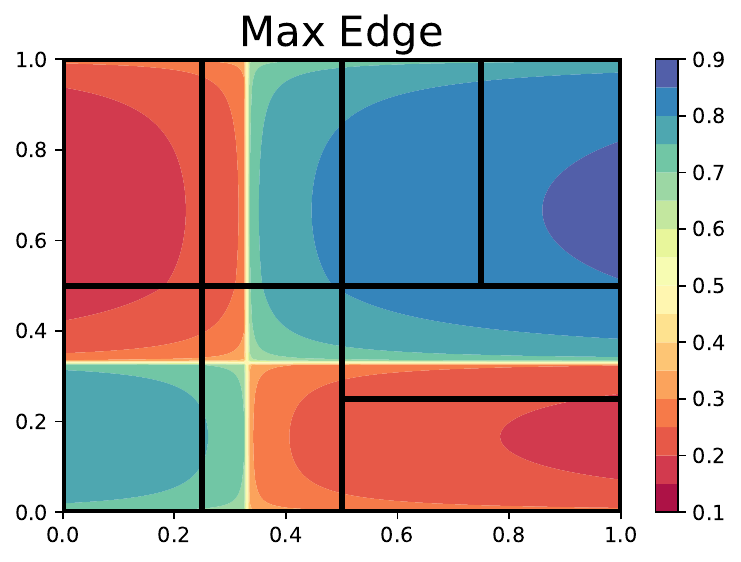}
			\end{minipage}
			\label{fig:randompartion1}
		}
		\hskip 0.2in
		\subfigure[Comparison of LPCT and LPCT-R]{
			\begin{minipage}{0.45\linewidth}
				\centering
				\includegraphics[width=\textwidth]{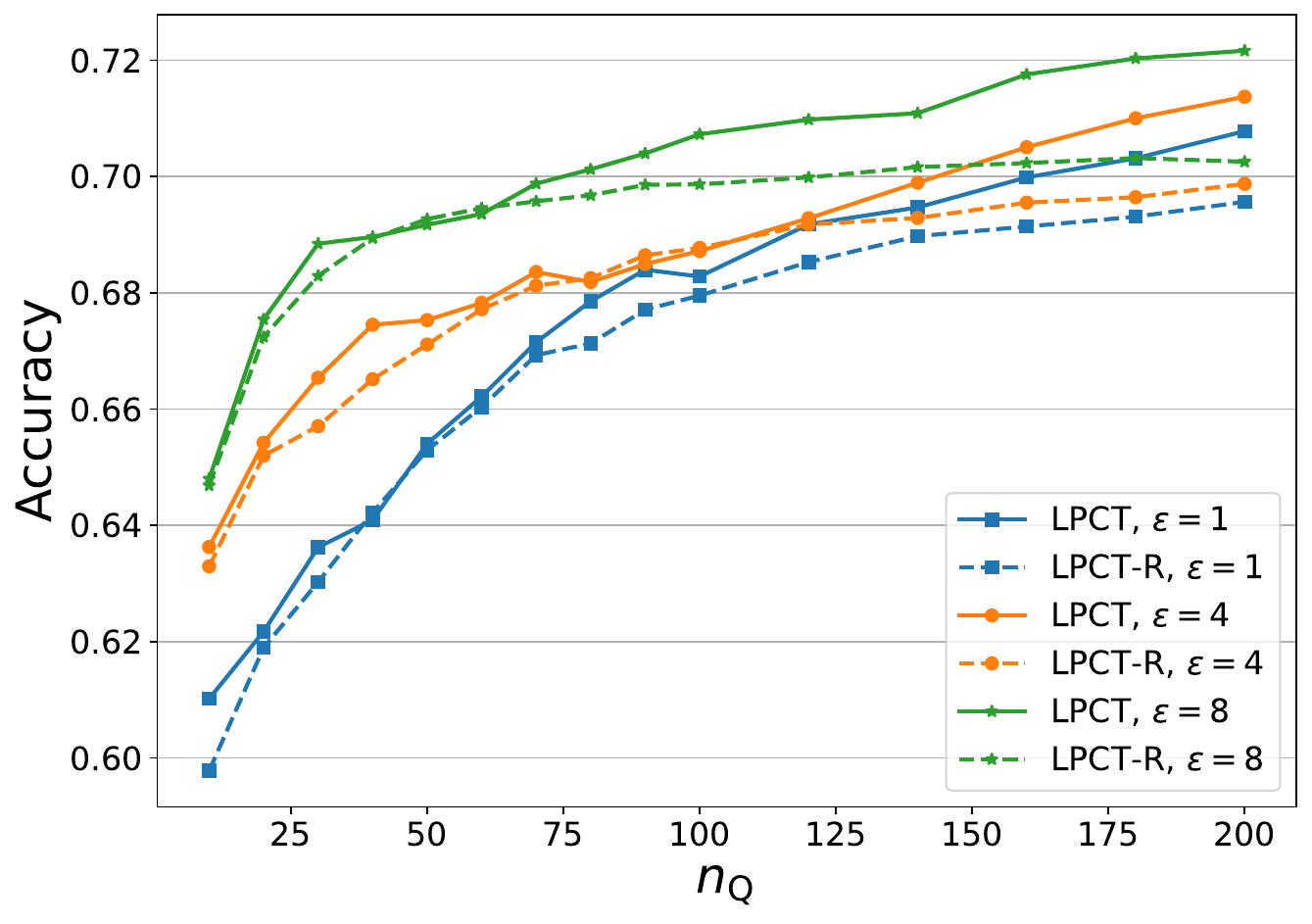}
			\end{minipage}
			\label{fig:randompartion2}
		}
		\centering
		\caption{Illustration of partition with public data. }
		\label{fig:randompartition}
	\end{figure}

	\section{Experiments on Real Data}\label{sec:realdata}

	\subsection{Data Description}\label{sec:realdatadescription}
	We conduct experiments on real-world data sets to show the superiority of LPCT.
	We utilize 11 real data sets that are available online.
	Each of these data sets contains sensitive information and necessitates privacy protection.
	Furthermore, these data sets cover a wide range of sample sizes and feature dimensions commonly encountered in real-world scenarios.
	When creating $D^{\mathrm{P}}$ and $D^{\mathrm{Q}}$ from the original data set $D$, some data sets have explicit criteria for splitting, reflecting real-world considerations.
	In these cases, we have $\mathrm{P} \neq \mathrm{Q}$.
	For others, the split is performed randomly and we have $\mathrm{P} = \mathrm{Q}$.
	A summary of key information for these data sets after pre-processing can be found in Table \ref{tab:informationrealdatasets}.
	Following the analysis in Section \ref{sec:rangeparameter}, we min-max scale each data to $[0,1]^d$ according to public data.
	Detailed data information and pre-processing procedures are provided in Appendix \ref{sec:realdatasetpreprocess}.

	\begin{table}[!t]
		\centering
		\caption{Information of real data sets.}
		\label{tab:informationrealdatasets}
		\resizebox{0.5\linewidth}{!}{
			\renewcommand{\arraystretch}{1.1}
			\setlength{\tabcolsep}{5pt}
			\begin{tabular}{|l|c|r|r|r|l|}
				\toprule
				Dataset            & $\mathrm{P} = \mathrm{Q}$ & \multicolumn{1}{c|}{$n_{\mathrm{P}}$} & \multicolumn{1}{c|}{$n_{\mathrm{Q}}$} & \multicolumn{1}{c|}{d} & \multicolumn{1}{c|}{Area} \\ \midrule
				\texttt{Anonymity} & $\checkmark$ & 453   & 50    & 14     & Social    \\
				\texttt{Diabetes}  & $\checkmark$ & 2,054 & 200   & 8      & Medical   \\
				\texttt{Email}     & $\checkmark$ & 3,977 & 200   & 3,000  & Text      \\
				\texttt{Employ}    & $\checkmark$ & 2,225 & 200   & 8      & Business  \\
				\texttt{Rice}      & $\checkmark$ & 240   & 3,510 & 7      & Image     \\
				\texttt{Census}    & $\times$     & 41,292                & 3,144 & 46     & Social    \\
				\texttt{Election}  & $\times$     & 1,752 & 222   & 9      & Social    \\
				\texttt{Employee}  & $\times$     & 3,319 & 504   & 9      & Business  \\
				\texttt{Jobs}      & $\times$     & 46,218                & 14,318                & 11     & Social    \\
				\texttt{Landcover} & $\times$     & 7,097 & 131   & 28     & Image     \\
				\texttt{Taxi}      & $\times$     & 3,297,639             & 117,367               & 93     & Social    \\ \bottomrule
			\end{tabular}
		}
	\end{table}

	\subsection{Competitors}
	
	We compare LPCT with several state-of-the-art competitors.
	For each model, we report the best result over its parameter grids based on an average of 20 replications.
	The tested methods are
	\begin{itemize}
		\item \textbf{LPCT} is described in Section \ref{sec:simulation}.
		In real data experiments, we enlarge the grids to $p\in\{1,2,3,4,5,6,7,8,10,12,14,16\}$ and $\lambda\in\{0.1$ $, 0.5, 1, $ $2, 5, 10, $ $ 50, 100, $ $ 200,$ $ 300, 400, $ $ 500, 750, 1000, 1250, $ $ 1500, 2000\}$ due to the large sample size.
		Since the max-edge rule is easily affected by useless features or categorical features \citep{ma2023decision}, we boost the performance of LPCT by incorporating the criterion reduction scheme from the original CART \citep{breiman1984classification} to the tree construction.
		Since Proposition \ref{thm:dppartition} holds for any partition, the method is also $\varepsilon$-LDP.
		We choose the Gini index as the reduction criterion.
		\item \textbf{LPCT-prune} is the pruned locally private classification tree in Algorithm \ref{alg:prunedeta}.
		We also adopt the criterion reduction scheme for tree construction.
		\item \textbf{CT} is the original classification tree \citep{breiman1984classification} with the same grid as LPCT.
		We define two approaches, \textbf{CT-W} and \textbf{CT-Q}.
		For CT-W, we fit CT on the whole data, namely $\mathrm{P}$ and $\mathrm{Q}$ data together, with no privacy protection.
		Its accuracy will serve as an upper bound.
		For CT-Q, we fit CT on $\mathrm{Q}$ data only.
		Its performance is taken into comparison.
		We use implementation by \texttt{sklearn} \citep{scikit-learn} with default parameters while varying \texttt{max\_depth} in $\{1,2,\cdots, 16\}$.
		\item \textbf{PHIST} is the locally differentially private histogram estimator proposed by \cite{berrett2019classification}.
		We fit PHIST on private data solely as it does not leverage public data.
		We choose the number of splits $h^{-1}$ along each axis in $\{1, 2, 3, 4, 5, 6\}$.
		\item \textbf{LPDT} is the locally private decision tree originally proposed for regression \citep{ma2023decision}.
		The approach creates a partition on $\mathrm{Q}$ data and utilizes $\mathrm{P}$ data for estimation.
		We adapt the strategy to classification.
		We select $p$ in the same range as LPCT.
		\item \textbf{PGLM} is the locally private generalized linear model proposed by \cite{wang2023generalized} with logistic link function.
		Note that their privacy result is under $(\varepsilon, \delta)$-LDP.
		We let $\delta = {1}/{n_{\mathrm{P}}^2}$ for a fair comparison.
		\item \textbf{PSGD} is the locally private stochastic gradient descent developed in \cite{lowy2023optimal}.
		The presented algorithm is in a sequentially interactive manner, meaning that the privatized information from $X_i$ depends on the information from $X_{i-1}$.
		We adopt the warm start paradigm where we first find a minimizer on the public data and then initialize the training process with the minimizer.
		Following \cite{lowy2023optimal}, we choose the parameters of the privacy mechanism, i.e. PrivUnitG, by the procedure they provided.
		We let $C = 1$ and learning rate $\eta \in \{10^{-5},10^{-4},\cdots, 1 \}$.
		We consider two models: \textbf{PSGD-L} is the linear classifier using the logistic loss; \textbf{PSGD-N} is the non-linear neural network classifier with one hidden layer and ReLU activation function.
		The number of neurons is selected in $d\in\{d/2, d, 2d\}$.
	\end{itemize}

	\subsection{Experiment Results}

	We first illustrate the substantial utility gain brought by public data.
	We choose \texttt{Diabetes} and \texttt{Employee}, representing the case of $\mathrm{P} = \mathrm{Q}$ and $\mathrm{P} \neq \mathrm{Q}$, respectively.
	On each data, we randomly select a portion of private data and 80 public samples to train three models: LPCT, LPDT, and PHIST.
	PHIST does not leverage any information of public data, while LPDT benefits from public data only through the partition.
	LPCT utilizes public data in both partition and estimation.
	In Figure \ref{fig:utilitybypublic}, the accuracy increases with respect to the extent of public data utilization for both identically and non-identically distributed public data.
	Moreover, the accuracy gain is significant.
	For instance, in Figure \ref{fig:utilitybypublic2}, 80 public samples and $0.2 \times 3319 \approx 664$ private samples achieve the same accuracy as using $0.6\times 3319\approx 1991$ private samples alone.
	The observation yields potential practical clues.
	For instance, when data-collecting resources are limited, organizations can put up rewards to get a small amount of public data rather than collecting a large amount of private data.

	\begin{figure}[!b]
		\centering
		\subfigure[$\mathrm{P} = \mathrm{Q}$]{
			\begin{minipage}{0.31\linewidth}
				\centering
				\includegraphics[width=\textwidth]{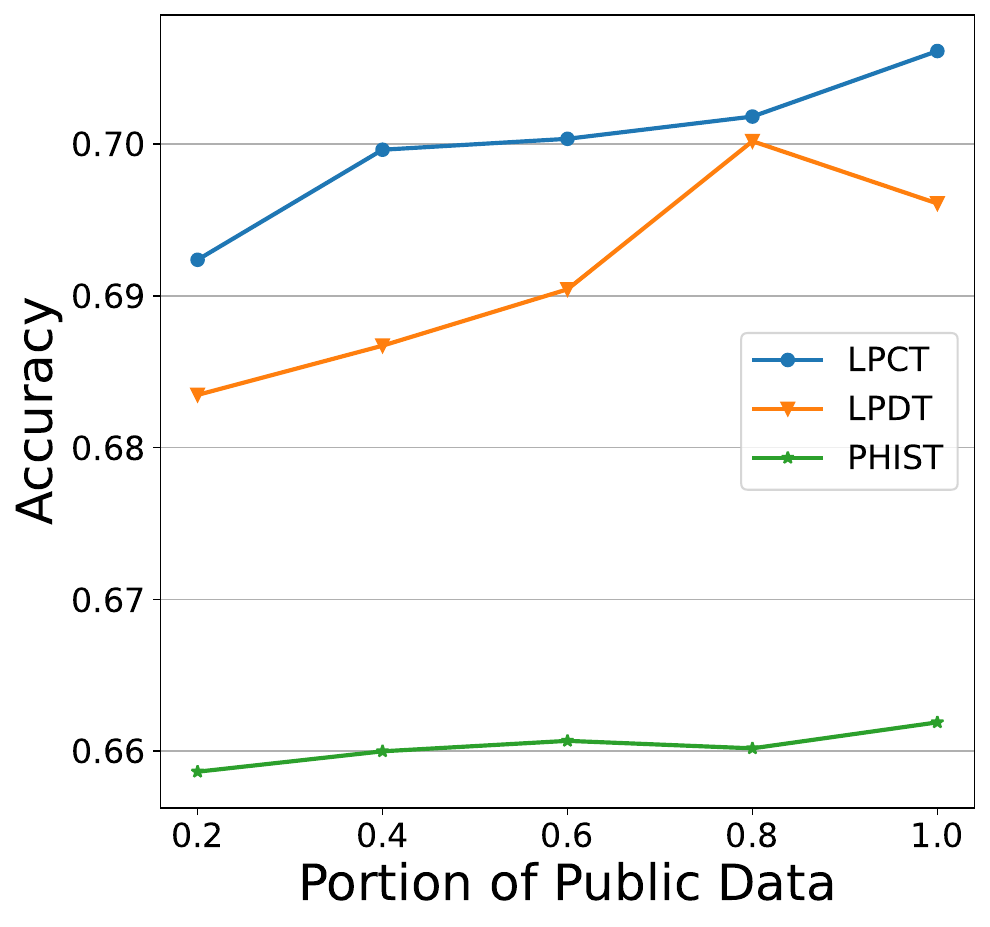}
			\end{minipage}
			\label{fig:utilitybypublic1}
		}
		\hskip 0.4in
		\subfigure[$\mathrm{P} \neq \mathrm{Q}$ ]{
			\begin{minipage}{0.31\linewidth}
				\centering
				\includegraphics[width=\textwidth]{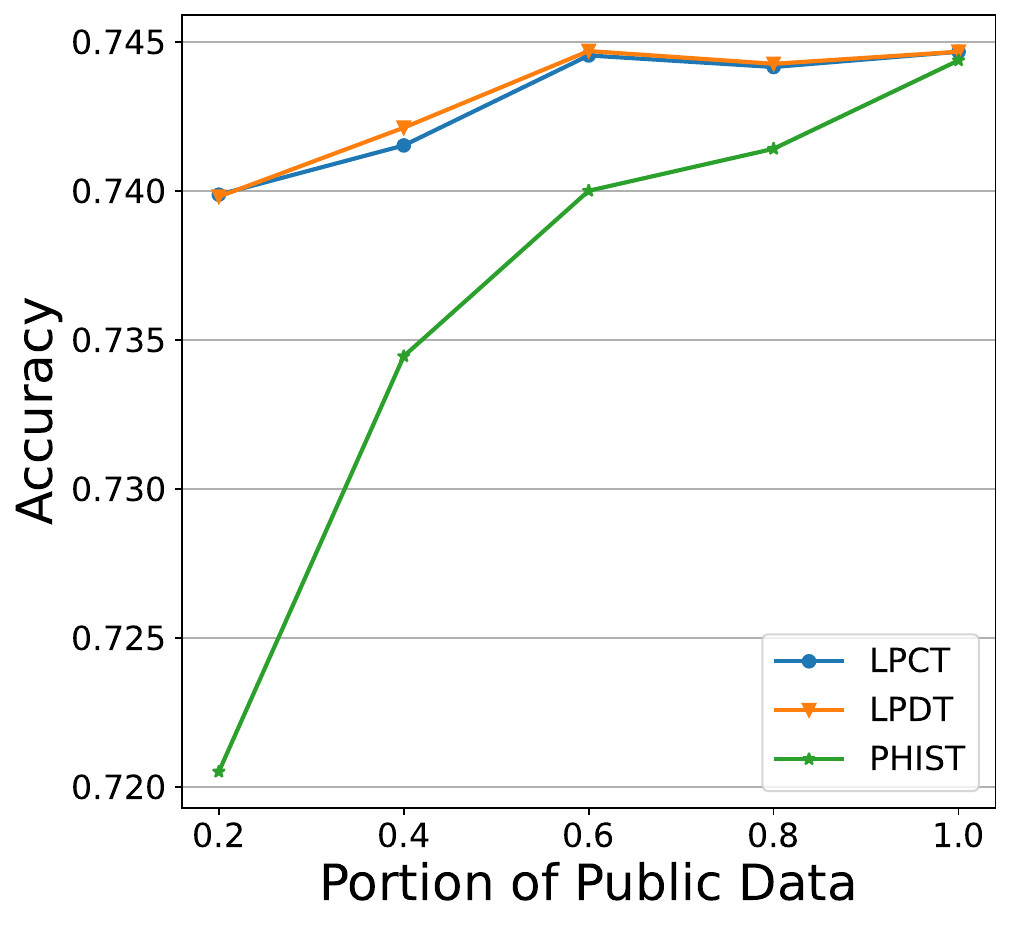}
			\end{minipage}
			\label{fig:utilitybypublic2}
		}
		\caption{ Illustration of utility gain brought by public data.} 
		\label{fig:utilitybypublic}
	\end{figure}

	
	Next, we report the averaged accuracy for $\varepsilon \in \{0.5, 2, 8\}$ in Table \ref{tab:realdata} on each real data set.
	Here, CT-W is not private and serves as a reference.
	To ensure significance, we adopt the Wilcoxon signed-rank test  \citep{wilcoxon1992individual} with a significance level of 0.05 to check if the result is significantly better.
	For each data set, trials that exceed a memory of 100GB or a running time of one hour are terminated.

	\begin{table}[p]
		\centering
		\caption{\footnotesize{Accuracy ($10^{-1}$) over real data sets.
				The best results are marked with *. The results that are not significantly worse than the best result are \textbf{bolded}.
				Due to memory limitation, PHIST is corrupted on three data sets.
				PSGD-L and PSGD-N exceed the running time of 3,600 seconds on one data set and are terminated.
				These results are marked with -.}}
		\label{tab:realdata}
		\subtable[$\varepsilon =0.5$]{
			\label{tab:realdataepsilon05}
			\centering
			\resizebox{0.93\linewidth}{!}{
				\renewcommand{\arraystretch}{1.06}
				\setlength{\tabcolsep}{5pt}
				
				\begin{tabular}{ll|llll|llllll}
					\toprule
					\multicolumn{1}{c|}{\multirow{2}{*}{Datasets}} & \multicolumn{1}{c|}{\multirow{2}{*}{CT-W}} & \multicolumn{2}{c}{LPCT}                               & \multicolumn{2}{c|}{LPCT-prune}                         & \multicolumn{1}{c}{\multirow{2}{*}{CT-Q}} & \multicolumn{1}{c}{\multirow{2}{*}{PHIST}} & \multicolumn{1}{c}{\multirow{2}{*}{LPDT}} & \multicolumn{1}{c}{\multirow{2}{*}{PGLM}} & \multicolumn{1}{c}{\multirow{2}{*}{PSGD-L}} & \multicolumn{1}{c}{\multirow{2}{*}{PSGD-N}} \\
					\multicolumn{1}{c|}{}                          & \multicolumn{1}{c|}{}                      & \multicolumn{1}{c}{Maxedge} & \multicolumn{1}{c}{CART} & \multicolumn{1}{c}{Maxedge} & \multicolumn{1}{c|}{CART} & \multicolumn{1}{c}{}                      & \multicolumn{1}{c}{}                       & \multicolumn{1}{c}{}                      & \multicolumn{1}{c}{}                      & \multicolumn{1}{c}{}                        & \multicolumn{1}{c}{}                        \\ \midrule
					\multicolumn{1}{l|}{\texttt{Anonymity}}        & 8.346                                      & \textbf{8.346}*             & \textbf{8.346}*          & \textbf{8.309}           & \textbf{8.309}           & \textbf{8.346}*                           & \textbf{8.346}*                            & {7.546}                                   & 5.899                                     & 7.790                                       & 8.123                                       \\
					\multicolumn{1}{l|}{\texttt{Diabetes}}         & 9.964                                      & 7.067                       & \textbf{7.590}*          & 6.661                       & \textbf{7.529}                     & 7.301                                     & 6.676                                      & 6.305                                     & 5.499                                     & 6.582                                       & 6.448                                       \\
					\multicolumn{1}{l|}{\texttt{Email}}            & 9.261                                      & 7.332                       & \textbf{7.487}*           & 7.205                       & 7.185                     & 7.283                                     & -                                          & 7.112                                     & 5.136                                     & 6.435                                       & 6.292                                       \\
					\multicolumn{1}{l|}{\texttt{Employ}}           & 9.048                                      & 7.403                       & \textbf{8.646}*          & 6.684                       & 6.362                     & \textbf{8.577}                            & 5.704                                      & 5.644                                     & 4.874                                     & 5.742                                       & 5.595                                       \\
					\multicolumn{1}{l|}{\texttt{Rice}}             & 8.992                                      & 9.039                       & \textbf{9.183}*          & 9.053                       & \textbf{9.173}           & \textbf{9.117}                            & 5.467                                      & 6.583                                     & 4.886                                     & 8.517                                       & 8.142                                       \\
					\multicolumn{1}{l|}{\texttt{Census}}           & 8.685                                      & 8.539                       & \textbf{8.743}*          & 8.471                       & 8.528                    & 8.285                                     & -                                          & 7.195                                     & 5.104                                     & 8.310                                       & 8.405                                       \\
					\multicolumn{1}{l|}{\texttt{Election}}         & 6.408                                      & 5.610                       & \textbf{5.846}           & 5.259                       & {5.773}            & \textbf{5.958}*                           & 5.060                                      & 5.089                                     & 4.948                                     & 5.207                                       & 5.074                                       \\
					\multicolumn{1}{l|}{\texttt{Employee}}         & 7.648                                      & \textbf{7.397}*             & 7.131                    & \textbf{7.381}                       & 7.040                     & 7.151                                     & {6.660}                                    & {6.098}                                   & 5.280                                     & 6.493                                       & 6.387                                       \\
					\multicolumn{1}{l|}{\texttt{Jobs}}             & 7.863                                      & 5.628                       & \textbf{7.814}*          & 5.561                       & \textbf{7.793}                     & 7.599                                     & 5.327                                      & 5.441                                     & 4.942                                     & 5.530                                       & 5.732                                       \\
					\multicolumn{1}{l|}{\texttt{Landcover}}        & 9.404                                      & 8.615                       & \textbf{8.960}*           & 8.299                       & 8.719                     & 6.670                                     & 8.362                                      & 8.366                                     & 4.955                                     & 8.368                                       & 8.349                                       \\
					\multicolumn{1}{l|}{\texttt{Taxi}}             & 9.601                                      &               5.647              &       \textbf{9.589}*                   & 5.637                       & \textbf{9.589}*                     & \textbf{9.490}                                     & -                                          & 5.438                                     & 4.912                                     & -                                           &      -                                       \\ \midrule
					\multicolumn{2}{c|}{\textbf{Rank Sum}}      & 36          & 18*       & 54          & 36        & 38        & 83         & 84        & 106        & 70          & 80          \\
					\multicolumn{2}{c|}{\textbf{No. 1}}         & 2           & 10*        & 2           & 5         & 5         & 1          & 0         & 0         & 0           & 0           \\ \bottomrule
				\end{tabular}
			}
		}
		\subtable[$\varepsilon = 2$]{
			\label{tab:realdataepsilon2}
			\centering
			\resizebox{0.93\linewidth}{!}{
				\renewcommand{\arraystretch}{1.06}
				\setlength{\tabcolsep}{5pt}
				\begin{tabular}{ll|llll|llllll}
					\toprule
					\multicolumn{1}{c|}{\multirow{2}{*}{Datasets}} & \multicolumn{1}{c|}{\multirow{2}{*}{CT-W}} & \multicolumn{2}{c}{LPCT}                               & \multicolumn{2}{c|}{LPCT-prune}                         & \multicolumn{1}{c}{\multirow{2}{*}{CT-Q}} & \multicolumn{1}{c}{\multirow{2}{*}{PHIST}} & \multicolumn{1}{c}{\multirow{2}{*}{LPDT}} & \multicolumn{1}{c}{\multirow{2}{*}{PGLM}} & \multicolumn{1}{c}{\multirow{2}{*}{PSGD-L}} & \multicolumn{1}{c}{\multirow{2}{*}{PSGD-N}} \\
					\multicolumn{1}{c|}{}                          & \multicolumn{1}{c|}{}                      & \multicolumn{1}{c}{Maxedge} & \multicolumn{1}{c}{CART} & \multicolumn{1}{c}{Maxedge} & \multicolumn{1}{c|}{CART} & \multicolumn{1}{c}{}                      & \multicolumn{1}{c}{}                       & \multicolumn{1}{c}{}                      & \multicolumn{1}{c}{}                      & \multicolumn{1}{c}{}                        & \multicolumn{1}{c}{}                        \\ \midrule
					\multicolumn{1}{l|}{\texttt{Anonymity}}        & 8.346                                      & \textbf{8.346}*             & \textbf{8.346}*          & {8.086}             & {8.168}           & \textbf{8.346}*                           & \textbf{8.346}*                            & \textbf{8.346}*                           & 5.899                                     & 7.965                                       & 8.039                                       \\
					\multicolumn{1}{l|}{\texttt{Diabetes}}         & 9.964                                      & 7.274                       & \textbf{7.590}*          & 6.661                       & 7.412                     & 7.301                                     & 6.676                                      & 6.676                                     & 5.499                                     & 6.653                                       & 6.674                                       \\
					\multicolumn{1}{l|}{\texttt{Email}}            & 9.261                                      & 7.350                       & \textbf{7.485}*          & 7.224                       & \textbf{7.459}                     & 7.283                                     & -                                          & 7.181                                     & 5.136                                     & 7.136                                       & 7.203                                       \\
					\multicolumn{1}{l|}{\texttt{Employ}}           & 9.048                                      & 7.702                       & \textbf{9.121}*          & 7.005                       & 6.578                     & {8.577}                                   & 6.174                                      & 7.181                                     & 4.874                                     & 5.786                                       & 5.792                                       \\
					\multicolumn{1}{l|}{\texttt{Rice}}             & 8.992                                      & 9.027                       & \textbf{9.183}*          & 9.053                       & \textbf{9.183}*           & \textbf{9.117}                            & 5.467                                      & 8.458                                     & 4.886                                     & \textbf{9.050}                              & 8.750                                       \\
					\multicolumn{1}{l|}{\texttt{Census}}           & 8.685                                      & \textbf{8.537}              & \textbf{8.743}*          & 8.460                       & \textbf{8.633}                     & 8.285                                     & -                                          & 8.017                                     & 5.104                                     & 8.432                                       & 8.471                                       \\
					\multicolumn{1}{l|}{\texttt{Election}}         & 6.408                                      & 5.648                       & \textbf{5.850}           & 5.327                       & {5.595}            & \textbf{5.958}*                           & 5.056                                      & 5.114                                     & 4.948                                     & 5.323                                       & 5.111                                       \\
					\multicolumn{1}{l|}{\texttt{Employee}}         & 7.648                                      & \textbf{7.437}*            & 7.175                    & \textbf{7.362}                       & 7.175                     & 7.151                                     & {6.472}                                    & {6.951}                                   & 5.280                                     & 6.400                                       & 6.383                                       \\
					\multicolumn{1}{l|}{\texttt{Jobs}}             & 7.863                                      & 5.628                       & \textbf{7.821}*          & 5.597                       & \textbf{7.778}            & 7.599                                     & 5.333                                      & 5.591                                     & 4.942                                     & 6.788                                       & 5.934                                       \\
					\multicolumn{1}{l|}{\texttt{Landcover}}        & 9.404                                      & 8.715                       & \textbf{8.960}*          & 8.623                       & 8.734                     & 6.670                                     & 8.367                                      & 8.493                                     & 4.955                                     & 8.553                                       & 8.369                                       \\
					\multicolumn{1}{l|}{\texttt{Taxi}}             & 9.601                                      &               5.659              &         \textbf{9.590}*                 & 5.643                       & \textbf{9.580}                     & \textbf{9.490}                                     & -                                          & 5.536                                     & 4.912                                     & -                                           &       -                                      \\ \midrule
					\multicolumn{2}{c|}{\textbf{Rank Sum}}                                                      & 36                          & 18*                       & 54                          & 36                        & 38                                       & 83                                         & 84                                        & 106                                        & 70                                          & 80                                          \\
					\multicolumn{2}{c|}{\textbf{No. 1}}                                                         & 3                           & 10*                        & 1                           & 5                         & 4                                         & 1                                          & 1                                         & 0                                         & 1                                           & 0                                           \\ \bottomrule
				\end{tabular}
			}
		}
		\subtable[$\varepsilon = 8$]{
			\label{tab:realdataepsilon8}
			\centering
			\resizebox{0.93\linewidth}{!}{
				\renewcommand{\arraystretch}{1.06}
				\setlength{\tabcolsep}{5pt}
				\begin{tabular}{ll|llll|llllll}
					\toprule
					\multicolumn{1}{c|}{\multirow{2}{*}{Datasets}} & \multicolumn{1}{c|}{\multirow{2}{*}{CT-W}} & \multicolumn{2}{c}{LPCT}               & \multicolumn{2}{c|}{LPCT-prune}         & \multicolumn{1}{c}{\multirow{2}{*}{CT-Q}} & \multicolumn{1}{c}{\multirow{2}{*}{PHIST}} & \multicolumn{1}{c}{\multirow{2}{*}{LPDT}} & \multicolumn{1}{c}{\multirow{2}{*}{PGLM}} & \multicolumn{1}{c}{\multirow{2}{*}{PSGD-L}} & \multicolumn{1}{c}{\multirow{2}{*}{PSGD-N}} \\
					\multicolumn{1}{c|}{}          & \multicolumn{1}{c|}{}      & \multicolumn{1}{c}{Maxedge} & \multicolumn{1}{c}{CART} & \multicolumn{1}{c}{Maxedge} & \multicolumn{1}{c|}{CART} & \multicolumn{1}{c}{}      & \multicolumn{1}{c}{}       & \multicolumn{1}{c}{}      & \multicolumn{1}{c}{}      & \multicolumn{1}{c}{}        & \multicolumn{1}{c}{}        \\ \midrule
					\multicolumn{1}{l|}{\texttt{Anonymity}}    & 8.346  & \textbf{8.346}*   & \textbf{8.346}*   & {8.221}            & {8.309}  & \textbf{8.346}* & \textbf{8.346}*   & \textbf{8.346}* & 5.395  & \textbf{8.338}    & 8.154 \\
					\multicolumn{1}{l|}{\texttt{Diabetes}}  & 9.964 & 7.503 & \textbf{7.930}* & 6.648  & 7.433  & 7.301   & 6.676   & 7.496      & 5.588  & 7.244  & 7.101\\
					\multicolumn{1}{l|}{\texttt{Email}}  & 9.261  & 7.355  & 7.520    & 7.251  & 7.491  & 7.283  & - & 7.326  & 5.271 & \textbf{8.446}   & 7.060   \\
					\multicolumn{1}{l|}{\texttt{Employ}}           & 9.048      & 7.619       & \textbf{9.068}*          & 7.393       & 7.492     & 8.577     & 9.043      & 7.572     & 5.237     & 5.795       & 5.798       \\
					\multicolumn{1}{l|}{\texttt{Rice}}             & 8.992      & 9.011       & \textbf{9.183}*          & 9.053       & \textbf{9.137}            & \textbf{9.117}            & 8.858      & 9.017     & 6.283     & 9.017       & 8.875       \\
					\multicolumn{1}{l|}{\texttt{Census}}           & 8.685      & 8.557       & \textbf{8.743}*          & 8.404       & \textbf{8.671}     & 8.285     & -          & 8.527     & 5.389     & 8.530       & 8.154       \\
					\multicolumn{1}{l|}{\texttt{Election}}         & 6.408      & 5.894       & \textbf{6.082}*          & 5.321       & 5.563     & \textbf{5.958}            & 5.015      & 5.892     & 5.462     & 5.765       & 5.705       \\
					\multicolumn{1}{l|}{\texttt{Employee}}         & 7.648      & \textbf{7.441}*             & 7.173    & \textbf{7.384}       & 7.105     & 7.151     & \textbf{7.437}             & \textbf{7.439}            & 4.930     & 6.775       & 6.883       \\
					\multicolumn{1}{l|}{\texttt{Jobs}}  & 7.863      & 5.650       & \textbf{7.861}*          & 5.584       & \textbf{7.794}     & 7.599     & 5.992      & 5.641     & 6.802     & 7.712       & 7.414       \\
					\multicolumn{1}{l|}{\texttt{Landcover}}        & 9.404      & 8.737       & 8.963    & 8.620       & 8.745     & 6.670     & 8.367      & 8.695     & 5.551     & \textbf{9.137}*             & 8.948       \\
					\multicolumn{1}{l|}{\texttt{Taxi}}             & 9.601      &   5.666          &     \textbf{9.610}*     & 5.637       & \textbf{9.579}     & \textbf{9.490}     & -          &      5.620     & 5.087     & -       & -      \\ \midrule
					\multicolumn{2}{c|}{\textbf{Rank Sum}}      & 42          & 18*       & 66          & 36       & 49        & 90         & 68        & 104        & 54          & 78          \\
					\multicolumn{2}{c|}{\textbf{No. 1}}         & 2           & 8*        & 1           & 4         & 4         & 2          & 2         & 0         & 3           & 0           \\ \bottomrule
				\end{tabular}
			}
		}
	\end{table}

	From Table \ref{tab:realdata}, we can see that LPCT-based methods generally outperform their competitors, both in the sense of average performance (rank sum) and best performance (number of No.1).
	Compared to LPCT, LPCT-prune performs reasonably worse due to its data-driven nature.
	In practice, however, all other methods require parameter tuning which limits the privacy budget for the final estimation.
	Thus, their real performance is expected to drop down while that of LPCT-prune remains.
	Methods with the CART rule substantially outperform those with the max-edge rule, as the CART rule is robust to useless features or categorical features and can fully employ data information.
	We note that on some data sets, the results of LPCT remain the same for all $\varepsilon$.
	This is the case where public data takes dominance and the estimator is barely affected by private data.
	Moreover, we observe that linear classifiers such as PSDG-L and PGLM perform poorly on some data sets.
	This is attributed to the non-linear underlying nature of these data sets.

	We compare the efficiency of the methods by reporting the total running time in Table \ref{tab:time}.
	PSGD-L and PSGD-N are generally slower and even take more than one hour for a single round on \texttt{taxi} data.
	Compared to these sequentially-interactive LDP methods, non-interactive methods achieve acceptable running time.
	We argue that the performance gap between LPCT and CT is caused by \texttt{cython} acceleration in \texttt{sklearn}.
	Moreover, PHIST causes memory leakage on three data sets.
	Though histogram-based methods and tree-based methods possess the same storage complexity theoretically \citep{ma2023decision}, LPCT avoids memory leakage when $d$ is large. 
	The results show that LPCT is an efficient method.
	
	\begin{table}[!t]
		\caption{Running time (s) on real data sets.}
		\label{tab:time}
		\centering
		\resizebox{1\linewidth}{!}{
			\renewcommand{\arraystretch}{1}
			\setlength{\tabcolsep}{5pt}
			\begin{tabular}{l|l|llll|llllll}
				\toprule
				\multicolumn{1}{c|}{\multirow{2}{*}{Datasets}} & \multicolumn{1}{c|}{\multirow{2}{*}{CT-W}} & \multicolumn{2}{c}{LPCT}                               & \multicolumn{2}{c|}{LPCT-prune}                         & \multicolumn{1}{c}{\multirow{2}{*}{CT-Q}} & \multicolumn{1}{c}{\multirow{2}{*}{PHIST}} & \multicolumn{1}{c}{\multirow{2}{*}{LPDT}} & \multicolumn{1}{c}{\multirow{2}{*}{PGLM}} & \multicolumn{1}{c}{\multirow{2}{*}{PSGD-L}} & \multicolumn{1}{c}{\multirow{2}{*}{PSGD-N}} \\
				\multicolumn{1}{c|}{}          & \multicolumn{1}{c|}{}      & \multicolumn{1}{c}{Maxedge} & \multicolumn{1}{c}{CART} & \multicolumn{1}{c}{Maxedge} & \multicolumn{1}{c|}{CART} & \multicolumn{1}{c}{}      & \multicolumn{1}{c}{}       & \multicolumn{1}{c}{}      & \multicolumn{1}{c}{}      & \multicolumn{1}{c}{}        & \multicolumn{1}{c}{}        \\ \midrule
				\texttt{Anonymity}             & 3e-3       & 1e-2        & 6e-3     & 9e-3        & 1e-2      & 2e-3      & 1e-2       & 2e-3      & 2e-3      & 8e-1        & 1e0         \\
				\texttt{Diabetes}              & 1e-2       & 1e-1        & 3e-2     & 3e-2        & 3e-2      & 3e-3      & 3e-2       & 6e-3      & 3e-3      & 3e0         & 6e0         \\
				\texttt{Email}                 & 3e0        & 1e0         & 3e-2     & 2e-1        & 1e0       & 6e-2      & -          & 1e0       & 4e0       & 8e0         & 3e3         \\
				\texttt{Employ}                & 5e-3       & 9e-2        & 2e-1     & 1e0         & 9e-1      & 3e-3      & 3e-2       & 4e-2      & 4e-3      & 5e0         & 7e0         \\
				\texttt{Rice}  & 1e-2       & 2e-2        & 1e-2     & 1e-2        & 2e-2      & 2e-2      & 5e-3       & 6e-3      & 8e-3      & 2e0         & 2e0         \\
				\texttt{Census}                & 6e-1       & 1e0         & 5e0      & 6e-1        & 9e-1      & 3e-1      & -          & 2e-1      & 5e-2      & 2e1         & 3e1         \\
				\texttt{Election}              & 1e-2       & 6e-2        & 3e-2     & 3e-2        & 3e-2      & 3e-3      & 1e0        & 3e-2      & 3e-3      & 3e0         & 5e0         \\
				\texttt{Employee}              & 6e-3       & 8e-2        & 1e-1     & 4e-2        & 1e-1      & 3e-3      & 5e-2       & 5e-2      & 4e-3      & 7e0         & 1e2         \\
				\texttt{Jobs}  & 1e-1       & 1e0         & 1e0      & 6e-1        & 6e-1      & 3e-2      & 8e-1       & 5e-1      & 3e-2      & 1e2         & 1e2         \\
				\texttt{Landcover}             & 2e-1       & 1e-1        & 8e-2     & 9e-2        & 1e-1      & 5e-3      & 2e-1       & 9e-2      & 5e-3      & 1e1         & 2e1         \\
				\texttt{Taxi}  & 1e2        &    2e2         &    7e1      & 4e1         & 4e1       & 1e0       & -          & 1e2       & 1e0       & -           & -                           \\ \bottomrule
			\end{tabular}
		}
	\end{table}

	\section{Conclusion and Discussions}\label{sec:comment}

	In this work, we address the under-explored problem of public data assisted LDP learning by investigating non-parametric classification under the posterior drift assumption. 
	We propose the locally private classification tree that achieves the mini-max optimal convergence rate. 
	Moreover, we introduce a data-driven pruning procedure that avoids parameter tuning and produces an estimator with a fast convergence rate. 
	Extensive experiments on synthetic and real data sets validate the superior performance of these methods. 
	Both theoretical and experimental results suggest that public data is significantly more effective compared to data under privacy protection.
	This finding leads to practical suggestions.
	For instance, prioritizing non-private data collection should be considered. When data-collecting resources are limited, organizations can offer rewards to encourage users to share non-private data instead of collecting a large amount of privatized data.

	We summarize the ways that public data contributes to LPCT.
	The benefits lie in four-fold: 
	(1) Elimination of range parameter. 
	Since the domain is unknown in the absence of the raw data, the convergence rate and performance of partition-based estimators are degraded by a range parameter. 
	For LPCT, public data helps to eliminate the range parameter both theoretically and empirically, as demonstrated in Section \ref{sec:rangeparameter} and \ref{sec:simulationgainofpublicdata}, respectively. 
	(2) Implicit information in partition. 
	Through the partition created on public data, LPCT implicitly feeds information from public data to the estimation. 
	The performance gain is explained empirically in Section \ref{sec:simulationgainofpublicdata}. 
	(3) Augmented estimation.
	In the weighted average estimation, the public data reduces the instability caused by perturbations and produces more accurate estimations.
	Theoretically, the convergence rate can be significantly improved.
	Empirically, even a small amount of public data improves the performance significantly (Section \ref{sec:simulation} and \ref{sec:realdata}).
	(4) Data-driven pruning. 
	As shown in Section \ref{sec:plpct}, we are capable of deriving the data-driven approach only with the existence of public data.

	We discuss the potential future works and improvements. 
	The assumption that $\mathrm{P}_X$ and $\mathrm{Q}_X$ are identical can be mitigated.
	Similar to \cite{ma2023decision}, we only need the density ratio to be bounded. 
	Moreover, the marginal densities are not necessarily lower bounded, if we add a parameter to guarantee that sufficient samples are contained in each cell \citep{ma2023decision}. 
	In practice, $\mathcal{X}_{\mathrm{P}}$ and $\mathcal{X}_{\mathrm{Q}}$ can be distinctive \citep{yu2021large, gu2023choosing}.
	This problem is related to domain adaption \citep{redko2020survey} and can serve as future work direction. 
	Determining whether a data set satisfies Assumption \ref{asp:rse} is of independent interest, as demonstrated in \cite{gu2023choosing, yu2023selective}.

	\acks{We would like to thank the reviewers and action editor for their help and advice, which led to a significant improvement of the article.
		The research is supported by the Special Funds of the National Natural Science Foundation of China (Grant No. 72342010). 
		Yuheng Ma is supported by the Outstanding Innovative Talents Cultivation Funded Programs 2023 of Renmin University of China.
		This research is also supported by Public Computing Cloud, Renmin University of China.
	}

	\newpage

	\appendix

	\section{Methodology}\label{app:methodology}

	\subsection{Estimator Summarization}\label{sec:estimatorsummarization}

	\begin{table}[htbp]
		\centering
		\resizebox{0.76\linewidth}{!}{
			\renewcommand{\arraystretch}{1.2}
			\setlength{\tabcolsep}{6pt}
			\begin{tabular}{lcr}
				\toprule
				Function & Meaning &  Defined \\ \hline
				$\eta_{\mathrm{P}}^*$, $\eta_{\mathrm{Q}}^*$        & true regression function of $\mathrm{P}$ and $\mathrm{Q}$       & \eqref{equ:defofgroundtruth}             \\
				$\overline{\eta}_{\mathrm{P}}^{\pi}$        & population tree estimation       & \eqref{equ:etapi}             \\
				$\widehat{\eta}_{\mathrm{P}}^{\pi}$, $\widehat{\eta}_{\mathrm{Q}}^{\pi}$        & non-private tree estimation of $\mathrm{P}$ and $\mathrm{Q}$      & \eqref{equ:reformetapi}             \\
				$\tilde{\eta}_{\mathrm{P}}^{\pi}$       & private tree estimation based on $\mathrm{P}$      & \eqref{equ:etadppi}            \\
				$\tilde{\eta}_{\mathrm{P},\mathrm{Q}}^{\pi}$       &  private tree estimation      & \eqref{equ:twosampleeta}            \\ 
				{$\widehat{\eta}_{\mathrm{P},\mathrm{Q}}^{\pi}$     }  & {non-private two-sample tree estimation  }    & \eqref{equ:twosamplenonprivateeta}            \\ 
				$ \tilde{\eta}^{\pi_k}_{\mathrm{P}, \mathrm{Q}}$ & ancestor private tree estimation      & \eqref{equ:defofetapik}\\
				$ \tilde{\eta}^{prune}_{\mathrm{P}, \mathrm{Q}}$ & pruned  private tree estimation      & Algorithm \ref{alg:prunedeta}\\
				\bottomrule
			\end{tabular}
		}
	\end{table}
	
	\subsection{Derivation of the Closed-form Solution for Pruning Procedure}\label{sec:derivationpruning}
	
	Remind that 
	\begin{align*}
		r_{k}^j :=\sqrt{\frac{4C_1\left(2^{p_0 - k + 3}\varepsilon^{-2} n_{\mathrm{P}}\right)\vee \left( \sum_{j'}\sum_{i=1}^{n_{\mathrm{P}}}\tilde{U}_i^{\mathrm{P}j'} \right) + C_1 \lambda^2\sum_{j'}\sum_{i=1}^{n_{\mathrm{P}}}U_i^{\mathrm{Q}j'}}{\left( \sum_{j'} \sum_{i=1}^{n_{\mathrm{P}}} \tilde{U}^{\mathrm{P} j}_i + \lambda \sum_{j'} \sum_{i=1}^{n_{\mathrm{Q}}}{U}^{\mathrm{Q} j}_i\right)^2}}. 
	\end{align*}
	For each {$\mathrm{ancestor}(A^{j}_{(p_0)}, k)$}, define the quantity 
	\begin{align*}
		w^j_k = \frac{\sum_{\mathrm{ancestor}(A^{j'}_{(p_0)}, k) = \mathrm{ancestor}(A^{j}_{(p_0)}, k) }
			\sum_{i=1}^{n_{\mathrm{P}}} \tilde{U}^{\mathrm{P} j' }_i 
		}
		{ 
			\sum_{\mathrm{ancestor}(A^{j'}_{(p_0)}, k) = \mathrm{ancestor}(A^{j}_{(p_0)}, k) }\left(
			\sum_{i=1}^{n_{\mathrm{P}}} \tilde{U}^{\mathrm{P} j' }_i + \lambda\sum_{i=1}^{n_{\mathrm{Q}}}{U}^{\mathrm{Q} j'}_i\right)
		}.
	\end{align*}
	Also, let $u_k =4\left(2^{p_0 - k + 3}\varepsilon^{-2} n_{\mathrm{P}}\right)\vee \left( \sum_{j'}\sum_{i=1}^{n_{\mathrm{P}}}\tilde{U}_i^{\mathrm{P}j'} \right)$.
	When the signs of $\tilde{\eta}_{\mathrm{P}}^{\pi_k} (x) - \frac{1}{2} $ and $\widehat{\eta}_{\mathrm{Q}}^{\pi_k} (x) - \frac{1}{2}$ are identical, the object in \eqref{equ:bestlamda} becomes
	\begin{align*}
		&\sqrt{C_1}\cdot \frac{|\tilde{\eta}^{\pi_k}_{\mathrm{P}, \mathrm{Q}}(x)  - \frac{1}{2}|}{r_k^j} 
		\\
		= & 
		\frac{\left|w_k^j\left(\tilde{\eta}_{\mathrm{P}}^{\pi_k} (x) - \frac{1}{2}\right) + (1-w_k^j)\cdot \left(\widehat{\eta}_{\mathrm{Q}}^{\pi_k} (x) - \frac{1}{2}\right)\right|}
		{ 
			\sqrt{
				\frac{w_k^{j2}u_k}{\left(\sum_{j'}\sum_{i=1}^{n_{\mathrm{P}}}\tilde{U}_i^{\mathrm{P}j'}\right)^2}
				+
				\frac{(1-w_k^j)^2}{\sum_{j'}\sum_{i=1}^{n_{\mathrm{Q}}}U_i^{\mathrm{Q}j'}}
			}
		}
		\\
		\leq & 
		\sqrt{
			\frac{\left(\sum_{j'}\sum_{i=1}^{n_{\mathrm{P}}}\tilde{U}_i^{\mathrm{P}j'}\right)^2}{u_k}
			\left(\tilde{\eta}^{\mathrm{P}}_{\pi_k} (x) - \frac{1}{2}\right)^2 
			+
			\sum_{j'}\sum_{i=1}^{n_{\mathrm{Q}}}U_i^{\mathrm{Q}j'} 
			\cdot 
			\left(\widehat{\eta}^{\mathrm{Q}}_{\pi_k} (x) - \frac{1}{2}\right)^2}
	\end{align*}
	where the inequality follows from Cauthy's inequality. 
	The maximum is attained when 
	\begin{align*}
		& \frac{w_k^{j2}u_k}{\left(\sum_{j'}\sum_{i=1}^{n_{\mathrm{P}}}\tilde{U}_i^{\mathrm{P}j'}\right)^2}
		\cdot \sum_{j'}\sum_{i=1}^{n_{\mathrm{Q}}}U_i^{\mathrm{Q}j'} 
		\cdot 
		\left(\widehat{\eta}_{\mathrm{Q}}^{\pi_k} (x) - \frac{1}{2}\right)^2
		\\
		= &
		\frac{(1-w_k^j)^2}{\sum_{j'}\sum_{i=1}^{n_{\mathrm{Q}}}U_i^{\mathrm{Q}j'}}
		\cdot 
		\frac{\left(\sum_{j'}\sum_{i=1}^{n_{\mathrm{P}}}\tilde{U}_i^{\mathrm{P}j'}\right)^2}{u_k}
		\left(\tilde{\eta}_{\mathrm{P}}^{\pi_k} (x) - \frac{1}{2}\right)^2 .
	\end{align*}
	By bringing in $w_k^j$, this leads to the solution 
	\begin{align*}
		\lambda =   \frac{u_k}{\sum_{j'}\sum_{i=1}^{n_{\mathrm{P}}}\tilde{U}_i^{\mathrm{P}j'}}\cdot \frac{\widehat{\eta}_{\mathrm{Q}}^{\pi_k} (x) - \frac{1}{2}}{\tilde{\eta}_{\mathrm{P}}^{\pi_k} (x) - \frac{1}{2}}.
	\end{align*}
	When $\tilde{\eta}_{\mathrm{P}}^{\pi_k} (x) - \frac{1}{2} $ and $\widehat{\eta}_{\mathrm{Q}}^{\pi_k} (x) - \frac{1}{2}$ have different signs, the object in \eqref{equ:bestlamda} is maximized at the extremum, i.e. 
	when $\lambda = 0$ or $\lambda = \infty$. 
	The maximum value is 
	\begin{align*}
		\sqrt{\frac{ \max \left(
				\frac{\left(\sum_{j'}\sum_{i=1}^{n_{\mathrm{P}}}\tilde{U}_i^{\mathrm{P}j'}\right)^2}{u_k}
				\left(\tilde{\eta}_{\mathrm{P}}^{\pi_k} (x) - \frac{1}{2}\right)^2 
				,
				\sum_{j'}\sum_{i=1}^{n_{\mathrm{Q}}}U_i^{\mathrm{Q}j'} 
				\cdot 
				\left(\widehat{\eta}_{\mathrm{Q}}^{\pi_k} (x) - \frac{1}{2}\right)^2\right)}{C_1}}.
	\end{align*}
	
	\section{Proofs}\label{app:proofs}
	
	\subsection{Proof of Proposition \ref{thm:dppartition}}
	
	\begin{proof}\textbf{of Proposition  \ref{thm:dppartition}}\;\;
		Since all mechanisms $\mathrm{R}_i$ are completely the same, we use $\mathrm{R}$ for notation simplicity.
		For each conditional distribution $\mathrm{R}\left(\tilde{U}_i, \tilde{V}_i| X_i = x, Y_i = y\right)$, we can compute the density ratio as 
		\begin{align}\nonumber
			& \sup_{x,x'\in \mathbb{R}^d, \;y,y'\in \{0,1\}, S\in\sigma(\mathcal{S})}\frac{\mathrm{R}\left((\tilde{U}, \tilde{V})\in S| X_i = x, Y_i = y \right)}{\mathrm{R}\left((\tilde{U}, \tilde{V})\in S| X_i = x', Y_i = y' \right)} \\
			\leq  &  \sup_{x,x'\in \mathbb{R}^d} \frac{d\mathrm{R}(\tilde{U}|X_i = x) }{d\mathrm{R}(\tilde{U}|X_i = x') } \cdot  \sup_{x,x'\in \mathbb{R}^d, y,y'\in \{0,1\}} \frac{d\mathrm{R}(\tilde{V}|X_i = x, Y_i = y)}{d\mathrm{R}(\tilde{V}|X_i = x', Y_i = y')}.\label{equ:densityratiodecomposition}
		\end{align}
		The first part is characterized by the Laplace mechanism. 
		Since the conditional density is identical if $x$ and $x'$ belongs to a same $A^j_(p)$, we have 
		\begin{align*}
			\sup_{x,x'\in \mathbb{R}^d} \frac{d\mathrm{R}(\tilde{U}|X_i = x) }{d\mathrm{R}(\tilde{U}|X_i = x') } = \sup_{u,u'\in \{0,1\}^{|\mathcal{I}|}} \frac{d\mathrm{R}(\tilde{U}|U_i = u) }{d\mathrm{R}(\tilde{U}|U_i = u') }. 
		\end{align*}
		The right-hand side can be computed as
		\begin{align*}
			\frac{d\mathrm{R}(\tilde{U}|U_i = u) }{d\mathrm{R}(\tilde{U}|U_i = u') } = \frac{\prod_{j}\exp\left(\frac{\varepsilon}{4}|\tilde{u}_i^j-u^j|\right)}{\prod_{j} \exp\left(\frac{\varepsilon}{4}|\tilde{u}_i^j-u^{j'}|\right)} \leq  \prod_{j}\exp\left(\frac{\varepsilon}{4}|u^j-u^{j'}|\right). 
		\end{align*}
		By definition, $u$ and $u'$ are one-hot vectors and differ at most on two entries. As a result, 
		\begin{align}\label{equ:densityratioU}
			\sup_{x,x'\in \mathbb{R}^d} \frac{d\mathrm{R}(\tilde{U}|X_i = x) }{d\mathrm{R}(\tilde{U}|X_i = x') } \leq \exp\left(\frac{\varepsilon}{2}\right)
		\end{align}
		For the second part, note that $v$ and $v'$ also differ at most on two entries. Then with analogous steps, we have
		\begin{align}
			\sup_{x,x'\in \mathbb{R}^d, y,y'\in \{0,1\}} \frac{d\mathrm{R}(\tilde{V}|X_i = x, Y_i = y)}{d\mathrm{R}(\tilde{V}|X_i = x', Y_i = y')} \leq \exp\left(\frac{\varepsilon}{2}\right). \label{equ:densityratioresultY}
		\end{align}
		Bringing \eqref{equ:densityratioU} and \eqref{equ:densityratioresultY} into \eqref{equ:densityratiodecomposition} yields the desired conclusion. 
	\end{proof}
	
	\subsection{Proof of Theorem \ref{thm:lowerbound}}

	\begin{proof}\textbf{of Theorem \ref{thm:lowerbound}}\;\;
		The proof is analogous to \citet[Theorem 3.2]{cai2021transfer}. 
		We first construct two classes of distributions $\mathrm{P}_{\sigma}$ and $\mathrm{Q}_{\sigma}$ for $\sigma\in \{-1, 1\}^{m}$. 
		We then combine the information inequality under privacy constraint \citep{duchi2018minimax} and Assouad's Lemma \citep{tsybakov2008introduction}. 
		We first define two classes of distributions by construction.  
		Define the useful quantity 
		\begin{align}\label{equ:lowerboundimportantvalue}
			r = c_r \left(n_{\mathrm{P}} \left(e^{\varepsilon} - 1\right)^2 + n_{\mathrm{Q}}^{\frac{2\alpha + 2d}{2 \gamma\alpha + d}}\right)^{-\frac{1}{2\alpha + 2 d}}, \;\; w = c_w r^d, \;\; m = \lfloor c_m r^{\alpha \beta - d}\rfloor
		\end{align}
		where $c_r$, $c_w$, and $c_m$ are some universal constants. 
		Let
		\begin{align*}
			G=\left\{\left(6 k_1 r, 6 k_2 r, \ldots, 6 k_d r\right), k_i=1,2, \ldots,\left\lfloor(6 r)^{-1}\right\rfloor, i=1,2, \ldots, d\right\} .
		\end{align*}
		be a grid of $\lfloor(6 r)^{-1}\rfloor^d$ points in the unit cube. Denote $x_1, x_2, \ldots, x_{\lfloor(6 r)^{-1}\rfloor^d}$ as points in $G$.
		Since $m \ll \lfloor(6 r)^{-1}\rfloor^d$, we are interested in $B\left(x_k, 2 r\right), k=1,2, \ldots, m$ which are balls with radius $2 r$ centered at $x_k$. 
		Let $B_c= [0,1]^d \backslash \bigcup_{k=1}^m B\left(x_k, 2 r\right)$ denote the points that aren't in any of the $m$ balls.
		Define function $g(\cdot)$ on $[0, \infty)$ :
		\begin{align*}
			g(z)= \begin{cases}1 & 0 \leq z<1 \\ 2-z & 1 \leq z<2 \\ 0 & z \geq 2\end{cases}
		\end{align*}
		And define
		\begin{align*}
			h_{\mathrm{P}}(z)=C_{\alpha} r^{\alpha} g^{\alpha}(z / r) , \;\;\; h_{\mathrm{Q}}(z)=C_{\gamma} C_{\alpha}^{\gamma} r^{\alpha \gamma} g^{\alpha \gamma}(z / r) 
		\end{align*}
		where $C_{\alpha}$, $C_{\gamma}$ are some universal constants. 
		Define the hypercube $\mathcal{H}$ of pairs $\left(P_\sigma, Q_\sigma\right)$ by
		\begin{align*}
			\mathcal{H}=\left\{\left(\mathrm{P}_\sigma, \mathrm{Q}_\sigma\right), \sigma=\left(\sigma_1, \sigma_2, \ldots, \sigma_m\right) \in\{-1,1\}^m\right\}
		\end{align*}
		with $\mathrm{P}_\sigma$ and $\mathrm{Q}_\sigma$ defined as follows. 
		For conditional density, we let
		\begin{align*}
			& \eta_{\mathrm{P}, \sigma}(x)= \begin{cases}\frac{1}{2}\left(1+\sigma_k h_{\mathrm{P}}\left(\left\|x-x_k\right\|\right)\right) & \text { if } x \in B\left(x_k, 2 r\right) \text { for some } k=1,2, \ldots, m \\
				\frac{1}{2} & \text{ if } x \in B_c.\end{cases} \\
			& \eta_{\mathrm{Q}, \sigma}(x)= \begin{cases}\frac{1}{2}\left(1+\sigma_k h_{\mathrm{Q}}\left(\left\|x-x_k\right\|\right)\right) & \text { if } x \in B\left(x_k, 2 r\right) \text { for some } k=1,2, \ldots, m \\
				\frac{1}{2} & \text{ if } x \in B_c.\end{cases}
		\end{align*}
		Let $\mathrm{P}_{\sigma, X}, \mathrm{Q}_{\sigma, X}, \sigma \in\{-1,1\}^m$ all have the same marginal distribution on $\mathcal{X}$, with density $d\mathrm{P}_X(x)$ defined as follows:
		\begin{align*}
			d\mathrm{P}_X(x)= \begin{cases} {w}\cdot r^{-d} & \text { if } x \in B\left(x_k, r\right) \text { for some } k=1,2, \ldots, m \\ c_{B_c}(1-m w) & \text { if } x \in B_c \\ 0 & \text { otherwise. }\end{cases}
		\end{align*}
		where $c_{B_c}$ is an normalizing constant that makes $d\mathrm{P}$ well-defined. 
		Similar to \cite{cai2021transfer}, with a proper choice of all constants, $\mathcal{H}$ is a subset of $\mathcal{F}$ that satisfies \ref{asp:rse}, \ref{asp:alphaholder}, and \ref{asp:margin}. 
		Then it suffices to show that 
		\begin{align*}
			\inf_{\widehat{f}}\sup_{\mathcal{H}}\mathbb{E}_{X, Y}\left[\mathcal{R}_{\mathrm{P}_{\sigma}} (\widehat{f}) - \mathcal{R}_{\mathrm{P}_{\sigma}}^*  \right] \gtrsim \left(n_{\mathrm{P}} \left(e^{\varepsilon} - 1\right)^2 + n_{\mathrm{Q}}^{\frac{2 \alpha + 2 d}{2 \gamma\alpha + d}}\right)^{-\frac{(\beta+1)\alpha}{2\alpha + 2d}}. 
		\end{align*}
		Let $TV(\cdot, \cdot)$ and $H(\cdot, \cdot)$ denote the total variation distance and the Hellinger distance between measures, respectively. 
		Let $Ham(\cdot, \cdot)$ denote the hamming distance between ordered indices. 
		Let $\sigma$ and $\sigma'$ be two indices in $\{-1,1\}^{m}$ such that $Ham(\sigma, \sigma') = 1$, i.e. they differ at only one element $\sigma_k$ and $\sigma_k'$. 
		We compute the total variation between the induced $\mathrm{P}_{\sigma}$ and $\mathrm{P}_{\sigma'}$. 
		By Scheffe's theorem \citep{tsybakov2008introduction}, there holds
		\begin{align*}
			TV(\mathrm{P}_{\sigma}, \mathrm{P}_{\sigma'}) = \frac{1}{2} \int_{\mathcal{X}\times \mathcal{Y}} \left|d\mathrm{P}_{\sigma}(x,y)- d\mathrm{P}_{\sigma'}(x,y)\right|.
		\end{align*}
		Since the marginal distributions for both measures are identical, we have
		\begin{align*}
			TV(\mathrm{P}_{\sigma}, \mathrm{P}_{\sigma'}) =  & \int_{\mathcal{X}} \left|\eta_{\mathrm{P}, \sigma}(x) - \eta_{\mathrm{P}, \sigma'}(x)\right|d\mathrm{P}_{\sigma, X}(x)\\
			= & \int_{B(x_k, r) } C_{\alpha} r^{\alpha} \cdot w \cdot r^{-d} dx = C_{\alpha}wr^{\alpha}.
		\end{align*}
		According to \cite{cai2021transfer}, we have 
		\begin{align}\label{equ:orderofh2}
			H^2(\mathrm{Q}_{\sigma}, \mathrm{Q}_{\sigma'}) = C_{\gamma}^2C_{\alpha}^{2\gamma}wr^{2\alpha\gamma}.
		\end{align}
		Let conditional distribution $\mathrm{R} = \{\mathrm{R}_i\}^{n_{\mathrm{P}}}_{i = 1}$ be any transformation on $\mathcal{X}^{n_\mathrm{P}}\times \mathcal{Y}^{n_\mathrm{P}}$ that is sequentially-interactive $\varepsilon$-local differentially private as defined in Definition \ref{def:ldp}. 
		Denote $\mathrm{R}\mathrm{P}_{\sigma}^{n_\mathrm{P}}$ by the random variable mapped compositionally through $\mathrm{R}$ and $\mathrm{P}_{\sigma}^{n_\mathrm{P}}$.
		Therefore, for $Ham(\sigma,\sigma') = 1$, there holds
		\begin{align}\label{equ:decompositionofh2}
			H^2(\mathrm{R}\mathrm{P}_{\sigma}^{n_{\mathrm{P}}} \times \mathrm{Q}_{\sigma}^{n_{\mathrm{Q}}}, \mathrm{R}\mathrm{P}_{\sigma'}^{n_{\mathrm{P}}} \times \mathrm{Q}_{\sigma'}^{n_{\mathrm{Q}}}) \leq H^2(\mathrm{R}\mathrm{P}_{\sigma}^{n_{\mathrm{P}}}, \mathrm{R}\mathrm{P}_{\sigma'}^{n_{\mathrm{P}}}) + H^2( \mathrm{Q}_{\sigma}^{n_{\mathrm{Q}}} \times \mathrm{Q}_{\sigma'}^{n_{\mathrm{Q}}}). 
		\end{align}
		For the first term in \eqref{equ:decompositionofh2}, combining \citet[Lemma 2.4]{tsybakov2008introduction} and \citet[Corollary 3]{duchi2018minimax} yields
		\begin{align*}
			H^2(\mathrm{R}\mathrm{P}_{\sigma}^{n_{\mathrm{P}}}, \mathrm{R}\mathrm{P}_{\sigma'}^{n_{\mathrm{P}}}) \leq 4 (e^{\alpha}-1)^2 n_{\mathrm{P}} TV^2(\mathrm{P}_{\sigma}, \mathrm{P}_{\sigma'}) = 4 C_{\alpha}^2 (e^{\alpha}-1)^2 n_{\mathrm{P}} w^2r^{2\alpha}. 
		\end{align*}
		This together with \eqref{equ:lowerboundimportantvalue} lead to 
		\begin{align}\nonumber
			H^2(\mathrm{R}\mathrm{P}_{\sigma}^{n_{\mathrm{P}}}, \mathrm{R}\mathrm{P}_{\sigma'}^{n_{\mathrm{P}}}) \leq & 4 C_{\alpha}^2c_w^2 (e^{\varepsilon}-1)^2 n_{\mathrm{P}} r^{2\alpha + 2d}  \\
			= & 4 C_{\alpha}^2c_w^2 c_r^{2\alpha + 2 d}(e^{\varepsilon}-1)^2 n_{\mathrm{P}}  \left(n_{\mathrm{P}} \left(\varepsilon - 1\right)^2 + n_{\mathrm{Q}}^{\frac{2\alpha + 2d}{2 \gamma\alpha + d}}\right)^{-1}
			\leq  4 C_{\alpha}^2c_w^2 c_r^{2\alpha + 2 d} .\label{equ:lowerboundfirstpartupperbound}
		\end{align} 
		For the second term in \eqref{equ:decompositionofh2}, \eqref{equ:orderofh2} implies
		\begin{align*}
			H^2( \mathrm{Q}_{\sigma}^{n_{\mathrm{Q}}} \times \mathrm{Q}_{\sigma'}^{n_{\mathrm{Q}}}) \leq  n_{\mathrm{Q}} H^2( \mathrm{Q}_{\sigma}\times \mathrm{Q}_{\sigma'}) = C_{\gamma}^2C_{\alpha}^{2\gamma} n_{\mathrm{Q}} wr^{2\alpha\gamma}. 
		\end{align*}
		Again, bringing in \eqref{equ:lowerboundimportantvalue} leads to
		\begin{align}\nonumber
			H^2( \mathrm{Q}_{\sigma}^{n_{\mathrm{Q}}} \times \mathrm{Q}_{\sigma'}^{n_{\mathrm{Q}}}) \leq c_w C_{\gamma}^2C_{\alpha}^{2\gamma} n_{\mathrm{Q}} r^{2\alpha\gamma + d} = &   C_{\gamma}^2C_{\alpha}^{2\gamma} c_w n_{\mathrm{Q}}\left(n_{\mathrm{P}} \left(e^{\varepsilon} - 1\right)^2 + n_{\mathrm{Q}}^{\frac{2\alpha + 2d}{2 \gamma\alpha + d}}\right)^{-{\frac{2 \gamma\alpha + d}{2\alpha + 2d}}}\\ 
			\leq & C_{\gamma}^2C_{\alpha}^{2\gamma} c_w.  \label{equ:lowerboundsecondpartupperbound} 
		\end{align}
		Bringing \eqref{equ:lowerboundfirstpartupperbound}, \eqref{equ:lowerboundsecondpartupperbound}  into \eqref{equ:decompositionofh2}, we have 
		\begin{align*}
			H^2(\mathrm{R}\mathrm{P}_{\sigma}^{n_{\mathrm{P}}} \times \mathrm{Q}_{\sigma}^{n_{\mathrm{Q}}}, \mathrm{R}\mathrm{P}_{\sigma'}^{n_{\mathrm{P}}} \times \mathrm{Q}_{\sigma'}^{n_{\mathrm{Q}}}) \leq \frac{1}{4}
		\end{align*}
		as long as the constants are small enough such that $\max (4 C_{\alpha}^2c_w^2 c_r^{2\alpha + 2 d} , C_{\gamma}^2C_{\alpha}^{2\gamma} c_w) \leq 1/8$.
		Let 
		\begin{align*}
			\widehat{\sigma} = \underset{{\sigma} \in \{-1,1\}^m}{\arg\min} \mathcal{R}_{\mathrm{P}_{\sigma}}(\widehat{f}) - \mathcal{R}^*_{\mathrm{P}_{\sigma}}. 
		\end{align*}
		Apply \citet[Theorem 2.12]{tsybakov2008introduction}, there holds 
		\begin{align}\label{equ:applyassouadlemma}
			\inf _{\widehat{\sigma}} \max _{\sigma } E_{X,Y \sim \mathbf{P}_{\sigma}} Ham(\widehat{\sigma}, \sigma) \geq \frac{m}{2}(1-\frac{\sqrt{2}}{2}).
		\end{align}
		Next, we reduce the mini-max risk into this quantity. 
		Since 
		\begin{align*}
			\mathcal{R}_{\mathrm{P}_{\sigma}}({f}_{\widehat{\sigma}}^*) - \mathcal{R}^* \leq  \mathcal{R}_{\mathrm{P}_{\sigma}}(\widehat{f}) - \mathcal{R}^* + \mathcal{R}_{\mathrm{P}_{\widehat{\sigma}}}(\widehat{f}) - \mathcal{R}^*
			\leq  2 \left(\mathcal{R}_{\mathrm{P}_{\sigma}}(\widehat{f}) - \mathcal{R}^*\right),
		\end{align*}
		for any estimator $\widehat{f}$, we can reduce the excess risk as
		\begin{align*}
			\mathbb{E}_{X,Y}\left[\mathcal{R}_{\mathrm{P}_{\sigma}} (\widehat{f}) - \mathcal{R}_{\mathrm{P}_{\sigma}}^*  \right] 
			\geq & \frac{1}{2}   \mathbb{E}_{X,Y}\left[\mathcal{R}_{\mathrm{P}_{\sigma}} ({f}_{\widehat{\sigma}}^*) - \mathcal{R}_{\mathrm{P}_{\sigma}}^*  \right] \\
			= & \frac{1}{2} \mathbb{E}_{X,Y}\left[ \int_{\mathcal{X}} \left|\eta_{\sigma}(x)-\frac{1}{2}\right|\eins\left\{ f_{\sigma}^*(x) \neq f_{\widehat{\sigma}}^*(x)\right\}d\mathrm{P}_{\sigma,X}(x) \right]\\
			= & \frac{1}{4} \mathbb{E}_{X,Y}\left[\sum_{\sigma_k\neq \widehat{\sigma}_k} 
			\int_{B(x_k, r)} c_{\alpha} r^{\alpha } d\mathrm{P}_{\sigma,X}(x)\right] = \frac{c_{\alpha} w r^{\alpha}}{4} \mathbb{E}_{X,Y}\left[Ham(\sigma, \widehat{\sigma})\right].
		\end{align*}
		This together with \eqref{equ:applyassouadlemma} yield 
		\begin{align*}
			\inf _{\widehat{f}} \max _{\mathcal{H} } \mathbb{E}_{X,Y}\left[\mathcal{R}_{\mathrm{P}_{\sigma}} (\widehat{f}) - \mathcal{R}_{\mathrm{P}_{\sigma}}^*  \right] \geq &\;\; \frac{1}{8} (1-\frac{\sqrt{2}}{2}) c_{\alpha} w r^{\alpha} m \\
			\gtrsim & \;\; r^{(\beta+1)\alpha} = \left(n_{\mathrm{P}} \left(e^{\varepsilon} - 1\right)^2 + n_{\mathrm{Q}}^{\frac{2 \alpha + 2 d}{2 \gamma\alpha + d}}\right)^{-\frac{(\beta+1)\alpha}{2\alpha + 2d}}. 
		\end{align*}

	\end{proof}
	
	\subsection{Proof of Theorem \ref{thm:utility}}\label{sec:proofoftheoremutility}

	To prove Theorem \ref{thm:utility}, we rely on the following decomposition.
	\begin{align}
		\tilde{\eta}^{\pi}_{\mathrm{P}, \mathrm{Q}}(x) - \frac{1}{2} 
		= & \underbrace{\tilde{\eta}^{\pi}_{\mathrm{P}, \mathrm{Q}}(x) - \widehat{\eta}^{\pi}_{\mathrm{P}, \mathrm{Q}}(x)}_{\textbf{Privatized Error}} 
		+ \underbrace{\widehat{\eta}^{\pi}_{\mathrm{P}, \mathrm{Q}}(x) - \overline{\eta}^{\pi}_{\mathrm{P}, \mathrm{Q}}
			(x)}_{\textbf{Sample Error}}
		+ \underbrace{\overline{\eta}^{\pi}_{\mathrm{P}, \mathrm{Q}}(x) -\frac{1}{2}}_{\textbf{Signal Strength}} \label{equ:decompositionofexcessrisk}
	\end{align}
	where we define 
	\begin{align}\label{equ:defetahatpi}
		\widehat{\eta}^{\pi}_{\mathrm{P}, \mathrm{Q}}(x) = 
		\sum_{j\in\mathcal{I}_p}\eins\{x\in A^j_{(p)}\}\frac{
			\sum_{i=1}^{n_{\mathrm{P}}} {V}^{\mathrm{P} j }_i + \lambda \sum_{i=1}^{n_{\mathrm{Q}}}{V}^{\mathrm{Q} j}_i
		}
		{ 
			\sum_{i=1}^{n_{\mathrm{P}}} {U}^{\mathrm{P} j }_i +  \lambda\sum_{i=1}^{n_{\mathrm{Q}}}{U}^{\mathrm{Q} j}_i
		}
	\end{align}
	\begin{align}\label{equ:defetaoverlinepi}
		\overline{\eta}^{\pi}_{\mathrm{P}, \mathrm{Q}}(x) = 
		\sum_{j\in\mathcal{I}_p}\eins\{x\in A^j_{(p)}\}\frac{
			n_{\mathrm{P}} \int_{A^j_{(p)}} \eta^*_{\mathrm{P}}(x')d\mathrm{P}_X(x') + \lambda n_{\mathrm{Q}} \int_{A^j_{(p)}} \eta^*_{\mathrm{Q}}(x')d\mathrm{Q}_X(x')
		}
		{ 
			n_{\mathrm{P}} \int_{A^j_{(p)}} d\mathrm{P}_X(x') + \lambda n_{\mathrm{Q}} \int_{A^j_{(p)}} d\mathrm{Q}_X(x')
		}.
	\end{align}
	Correspondingly, we let $\widehat{f}^{\pi}_{\mathrm{P}, \mathrm{Q}}(x) = \eins\left(  \widehat{\eta}^{\pi}_{\mathrm{P}, \mathrm{Q}}(x) > \frac{1}{2}\right)$ and $\overline{f}^{\pi}_{\mathrm{P}, \mathrm{Q}}(x) = \eins\left(  \overline{\eta}^{\pi}_{\mathrm{P}, \mathrm{Q}}(x) > \frac{1}{2}\right)$.
	
	Loosely speaking, the first error term quantifies the degradation brought by adding noises to the estimator, which we call the privatized error. The second term corresponds to the expected estimation error brought by the randomness of the data, which we call the sample error. 
	The last term is called signal strength, which quantifies the uncertainty of the population regression function.
	In general, we want the absolute value of the third term to be large, (i.e. the population regression function is sure about the label) while the first and the second terms to be small (i.e. the noise is small). 
	In the proof of Theorem \ref{thm:utility}, we specify a region $\Omega_n$ containing points with strong signal strength.
	We show that the first two terms are dominated by the signal strength on $\Omega_n$. 
	The following three lemmas provide bounds for each of the three terms. 
	
	\begin{lemma}[\textbf{Bounding of Privatized Error}]\label{lem:boundingofprivatisederror}
		Let $\pi$ be constructed through the max-edge partition in Section \ref{sec:partition} with depth $p$.
		Let $\tilde{\eta}^{\pi}_{\mathrm{P},\mathrm{Q}}$ be the locally private two-sample tree classifier in \eqref{equ:twosampleeta}. 
		Let $\widehat{\eta}^{\pi}_{\mathrm{P},\mathrm{Q}}$ be defined in \eqref{equ:defetahatpi}. 
		Then for $n_{\mathrm{P}}^{-\frac{\alpha}{2\alpha + 2d}}\lesssim \varepsilon \lesssim n_{\mathrm{P}}^{\frac{d}{4\alpha + 2d}}$,  under the condition of Theorem \ref{thm:utility}, there holds
		\begin{align*}
			\| \tilde{\eta}^{\pi}_{\mathrm{P},\mathrm{Q}} - \widehat{\eta}^{\pi}_{\mathrm{P},\mathrm{Q}} \|_{\infty}
			\lesssim \frac{ 2^p \cdot \sqrt{n_{\mathrm{P}} \log (n_{\mathrm{P}} +  n_{\mathrm{Q}})}}{\varepsilon (n_{\mathrm{P}} + \lambda n_{\mathrm{Q}})
			}
		\end{align*}
		with probability $\mathrm{P}^{n_{\mathrm{P}}}\otimes \mathrm{Q}^{n_{\mathrm{Q}}}\otimes\mathrm{R}$ at least $1-3/(n_{\mathrm{P}} + n_{\mathrm{Q}})^2$.
	\end{lemma}

	\begin{lemma}[\textbf{Bounding of Sample Error}]\label{lem:boundingofsampleerror}
		Let $\pi$ be constructed through the max-edge partition in Section \ref{sec:partition} with depth $p$.
		Let $\widehat{\eta}^{\pi}_{\mathrm{P},\mathrm{Q}}$ and $\overline{\eta}^{\pi}_{\mathrm{P},\mathrm{Q}}$ be defined in \eqref{equ:defetahatpi} and \eqref{equ:defetaoverlinepi}. 
		Then for $n_{\mathrm{P}}^{-\frac{\alpha}{2\alpha + 2d}}\lesssim \varepsilon \lesssim n_{\mathrm{P}}^{\frac{\alpha}{4\alpha + 2d}}$,  under the condition of Theorem \ref{thm:utility}, there holds
		\begin{align*}
			\|\widehat{\eta}^{\pi}_{\mathrm{P},\mathrm{Q}} - \overline{\eta}^{\pi}_{\mathrm{P},\mathrm{Q}}\|_{\infty}  \lesssim \sqrt{\frac{2^{p} (n_{\mathrm{P}} + \lambda^2 n_{\mathrm{Q}}) \cdot \log (n_{\mathrm{P}} +  n_{\mathrm{Q}}) }{ (n_{\mathrm{P}} + \lambda n_{\mathrm{Q}})^2 }}
		\end{align*}
		with probability $\mathrm{P}^{n_{\mathrm{P}}}\otimes \mathrm{Q}^{n_{\mathrm{Q}}}\otimes\mathrm{R}$ at least $1-3/(n_{\mathrm{P}} + n_{\mathrm{Q}})^2$.
	\end{lemma}

	\begin{lemma}[\textbf{Bounding of Signal Strength}]\label{lem:boundingofapproximationerror}
		Let $\pi$ be constructed through the max-edge partition in Section \ref{sec:partition} with depth $p$. 
		Define the region 
		\begin{align}\label{equ:defofregiondelta}
			\Delta_n := \left\{x \in \mathcal{X} \bigg|  \left|\eta_{\mathrm{P}}^*(x) - {1}/{2}\right| \geq c_{\Delta}\cdot 2^{-p \alpha / d }\right\}
		\end{align}
		for some constant $c_{\Delta}$ large enough. 
		Let $\overline{\eta}^{\pi}_{\mathrm{P},\mathrm{Q}}$ be defined in \eqref{equ:defetaoverlinepi}. 
		Then under the condition of Theorem \ref{thm:utility}, there holds
		\begin{align*}
			\left(\overline{\eta}^{\pi}_{\mathrm{P},\mathrm{Q}}(x) -\frac{1}{2}\right)\cdot \left({\eta}^{*}_{\mathrm{P}}(x) -\frac{1}{2}\right)> 0 ,\;\;\;\; \text{and} \;\;\;\; \left|\overline{\eta}^{\pi}_{\mathrm{P},\mathrm{Q}}(x) -\frac{1}{2}\right|\gtrsim  \frac{n_{\mathrm{P}} 2^{-p\alpha/d} +\lambda n_{\mathrm{Q}} 2^{-p\gamma\alpha/d} }{n_{\mathrm{P}} 
				+ \lambda n_{\mathrm{Q}} }
		\end{align*} for all $x \in \Delta_n$.
		
	\end{lemma}

	We also need the following technical lemma.

	\begin{lemma}\label{lem:boudnofnlamdasq}
		Suppose $n_{\mathrm{P}}^{-\frac{\alpha}{2\alpha + 2d}}\lesssim \varepsilon \lesssim n_{\mathrm{P}}^{\frac{\alpha}{4\alpha + 2d}}$. 
		If we let
		\begin{align*}
			\lambda \asymp \delta^{\frac{(\gamma - 1)\alpha}{d} + (\frac{(\gamma + 1)\alpha}{d} -1)\wedge 0} \asymp \left( \frac{n_{\mathrm{P}} \varepsilon^2}{\log (n_{\mathrm{P}} + n_{\mathrm{Q}})} + \left(\frac{n_{\mathrm{Q}}}{\log (n_{\mathrm{P}} + n_{\mathrm{Q}})}\right)^{\frac{2\alpha + 2d}{2 \gamma\alpha + d}}\right)^{-\frac{(\gamma\alpha - \alpha )\wedge (2\gamma \alpha - d)}{2\alpha + 2 d}},
		\end{align*}
		then the following statements hold:
		\begin{align*}
			\frac{\varepsilon^2 (n_{\mathrm{P}} + \lambda^2 n_{\mathrm{Q}})}{\log (n_{\mathrm{P}} +  n_{\mathrm{Q}})}\gtrsim \delta^{-\frac{ 2\alpha + 2d }{d }} , \;\;\; \frac{\left(n_{\mathrm{P}} + \lambda^2 n_{\mathrm{Q}}\right)\log (n_{\mathrm{P}} +  n_{\mathrm{Q}}) }{(n_{\mathrm{P}} + \lambda n_{\mathrm{Q}})^2} \lesssim \delta, \;\;\; \frac{\sqrt{n_{\mathrm{P}} \varepsilon^{-2} \log (n_{\mathrm{P}} +  n_{\mathrm{Q}})} }{(n_{\mathrm{P}} + \lambda n_{\mathrm{Q}})} \lesssim \delta.
		\end{align*}
	\end{lemma}

	All proofs of the above lemmas are provided in Section \ref{sec:proofoferroranalysis}.

	\begin{proof}\textbf{of Theorem  \ref{thm:utility}}\;\;
		Remind the definition in \eqref{equ:defofregiondelta}. 
		Theorem \ref{thm:utility} takes $2^{p}\asymp \delta^{-1}$ and $\lambda \asymp \delta^{(\gamma - 1) \alpha / d }$. 
		We prove the theorem by separately considering the region
		\begin{align*}
			\Omega_n := \left\{x \in \mathcal{X} \bigg|  \left|\eta_{\mathrm{P}}^*(x) - {1}/{2}\right| \geq c_{\Omega}\cdot \delta^{ \alpha / d }\right\}
		\end{align*}
		for constant $ c_{\Omega}$  such that $\Omega_n \subset \Delta_n$. 
		In this case, suppose we have $2^{-p}=  
		C_p\delta$.
		Then it suffices to take $c_{\Omega} = 2 C_p^{\alpha/d} \cdot c_{\Delta} $. 
		
		For $x\in \Omega_n$, we show that $\tilde{f}^{\pi}_{\mathrm{P},\mathrm{Q}}(x) = f^*_{\mathrm{P}}(x)$. 
		Without loss of generality, let $f^*_{\mathrm{P}}(x) = 1$. 
		By Lemma \ref{lem:boundingofprivatisederror}, \ref{lem:boundingofsampleerror}, and \ref{lem:boundingofapproximationerror}
		we have 
		\begin{align}\nonumber
			\tilde{f}^{\pi}_{\mathrm{P},\mathrm{Q}}(x) -\frac{1}{2} \geq & \;\overline{\eta}^{\pi}_{\mathrm{P}, \mathrm{Q}}(x) - \frac{1}{2}  - \|\overline{\eta}^{\pi}_{\mathrm{P}, \mathrm{Q}}  - \widehat{\eta}^{\pi}_{\mathrm{P}, \mathrm{Q}}\|_{\infty} - \| \widehat{\eta}^{\pi}_{\mathrm{P}, \mathrm{Q}} - \tilde{\eta}^{\pi}_{\mathrm{P}, \mathrm{Q}}\|_{\infty}\\
			\gtrsim & \;  \frac{n_{\mathrm{P}} 2^{-p\alpha/d} +\lambda n_{\mathrm{Q}} 2^{-p\gamma\alpha/d} }{n_{\mathrm{P}} 
				+ \lambda n_{\mathrm{Q}} }  - \frac{ 2^p \cdot \sqrt{n_{\mathrm{P}}}}{\varepsilon (n_{\mathrm{P}} + \lambda n_{\mathrm{Q}})
			}  -    \sqrt{\frac{2^{p} (n_{\mathrm{P}} + \lambda^2 n_{\mathrm{Q}}) \cdot \log (n_{\mathrm{P}} +  n_{\mathrm{Q}}) }{ (n_{\mathrm{P}} + \lambda n_{\mathrm{Q}})^2 }}  \label{equ:signofetadpwithcomega}
		\end{align}
		with probability $1-3/(n_{\mathrm{P}} + n_{\mathrm{Q}})^2 $ with respect to $\mathrm{P}^{n_{\mathrm{P}}}\otimes\mathrm{Q}^{n_{\mathrm{Q}}}\otimes \mathrm{R} $. 
		We first compare the signal strength with the privatized error.
		From the conditions in Theorem \ref{thm:utility}, we have the fact that $\lambda \geq 2^{-(\gamma - 1)p\alpha / d}$. 
		\begin{align*}
			\frac{n_{\mathrm{P}} 2^{-p\alpha/d} +\lambda n_{\mathrm{Q}} 2^{-p\gamma\alpha/d} }{n_{\mathrm{P}} 
				+ \lambda n_{\mathrm{Q}} } \cdot \left(\frac{ 2^p \cdot \sqrt{n_{\mathrm{P}}}}{\varepsilon (n_{\mathrm{P}} + \lambda n_{\mathrm{Q}})
			}\right)^{-1} \geq 2^{-p\alpha/d} \cdot 2^{-p}\cdot\varepsilon\cdot \frac{n_{\mathrm{P}} + \lambda^2 n_{\mathrm{Q}}}{\sqrt{n_{\mathrm{P}} \log (n_{\mathrm{P}} +  n_{\mathrm{Q}})}}.
		\end{align*}
		By bringing in the value of $\delta$ and applying Lemma \ref{lem:boudnofnlamdasq} (the first conclusion) , there holds
		\begin{align*}
			2^{-p\alpha/d} \cdot 2^{-p}\cdot\varepsilon\cdot \frac{n_{\mathrm{P}} + \lambda^2 n_{\mathrm{Q}}}{\sqrt{n_{\mathrm{P}} \log (n_{\mathrm{P}} +  n_{\mathrm{Q}})}} \gtrsim & C_p^{\frac{\alpha + d}{d}}\cdot \delta^{-\frac{\alpha + d}{d}} \cdot n_{\mathrm{P}}^{-\frac{1}{2}} \cdot \varepsilon^{-1}\cdot \sqrt{\log (n_{\mathrm{P}} +  n_{\mathrm{Q}})}\\
			\gtrsim  &  C_p^{\frac{\alpha + d}{d}}\cdot \left(\frac{n_{\mathrm{P}}\varepsilon^2}{\log (n_{\mathrm{P}} +  n_{\mathrm{Q}} )}\right)^{\frac{1}{2}}\cdot n_{\mathrm{P}}^{-\frac{1}{2}} \cdot \varepsilon^{-1} \cdot \sqrt{\log (n_{\mathrm{P}} +  n_{\mathrm{Q}})}
			\\
			= &  C_p^{\frac{\alpha + d}{d}}. 
		\end{align*}
		Thus, for sufficiently large $C_p$, we have
		\begin{align}\label{equ:approximationlargerthanprivatized}
			\frac{n_{\mathrm{P}} 2^{-p\alpha/d} +\lambda n_{\mathrm{Q}} 2^{-p\gamma\alpha/d} }{n_{\mathrm{P}} 
				+ \lambda n_{\mathrm{Q}} } \geq 2  \frac{ 2^p \cdot \sqrt{n_{\mathrm{P}}}}{\varepsilon (n_{\mathrm{P}} + \lambda n_{\mathrm{Q}})
			}. 
		\end{align}
		Analogously, we compare the signal strength with the sample error, which is 
		\begin{align*}
			& \frac{n_{\mathrm{P}} 2^{-p\alpha/d} +\lambda n_{\mathrm{Q}} 2^{-p\gamma\alpha/d} }{n_{\mathrm{P}} 
				+ \lambda n_{\mathrm{Q}} } \cdot \left(\frac{2^{p} (n_{\mathrm{P}} + \lambda^2 n_{\mathrm{Q}}) \cdot \log (n_{\mathrm{P}} +  n_{\mathrm{Q}}) }{ (n_{\mathrm{P}} + \lambda n_{\mathrm{Q}})^2 }\right)^{-\frac{1}{2}} 
			\\=
			& 2^{-p\alpha / d} \cdot \delta^{\frac{1}{2}} \cdot \sqrt{n_{\mathrm{P}} + \lambda^2 n_{\mathrm{Q}}}\cdot \log^{-\frac{1}{2}} (n_{\mathrm{P}} +  n_{\mathrm{Q}}).
		\end{align*}
		By bringing in the value of $\delta$ and applying Lemma \ref{lem:boudnofnlamdasq} (the first conclusion) , there holds
		\begin{align*}
			2^{-p\alpha / d} \cdot \delta^{\frac{1}{2}} \cdot \sqrt{n_{\mathrm{P}} + \lambda^2 n_{\mathrm{Q}}}\cdot \log^{-\frac{1}{2}} (n_{\mathrm{P}} +  n_{\mathrm{Q}}) \gtrsim \delta^{-\frac{1}{2}} \cdot \varepsilon^{-1} \gtrsim 1. 
		\end{align*}
		Thus, for sufficiently large $n_{\mathrm{P}}$ and $n_{\mathrm{Q}}$, we have
		\begin{align}\label{equ:approximationlargerthansample}
			\frac{n_{\mathrm{P}} 2^{-p\alpha/d} +\lambda n_{\mathrm{Q}} 2^{-p\gamma\alpha/d} }{n_{\mathrm{P}} 
				+ \lambda n_{\mathrm{Q}} } \geq 2  \sqrt{\frac{2^{p} (n_{\mathrm{P}} + \lambda^2 n_{\mathrm{Q}}) \cdot \log (n_{\mathrm{P}} +  n_{\mathrm{Q}}) }{ (n_{\mathrm{P}} + \lambda n_{\mathrm{Q}})^2 }}  . 
		\end{align}
		Combined with \eqref{equ:approximationlargerthanprivatized}
		and \eqref{equ:approximationlargerthansample}, 
		\eqref{equ:signofetadpwithcomega} yields
		$ \tilde{\eta}^{\pi}_{\mathrm{P}, \mathrm{Q}}(x) -\frac{1}{2} \geq 0$,
		which indicates that $\tilde{f}^{\pi}_{\mathrm{P}, \mathrm{Q}}(x) = f^*_{\mathrm{P}}(x)$ for all $x\in \Omega_n$. 
		
		For $\Omega_{n}^c$, by Assumption \ref{asp:margin}, we have 
		\begin{align*}
			\int_{\Omega_n^c} \left|\eta^*_{\mathrm{P}}(x) -\frac{1}{2}\right| \eins\left\{\tilde{f}^{\pi}_{\mathrm{P}, \mathrm{Q}}(x) \neq f^*_{\mathrm{P}}(x)\right\} d\mathrm{P}_X(x) \leq c_{\Omega} \delta^{\alpha /d }\int_{\Omega_n^c}  d\mathrm{P}_X(x) \leq c_{\Omega}^{1+\beta} C_{\beta} \delta^{(\beta+1)\alpha /d}. 
		\end{align*}
		Together, there holds
		\begin{align*}
			\mathcal{R}_{\mathrm{P}} (\tilde{f}^{\pi}_{\mathrm{P}, \mathrm{Q}}) - \mathcal{R}_{\mathrm{P}}^* = & \int_{\mathcal{X}} \left|\eta^*_{\mathrm{P}}(x) -\frac{1}{2}\right| \eins\left\{\tilde{f}^{\pi}_{\mathrm{P}, \mathrm{Q}}(x) \neq f^*_{\mathrm{P}}(x)\right\} 
			d\mathrm{P}_X(x)
			\\ 
			=& \int_{\Omega_n^c} \left|\eta^*_{\mathrm{P}}(x) -\frac{1}{2}\right| d\mathrm{P}_X(x) \lesssim \delta^{(\beta+1)\alpha /d}. 
		\end{align*}
	\end{proof}

	\subsection{Proofs of Lemmas in Section \ref{sec:proofoftheoremutility}}\label{sec:proofoferroranalysis}
	
	To prove lemmas in Section \ref{sec:proofoftheoremutility}, we first present several technical results. 
	Their proof will be found in Section \ref{sec:proofoferroranalysis}


	\begin{lemma}\label{lem:laplacebound}
		Suppose $\xi_{i}, i = 1,\cdots, n$ are independent sub-exponential random variables with parameters $(a, b)$ \cite[Definition 2.9]{wainwright2019high}. Then there holds
		\begin{align*}
			\mathbb{P}\left[\left|\frac{1}{n}\sum_{i=1}^n \xi_{i} - \mathbb{E} \frac{1}{n}\sum_{i=1}^n \xi_{i} \right| \geq t  \right] \leq 2 e^{ - \frac{n t^2 }{2 a^2} }. 
		\end{align*}
		for any $t>0$.  Moreover, a standard Laplace random variable is sub-exponential with parameters $(\sqrt{2},1)$.
		A union bound argument yields that 
		\begin{align*}
			\mathbb{P}\left[\left|\frac{1}{n}\sum_{i=1}^n \xi_{i}^j - \mathbb{E} \frac{1}{n}\sum_{i=1}^n \xi_{i}^j \right| \geq t  ,\;\;\; \text{for  } j =1 ,\cdots ,\; 2^p\right] \leq 2 e^{ p \log 2  - \frac{n t^2 }{2 a^2} }. 
		\end{align*}
	\end{lemma}

	\begin{lemma}\label{lem:cellbound}
		Let $\pi$ be a partition generated from Algorithm \ref{alg:partition} with depth $p$. Suppose Assumption \ref{asp:boundedmarginal} holds. 
		Suppose $D^{\mathrm{P}}=\{(X_1^{\mathrm{P}},Y_1^{\mathrm{P}}),\cdots, (X_{n_{\mathrm{P}}}^{\mathrm{P}},Y_{n_{\mathrm{P}}}^{\mathrm{P}})\}$ and $D^{\mathrm{Q}}=\{(X_1^{\mathrm{Q}},Y_1^{\mathrm{Q}}),\cdots, (X_{n_{\mathrm{Q}}}^{\mathrm{Q}},Y_{n_{\mathrm{Q}}}^{\mathrm{Q}})\}$ are drawn i.i.d. from $\mathrm{P}$ and $\mathrm{Q}$, respectively. 
		Let $0<\lambda \lesssim (n_{\mathrm{P}} + n_{\mathrm{Q}}) / \log (n_{\mathrm{P}} + n_{\mathrm{Q}} )$.
		Then for all $A_{(p)}^{j}$ in $\pi_p$, there holds
		\begin{align*}
			&\left|\frac{1}{n_{\mathrm{P}} + \lambda n_{\mathrm{Q}}}\left(\sum_{i=1}^{n_{\mathrm{P}}}  \mathbf{1}\{X_i^{\mathrm{P}} \in A_{(p)}^{j}\} + \lambda \sum_{i=1}^{n_{\mathrm{Q}}}  \mathbf{1}\{X_i^{\mathrm{Q}} \in A_{(p)}^{j}\}\right) - \int_{A_{(p)}^{j}}d\mathrm{P}_X(x')  \right| \\
			\lesssim  & \sqrt{\frac{ (n_{\mathrm{P}} + \lambda^2 n_{\mathrm{Q}}) \log (n_{\mathrm{P}} + n_{\mathrm{Q}})}{2^p(n_{\mathrm{P}} + \lambda n_{\mathrm{Q}})^2} }
		\end{align*}
		with probability $\mathrm{P}^{n_{\mathrm{P}}} \otimes \mathrm{Q}^{n_{\mathrm{Q}}}$ at least $1-1/(n_{\mathrm{P}} + n_{\mathrm{Q}})^2$. 
	\end{lemma}
	
	\begin{lemma}\label{lem:cellboundY}
		Let $\pi$ be a partition generated from Algorithm \ref{alg:partition} with depth $p$. Suppose Assumption \ref{asp:boundedmarginal} holds. 
		Suppose $D^{\mathrm{P}}=\{(X_1^{\mathrm{P}},Y_1^{\mathrm{P}}),\cdots, (X_{n_{\mathrm{P}}}^{\mathrm{P}},Y_{n_{\mathrm{P}}}^{\mathrm{P}})\}$ and $D^{\mathrm{Q}}=\{(X_1^{\mathrm{Q}},Y_1^{\mathrm{Q}}),\cdots, (X_{n_{\mathrm{Q}}}^{\mathrm{Q}},Y_{n_{\mathrm{Q}}}^{\mathrm{Q}})\}$ are drawn i.i.d. from $\mathrm{P}$ and $\mathrm{Q}$, respectively. 
		Let $0<\lambda \lesssim (n_{\mathrm{P}} + n_{\mathrm{Q}}) / \log (n_{\mathrm{P}} + n_{\mathrm{Q}} )$.
		Then for all $A_{(p)}^{j}$ in $\pi_p$, there holds
		\begin{align*}
			&\frac{1}{n_{\mathrm{P}} + \lambda n_{\mathrm{Q}}} \Bigg|\left(\sum_{i=1}^{n_{\mathrm{P}}}  Y_{i}^{\mathrm{P}}\mathbf{1}\{X_i^{\mathrm{P}} \in A_{(p)}^{j}\} + \lambda\sum_{i=1}^{n_{\mathrm{Q}}}  Y_{i}^{\mathrm{Q}}\mathbf{1}\{X_i^{\mathrm{Q}} \in A_{(p)}^{j}\}\right) - \\ & \left({n_{\mathrm{P}}\int_{A_{(p)}^{j}}\eta^*_{\mathrm{P}}(x')d\mathrm{P}_X(x') + \lambda \; n_{\mathrm{Q}}\int_{A_{(p)}^{j}}\eta^*_{\mathrm{Q}}(x')d\mathrm{Q}_X(x')}  \right)\Bigg| 
			\lesssim    \sqrt{\frac{ (n_{\mathrm{P}} + \lambda^2 n_{\mathrm{Q}}) \log (n_{\mathrm{P}} + n_{\mathrm{Q}})}{2^p(n_{\mathrm{P}} + \lambda n_{\mathrm{Q}})^2} }
		\end{align*}
		with probability $\mathrm{P}^{n_{\mathrm{P}}} \otimes \mathrm{Q}^{n_{\mathrm{Q}}}$ at least $1-1/(n_{\mathrm{P}} + n_{\mathrm{Q}})^2$. 
	\end{lemma}
	
	\begin{lemma}\label{lem::treeproperty}
		Let $\pi_i :=\left\{A_{(i)}^j,  j \in \{ 1,\cdots, 2^p \}\right\}$ be the partition generated by Algorithm \ref{alg:partition} with depth $i$. Then for any $A_{(i)}^j$, there holds 
		\begin{align*}
			2^{-1}\sqrt{d}\cdot 2^{-i/d}\leq \mathrm{diam}(A_{(i)}^j)\leq 2\sqrt{d} \cdot 2^{-i/d}.
		\end{align*}
	\end{lemma}

	\begin{proof}\textbf{of Lemma \ref{lem:boundingofprivatisederror}}\;\;
		We intend to bound
		\begin{align*}
			\| \tilde{\eta}^{\pi}_{\mathrm{P}, \mathrm{Q}} - \widehat{\eta}^{\pi}_{\mathrm{P}, \mathrm{Q}} \|_{\infty}
			=  & \max_{j\in \mathcal{I}}
			\left| 
			\frac{
				\sum_{i=1}^{n_{\mathrm{P}}} \tilde{V}^{\mathrm{P} j }_i + \lambda \sum_{i=1}^{n_{\mathrm{Q}}}{V}^{\mathrm{Q} j}_i
			}
			{ 
				\sum_{i=1}^{n_{\mathrm{P}}} \tilde{U}^{\mathrm{P} j }_i +  \lambda\sum_{i=1}^{n_{\mathrm{Q}}}{U}^{\mathrm{Q} j}_i
			} 
			- 
			\frac{
				\sum_{i=1}^{n_{\mathrm{P}}} {V}^{\mathrm{P} j }_i + \lambda \sum_{i=1}^{n_{\mathrm{Q}}}{V}^{\mathrm{Q} j}_i
			}
			{ 
				\sum_{i=1}^{n_{\mathrm{P}}} {U}^{\mathrm{P} j }_i +  \lambda\sum_{i=1}^{n_{\mathrm{Q}}}{U}^{\mathrm{Q} j}_i
			}\right|
		\end{align*}
		For each $j\in \mathcal{I}$ and any $x\in A^j_{(p)}$, the error can be decomposed as 
		\begin{align*}
			& \left| \frac{
				\sum_{i=1}^{n_{\mathrm{P}}} \tilde{V}^{\mathrm{P} j }_i + \lambda \sum_{i=1}^{n_{\mathrm{Q}}}{V}^{\mathrm{Q} j}_i
			}
			{ 
				\sum_{i=1}^{n_{\mathrm{P}}} \tilde{U}^{\mathrm{P} j }_i +  \lambda\sum_{i=1}^{n_{\mathrm{Q}}}{U}^{\mathrm{Q} j}_i
			} 
			- 
			\frac{
				\sum_{i=1}^{n_{\mathrm{P}}} \tilde{V}^{\mathrm{P} j }_i + \lambda \sum_{i=1}^{n_{\mathrm{Q}}}{V}^{\mathrm{Q} j}_i
			}
			{ 
				\sum_{i=1}^{n_{\mathrm{P}}} {U}^{\mathrm{P} j }_i +  \lambda\sum_{i=1}^{n_{\mathrm{Q}}}{U}^{\mathrm{Q} j}_i
			}\right| 
			\allowbreak
			\\
			+ & \left| \frac{
				\sum_{i=1}^{n_{\mathrm{P}}} \tilde{V}^{\mathrm{P} j }_i + \lambda \sum_{i=1}^{n_{\mathrm{Q}}}{V}^{\mathrm{Q} j}_i
			}
			{ 
				\sum_{i=1}^{n_{\mathrm{P}}} {U}^{\mathrm{P} j }_i +  \lambda\sum_{i=1}^{n_{\mathrm{Q}}}{U}^{\mathrm{Q} j}_i
			} 
			- 
			\frac{
				\sum_{i=1}^{n_{\mathrm{P}}} {V}^{\mathrm{P} j }_i + \lambda \sum_{i=1}^{n_{\mathrm{Q}}}{V}^{\mathrm{Q} j}_i
			}
			{ 
				\sum_{i=1}^{n_{\mathrm{P}}} {U}^{\mathrm{P} j }_i +  \lambda\sum_{i=1}^{n_{\mathrm{Q}}}{U}^{\mathrm{Q} j}_i
			}\right| \\ 
			\leq & 
			\underbrace{\left|\sum_{i=1}^{n_{\mathrm{P}}} \tilde{V}^{\mathrm{P} j }_i + \lambda \sum_{i=1}^{n_{\mathrm{Q}}}{V}^{\mathrm{Q} j}_i\right| \left| \frac{
					\frac{4}{\varepsilon}\sum_{i=1}^{n_{\mathrm{P}}} \zeta_i^j
				}
				{ 
					\left(
					\sum_{i=1}^{n_{\mathrm{P}}} \tilde{U}^{\mathrm{P} j }_i +  \lambda\sum_{i=1}^{n_{\mathrm{Q}}}{U}^{\mathrm{Q} j}_i
					\right)
					\left(
					\sum_{i=1}^{n_{\mathrm{P}}} {U}^{\mathrm{P} j }_i +  \lambda\sum_{i=1}^{n_{\mathrm{Q}}}{U}^{\mathrm{Q} j}_i
					\right)
				} 
				\right|}_{(I)}\\
			+ & 
			\underbrace{\left|\frac{\frac{4}{\varepsilon}
					\sum_{i=1}^{n_{\mathrm{P}}} \xi_i^j
				}
				{ 
					\sum_{i=1}^{n_{\mathrm{P}}} {U}^{\mathrm{P} j }_i +  \lambda\sum_{i=1}^{n_{\mathrm{Q}}}{U}^{\mathrm{Q} j}_i
				} 
				\right|
			}_{(II)}.
		\end{align*}

		For the first term in $(I)$, we can apply Lemma \ref{lem:laplacebound} and \ref{lem:cellboundY} to get 
		\begin{align*}
			&\left|\sum_{i=1}^{n_{\mathrm{P}}} \tilde{V}^{\mathrm{P} j }_i + \lambda \sum_{i=1}^{n_{\mathrm{Q}}}{V}^{\mathrm{Q} j}_i\right|  \\
			\leq & \left|\sum_{i=1}^{n_{\mathrm{P}}} {V}^{\mathrm{P} j }_i + \lambda \sum_{i=1}^{n_{\mathrm{Q}}}{V}^{\mathrm{Q} j}_i\right| + \left|\frac{4}{\varepsilon}\sum_{i=1}^{n_{\mathrm{P}}} \xi_i^j\right|\\
			\lesssim & (n_{\mathrm{P}} + \lambda n_{\mathrm{Q}})\left( \sqrt{\frac{ (n_{\mathrm{P}} + \lambda^2 n_{\mathrm{Q}}) \log (n_{\mathrm{P}} + n_{\mathrm{Q}})}{2^p (n_{\mathrm{P}} + \lambda n_{\mathrm{Q}})^2 } } + 2^{-p} +\frac{\sqrt{n_{\mathrm{P}}\varepsilon^{-2} \log (n_{\mathrm{P}} +  n_{\mathrm{Q}})}}{n_{\mathrm{P}} + \lambda n_{\mathrm{Q}}}\right) 
		\end{align*}
		with probability at least $ 1- 2/(n_{\mathrm{P}} + n_{\mathrm{Q}})^2$, where we used the fact that 
		\begin{align*}
			\int_{A^j_{(p)}}d\mathrm{P}_X(x) \asymp 2^{-p} , \;\;\;\; \int_{A^j_{(p)}}\eta_{\mathrm{P}}^*(x) d\mathrm{P}_X(x) \asymp 2^{-p} , \;\;\;\; \int_{A^j_{(p)}}\eta_{\mathrm{Q}}^*(x) d\mathrm{Q}_X(x) \asymp 2^{-p} 
		\end{align*}
		Using Lemma \ref{lem:boudnofnlamdasq} (the second and the third conclusions), we know that the second term is the dominating term, indicating that
		\begin{align*}
			\left|\sum_{i=1}^{n_{\mathrm{P}}} \tilde{V}^{\mathrm{P} j }_i + \lambda \sum_{i=1}^{n_{\mathrm{Q}}}{V}^{\mathrm{Q} j}_i\right| \lesssim  
			(n_{\mathrm{P}} + \lambda n_{\mathrm{Q}})\cdot 2^{-p}.
		\end{align*} 
		Analogously, by Lemma \ref{lem:cellbound}, there holds 
		\begin{align}\nonumber
			\left|\sum_{i=1}^{n_{\mathrm{P}}} {U}^{\mathrm{P} j }_i + \lambda \sum_{i=1}^{n_{\mathrm{Q}}}{U}^{\mathrm{Q} j}_i\right| 
			\gtrsim & (n_{\mathrm{P}} + \lambda n_{\mathrm{Q}})\left(  2^{-p} - \sqrt{\frac{ (n_{\mathrm{P}} + \lambda^2  n_{\mathrm{Q}}) \log (n_{\mathrm{P}} + n_{\mathrm{Q}})}{2^p(n_{\mathrm{P}} + \lambda n_{\mathrm{Q}})^2} }  \right)
			\\ \asymp & 
			(n_{\mathrm{P}} + \lambda n_{\mathrm{Q}})\cdot 2^{-p}\label{equ:lowerboundofUinprivatizederror}
		\end{align}
		as well as
		\begin{align*}
			\left|\sum_{i=1}^{n_{\mathrm{P}}} \tilde{U}^{\mathrm{P} j }_i + \lambda \sum_{i=1}^{n_{\mathrm{Q}}}{U}^{\mathrm{Q} j}_i\right| \geq  & \left|\sum_{i=1}^{n_{\mathrm{P}}} {U}^{\mathrm{P} j }_i + \lambda \sum_{i=1}^{n_{\mathrm{Q}}}{U}^{\mathrm{Q} j}_i\right| - \left|\frac{4}{\varepsilon}\sum_{i=1}^{n_{\mathrm{P}}} \zeta_i^j\right|\\
			\gtrsim & (n_{\mathrm{P}} + \lambda n_{\mathrm{Q}})\left( 2^{-p}  -\frac{\sqrt{n_{\mathrm{P}}\varepsilon^{-2} \log (n_{\mathrm{P}} +  n_{\mathrm{Q}})}}{n_{\mathrm{P}} + \lambda n_{\mathrm{Q}}}\right)\\
			\asymp &
			(n_{\mathrm{P}} + \lambda n_{\mathrm{Q}})\cdot 2^{-p}. 
		\end{align*}
		The last step is by Lemma \ref{lem:boudnofnlamdasq} (the second and the third conclusions), knowing that $2^{-p}$ is the dominating term. 
		Combining these three pieces, there holds
		\begin{align}\label{equ:privatized1}
			(I) \lesssim \sqrt{n_{\mathrm{P}}\varepsilon^{-2} \log (n_{\mathrm{P}} +  n_{\mathrm{Q}})}\cdot \left((n_{\mathrm{P}} + \lambda n_{\mathrm{Q}})\cdot 2^{-p}\right)^{-1} \asymp \frac{ 2^p \cdot \sqrt{n_{\mathrm{P}} \log (n_{\mathrm{P}} +  n_{\mathrm{Q}})}}{\varepsilon (n_{\mathrm{P}} + \lambda n_{\mathrm{Q}})
			}.
		\end{align}
		For $(II)$, \eqref{equ:lowerboundofUinprivatizederror} and Lemma \ref{lem:laplacebound} together yields 
		\begin{align}\label{equ:privatized2}
			(II) \lesssim\frac{ 2^p \cdot \sqrt{n_{\mathrm{P}}\log (n_{\mathrm{P}} +  n_{\mathrm{Q}})}}{\varepsilon (n_{\mathrm{P}} + \lambda n_{\mathrm{Q}})
			}. 
		\end{align}
		\eqref{equ:privatized1} and \eqref{equ:privatized2} together lead to the desired conclusion. 
	\end{proof}
	
	\begin{proof}\textbf{of Lemma \ref{lem:boundingofsampleerror}}\;\;
		Let $\eta_{\mathrm{P}, \mathrm{Q}}^*(x) = \left(n_{\mathrm{P}}\eta_{\mathrm{P}}^*(x) +  \lambda n_{\mathrm{Q}}\eta_{\mathrm{Q}}^*(x)\right)/(n_{\mathrm{P}} + \lambda n_{\mathrm{Q}})$. We intend to bound
		\begin{align*}
			\|  \widehat{\eta}^{\pi}_{\mathrm{P},\mathrm{Q}} - \overline{\eta}^{\pi}_{\mathrm{P},\mathrm{Q}} \|_{\infty}
			=  & \max_{j\in \mathcal{I}}\left|
			\frac{
				\sum_{i=1}^{n_{\mathrm{P}}} {V}^{\mathrm{P} j }_i + \lambda \sum_{i=1}^{n_{\mathrm{Q}}}{V}^{\mathrm{Q} j}_i
			}
			{ 
				\sum_{i=1}^{n_{\mathrm{P}}} {U}^{\mathrm{P} j }_i +  \lambda\sum_{i=1}^{n_{\mathrm{Q}}}{U}^{\mathrm{Q} j}_i
			} 
			-
			\frac{
				(n_{\mathrm{P}}+ \lambda n_{\mathrm{Q}}) \int_{A^j_{(p)}} \eta^*_{\mathrm{P}, \mathrm{Q}}(x')d\mathrm{P}_X(x') 
			}
			{ 
				(n_{\mathrm{P}}+ \lambda n_{\mathrm{Q}}) \int_{A^j_{(p)}} d\mathrm{P}_X(x') 
			}
			\right|
		\end{align*}
		For each $j\in \mathcal{I}$ and any $x\in A^j_{(p)}$, the error can be decomposed as  
		\begin{align*}
			& \left|
			\frac{
				\sum_{i=1}^{n_{\mathrm{P}}} {V}^{\mathrm{P} j }_i + \lambda \sum_{i=1}^{n_{\mathrm{Q}}}{V}^{\mathrm{Q} j}_i
			}
			{ 
				\sum_{i=1}^{n_{\mathrm{P}}} {U}^{\mathrm{P} j }_i +  \lambda\sum_{i=1}^{n_{\mathrm{Q}}}{U}^{\mathrm{Q} j}_i
			} 
			-
			\frac{
				(n_{\mathrm{P}}+ \lambda n_{\mathrm{Q}}) \int_{A^j_{(p)}} \eta^*_{\mathrm{P}, \mathrm{Q}}(x')d\mathrm{P}_X(x') 
			}
			{ 
				(n_{\mathrm{P}}+ \lambda n_{\mathrm{Q}}) \int_{A^j_{(p)}} d\mathrm{P}_X(x') 
			}
			\right| \\
			\leq &\underbrace{\left| \int_{A^j_{(p)}} \eta^*_{\mathrm{P}, \mathrm{Q}}(x')d\mathrm{P}_X(x') \right| \cdot \left|\frac{ \sum_{i=1}^{n_{\mathrm{P}}} {U}^{\mathrm{P} j }_i +  \lambda\sum_{i=1}^{n_{\mathrm{Q}}}{U}^{\mathrm{Q} j}_i - (n_{\mathrm{P}}+ \lambda n_{\mathrm{Q}}) \int_{A^j_{(p)}} d\mathrm{P}_X(x') }{\left(\sum_{i=1}^{n_{\mathrm{P}}} {U}^{\mathrm{P} j }_i +  \lambda\sum_{i=1}^{n_{\mathrm{Q}}}{U}^{\mathrm{Q} j}_i\right) \int_{A^j_{(p)}} d\mathrm{P}_X(x') }\right|}_{(I)}\\
			+  & 
			\underbrace{\left|\frac{
					\sum_{i=1}^{n_{\mathrm{P}}} {V}^{\mathrm{P} j }_i + \lambda \sum_{i=1}^{n_{\mathrm{Q}}}{V}^{\mathrm{Q} j}_i - (n_{\mathrm{P}}+ \lambda n_{\mathrm{Q}}) \int_{A^j_{(p)}} \eta^*_{\mathrm{P}, \mathrm{Q}}(x')d\mathrm{P}_X(x')  }{\sum_{i=1}^{n_{\mathrm{P}}} {U}^{\mathrm{P} j }_i +  \lambda\sum_{i=1}^{n_{\mathrm{Q}}}{U}^{\mathrm{Q} j}_i}\right|}_{(II)}
		\end{align*}
		We bound two terms separately. For the first term in $(I)$, Assumption \ref{asp:boundedmarginal} yield
		\begin{align*}
			\int_{A^j_{(p)}} \eta^*_{\mathrm{P}, \mathrm{Q}}(x')d\mathrm{P}_X(x') \lesssim 2^{-p} 
		\end{align*}
		For the second term, \eqref{equ:lowerboundofUinprivatizederror} and 
		Lemma \ref{lem:cellbound} yields 
		\begin{align*}
			& \left|\frac{ \sum_{i=1}^{n_{\mathrm{P}}} {U}^{\mathrm{P} j }_i +  \lambda\sum_{i=1}^{n_{\mathrm{Q}}}{U}^{\mathrm{Q} j}_i - (n_{\mathrm{P}}+ \lambda n_{\mathrm{Q}}) \int_{A^j_{(p)}} d\mathrm{P}_X(x') }{\left(\sum_{i=1}^{n_{\mathrm{P}}} {U}^{\mathrm{P} j }_i +  \lambda\sum_{i=1}^{n_{\mathrm{Q}}}{U}^{\mathrm{Q} j}_i\right) \int_{A^j_{(p)}} d\mathrm{P}_X(x') }\right| 
			\\ \lesssim & (n_{\mathrm{P}} + \lambda n_{\mathrm{Q}}) \cdot \sqrt{\frac{ (n_{\mathrm{P}} + \lambda^2 n_{\mathrm{Q}}) \log (n_{\mathrm{P}} + n_{\mathrm{Q}})}{2^p(n_{\mathrm{P}} + \lambda n_{\mathrm{Q}})^2 } } \cdot \left( (n_{\mathrm{P}} + \lambda n_{\mathrm{Q}}) \cdot  2^{-2p} \right)^{-1} \\ 
			\asymp  & 2^{p} \cdot \sqrt{\frac{ 2^p (n_{\mathrm{P}} + \lambda^2 n_{\mathrm{Q}})\log (n_{\mathrm{P}} + n_{\mathrm{Q}})}{(n_{\mathrm{P}} + \lambda n_{\mathrm{Q}})^2} }
		\end{align*}
		with probability $1-1/( n_{\mathrm{P}} + n_{\mathrm{Q}})^2$. 
		Together, these two pieces yield
		\begin{align}\label{equ:sampleerrorpart1}
			(I) \lesssim  \sqrt{\frac{ 2^p (n_{\mathrm{P}} + \lambda^2 n_{\mathrm{Q}}) \log (n_{\mathrm{P}} + n_{\mathrm{Q}})}{(n_{\mathrm{P}} + \lambda n_{\mathrm{Q}})^2} }. 
		\end{align}
		Similarly for $(II)$, we have
		\begin{align}\label{equ:sampleerrorpart2}
			(II) \lesssim  \sqrt{\frac{ (n_{\mathrm{P}} + \lambda^2 n_{\mathrm{Q}}) \log (n_{\mathrm{P}} + n_{\mathrm{Q}})}{2^p(n_{\mathrm{P}} + \lambda n_{\mathrm{Q}})^2} } \cdot   2^{p}\asymp \sqrt{\frac{ 2^p (n_{\mathrm{P}} + \lambda^2 n_{\mathrm{Q}}) \cdot \log (n_{\mathrm{P}} + n_{\mathrm{Q}})}{(n_{\mathrm{P}} + \lambda n_{\mathrm{Q}})^2} }. 
		\end{align}
		\eqref{equ:sampleerrorpart1} and  \eqref{equ:sampleerrorpart2} together yield the desired conclusion. 
	\end{proof}

	\begin{proof}\textbf{of Lemma \ref{lem:boundingofapproximationerror}}\;\;
		Without loss of generality, consider $f^*_{\mathrm{P}}(x) = 1$ for any $x\in\Delta_n$. 
		By Assumption \ref{asp:rse}, we have 
		\begin{align}\label{equ:lowerboundofoverlineeta}
			\overline{\eta}^{\pi}_{\mathrm{P},\mathrm{Q}}(x) \geq  
			\frac{
				n_{\mathrm{P}} \int_{A^j_{(p)}} \eta^*_{\mathrm{P}}(x')  d\mathrm{P}_X(x') + \lambda n_{\mathrm{Q}} \int_{A^j_{(p)}} \left(\frac{1}{2} + C_{\gamma} \left(\eta^*_{\mathrm{P}}(x')-\frac{1}{2}\right)^{\gamma} \right)d\mathrm{Q}_X(x')
			}
			{ 
				n_{\mathrm{P}} \int_{A^j_{(p)}} d\mathrm{P}_X(x') + \lambda n_{\mathrm{Q}} \int_{A^j_{(p)}} d\mathrm{Q}_X(x')
			}
		\end{align}
		If Assumption \ref{asp:alphaholder} holds, for any $x' \in A^{j}_{(p)}$, we have 
		\begin{align*}
			\eta_{\mathrm{P}}^*(x') -\frac{1}{2} \geq &\eta_{\mathrm{P}}^*(x) -\frac{1}{2} - |\eta_{\mathrm{P}}^*(x) - \eta_{\mathrm{P}}^*(x')| \\\geq &\eta_{\mathrm{P}}^*(x) -\frac{1}{2} - c_L \|x-x'\|^{\alpha}
			\geq  \eta_{\mathrm{P}}^*(x) -\frac{1}{2} - 2 c_L \sqrt{d} 2^{-p\alpha/d} \geq (c_{\Delta} - 2 c_L \sqrt{d})2^{-p\alpha/d}
		\end{align*}
		where the second last and the last inequalities come from Lemma \ref{lem::treeproperty} and the definition of $\Delta_n$, respectively. For sufficiently large $c_{\Delta}$, we have $\eta_{\mathrm{P}}^*(x') \geq 1/2 + 2^{- p\alpha/d}$ for all $x'\in  A^{j}_{(p)}$, and thus \eqref{equ:lowerboundofoverlineeta} yields 
		\begin{align*}
			\overline{\eta}^{\pi}_{\mathrm{P},\mathrm{Q}}(x) \geq & \frac{1}{2} + \frac{
				n_{\mathrm{P}} \int_{A^j_{(p)}} 2^{-p\alpha/d}  d\mathrm{P}_X(x') + \lambda n_{\mathrm{Q}} \int_{A^j_{(p)}}  C_{\gamma} 2^{-p\gamma\alpha/d} d\mathrm{Q}_X(x')
			}
			{ 
				n_{\mathrm{P}} \int_{A^j_{(p)}} d\mathrm{P}_X(x') + \lambda n_{\mathrm{Q}} \int_{A^j_{(p)}} d\mathrm{Q}_X(x')
			}\\
			\geq  &   \frac{1}{2} + \frac{\underline{c}^2\cdot \left( 1\vee C_{\gamma}\right)}{\overline{c}^2} \frac{n_{\mathrm{P}} 2^{-p\alpha/d} +\lambda n_{\mathrm{Q}} 2^{-p\gamma\alpha/d} }{n_{\mathrm{P}} 
				+ \lambda n_{\mathrm{Q}} }
		\end{align*}
		which yields the desired result. 
		Note that the factor $\underline{c}/\overline{c}$ comes from the lower bound of ratio $\int_{A^j_{(p)}} d\mathrm{Q}_X(x') / \int_{A^j_{(p)}} d\mathrm{P}_X(x')$. 
	\end{proof}

	\begin{proof}\textbf{of Lemma \ref{lem:boudnofnlamdasq}}\;\;
		We prove the statement by discussing four cases depending on the value of $n_{\mathrm{P}}$, $n_{\mathrm{Q}}$, and $\gamma$, $\alpha$, $d$.
		
		\textbf{Case ($\mathbf{i}$)}: 
		$ \frac{n_{\mathrm{P}} \varepsilon^2}{\log (n_{\mathrm{P}} + n_{\mathrm{Q}})} \geq  \left(\frac{n_{\mathrm{Q}}}{\log (n_{\mathrm{P}} + n_{\mathrm{Q}})}\right)^{\frac{2\alpha + 2d}{2 \gamma\alpha + d}}$
		and $\frac{(\gamma+1)\alpha}{d}\geq 1$. 
		We prove the three conclusions by turn. 
		The first conclusion follows from 
		\begin{align*}
			\frac{n_{\mathrm{P}} + \lambda^2 n_{\mathrm{Q}}}{\delta^{-( 2\alpha+2d)/ d}} \geq &    n_{\mathrm{P}}\cdot \delta^{ \frac{2 \alpha + 2 d}{d}} \\
			\gtrsim &  
			n_{\mathrm{P}}\cdot
			\left(\frac{n_{\mathrm{P}} \varepsilon^2}{\log(n_{\mathrm{P}} +  n_{\mathrm{Q}})}\right)^{- \frac{2\alpha + 2d}{2\alpha + 2 d}} \gtrsim \frac{\log (n_{\mathrm{P}} +  n_{\mathrm{Q}})}{\varepsilon^2}. 
		\end{align*}
		For the second conclusion, let $n_{\mathrm{Q}} = n_{\mathrm{P}}^a$ for some constant $a$. 
		Note that we have $\frac{2\gamma\alpha +d }{2\alpha + 2d} \geq a > 0$.
		By replacing all terms with equivalent order of $n_{\mathrm{P}}$, it suffices to prove
		\begin{align*}
			& n_{\mathrm{P}} +  n_{\mathrm{P}}^{ a -\frac{2(\gamma - 1)\alpha}{2\alpha + 2 d}} \cdot \varepsilon^{-\frac{2(\gamma - 1)\alpha}{\alpha +  d}}\\
			\lesssim  &\; n_{\mathrm{P}}^{\frac{4\alpha +3d}{2\alpha + 2 d}}\cdot \varepsilon^{-\frac{d}{\alpha + d}}  + n_{\mathrm{P}}^{1 + a - \frac{(\gamma - 1)\alpha + d}{2\alpha + 2d}}\cdot \varepsilon^{-\frac{(\gamma - 1)\alpha + d}{\alpha + d}}  + n_{\mathrm{P}}^{2 a - \frac{2(\gamma - 1)\alpha + d}{2\alpha + 2d} }\cdot \varepsilon^{-\frac{2(\gamma - 1)\alpha + d}{\alpha + d}}
		\end{align*}
		The left-hand side is dominated respectively by $n_{\mathrm{P}} \lesssim n_{\mathrm{P}}^{\frac{4\alpha +3d}{2\alpha + 2 d}} \cdot \varepsilon^{-\frac{d}{\alpha + d}} $ and 
		\begin{align*}
			n_{\mathrm{P}}^{a-\frac{2(\gamma - 1)\alpha}{2\alpha + 2 d}}  \cdot \varepsilon^{-\frac{2(\gamma - 1)\alpha}{\alpha +  d}} \lesssim n_{\mathrm{P}}^{a -\frac{2(\gamma - 1)\alpha}{2\alpha + 2 d} + \frac{1}{2}  + \frac{\gamma \alpha}{2\alpha + 2d}} \cdot \varepsilon^{-\frac{(\gamma - 1)\alpha + d}{\alpha + d}} = n_{\mathrm{P}}^{1 + a - \frac{(\gamma - 1)\alpha + d}{2\alpha + 2d}}\cdot \varepsilon^{-\frac{(\gamma - 1)\alpha + d}{\alpha + d}} .
		\end{align*}
		Here, the inequality is drawn by 
		\begin{align*}
			n_{\mathrm{P}}^{\frac{1}{2}} \cdot \varepsilon^{-\frac{d - \alpha }{\alpha + d}} \geq n_{\mathrm{P}}^{\frac{1}{2} (1 - \frac{d-\alpha}{\alpha + d})}, \;\;\;  n_{\mathrm{P}}^{\frac{\gamma \alpha }{2\alpha + 2d}} \cdot \varepsilon^{\frac{\gamma  \alpha }{\alpha + d}} \gtrsim 1.
		\end{align*}
		Note that we omit the log factor since the polynomial part of the left-hand side is strictly dominated by the right-hand side, yielding that the conclusion holds with sufficiently large $n_{\mathrm{P}}$ and $n_{\mathrm{Q}}$. 
		For the last statement, the same procedure yields that it suffices to show 
		\begin{align*}
			n_{\mathrm{P}}^{\frac{1}{2}}  \lesssim \left(n_{\mathrm{P}}^{\frac{2\alpha +d}{2\alpha + 2 d}}  + n_{\mathrm{P}}^{ a - \frac{(\gamma - 1)\alpha + d}{2\alpha + 2d}} \right)\cdot \varepsilon^{\frac{\alpha}{\alpha + d}},
		\end{align*}
		which holds since $n_{\mathrm{P}}^{\frac{\alpha}{2\alpha + 2d}} \cdot \varepsilon^{\frac{\alpha}{\alpha + d}} \gtrsim n_{\mathrm{P}}^{\frac{d}{\alpha + d}\cdot \frac{\alpha}{2\alpha + 2d}}$. 
		
		\textbf{Case ($\mathbf{ii}$)}: 
		$ \frac{n_{\mathrm{P}} \varepsilon^2}{\log (n_{\mathrm{P}} + n_{\mathrm{Q}})} \geq  \left(\frac{n_{\mathrm{Q}}}{\log (n_{\mathrm{P}} + n_{\mathrm{Q}})}\right)^{\frac{2\alpha + 2d}{2 \gamma\alpha + d}}$
		and $\frac{(\gamma+1)\alpha}{d}< 1$. 
		Compared to case ($i$), the value is $\lambda$ is multiplied by a factor $\delta^{(\frac{(\gamma + 1)\alpha}{d} - 1)\wedge 0} >1 $. 
		Thus, the first conclusion holds since the value of the left-hand side increases. 
		The third conclusion holds since the value of the left-hand side decreases. 
		For the second conclusion, by replacing all terms with equivalent order of $n_{\mathrm{P}}$, it suffices to prove
		\begin{align*}
			& n_{\mathrm{P}} +  n_{\mathrm{P}}^{ a -\frac{4\gamma \alpha - 2d}{2\alpha + 2 d}} \cdot \varepsilon^{-\frac{4\gamma \alpha - 2d}{\alpha +  d}} \\
			\lesssim & n_{\mathrm{P}}^{\frac{4\alpha +3d}{2\alpha + 2 d}} \cdot \varepsilon^{-\frac{d}{\alpha +  d}} + n_{\mathrm{P}}^{1 + a - \frac{2\gamma\alpha}{2\alpha + 2d}}\cdot \varepsilon^{-\frac{2\gamma \alpha }{\alpha +  d}} + n_{\mathrm{P}}^{2 a - \frac{4\gamma \alpha - d}{2\alpha + 2d} }\cdot \varepsilon^{-\frac{4\gamma \alpha - d}{\alpha +  d}}.
		\end{align*}
		For the left-hand side, the first term is dominated as $n_{\mathrm{P}} \lesssim n_{\mathrm{P}}^{\frac{4\alpha +3d}{2\alpha + 2 d}} \cdot \varepsilon^{-\frac{d}{\alpha + d}}$. 
		Since $\frac{ - 2\gamma \alpha + 2 d}{2\alpha + 2d} < 1$, we have 
		$
		n_{\mathrm{P}}\cdot \left(n_{\mathrm{P}}\cdot \varepsilon^2\right)^{\frac{  2\gamma \alpha - 2 d}{2\alpha + 2d}}\gtrsim n_{\mathrm{P}}
		$.
		Thus, the second term of left-hand side is controlled by 
		$n_{\mathrm{P}}^{ a -\frac{4\gamma \alpha - 2d}{2\alpha + 2 d}} \cdot \varepsilon^{-\frac{4\gamma \alpha - 2d}{\alpha +  d}}
		\lesssim 
		n_{\mathrm{P}}^{1 + a - \frac{2\gamma\alpha}{2\alpha + 2d}} \cdot \varepsilon^{-\frac{2\gamma \alpha }{\alpha +  d}}$.

		\textbf{Case ($\mathbf{iii}$)}: 
		$ \frac{n_{\mathrm{P}} \varepsilon^2}{\log (n_{\mathrm{P}} + n_{\mathrm{Q}})} <  \left(\frac{n_{\mathrm{Q}}}{\log (n_{\mathrm{P}} + n_{\mathrm{Q}})}\right)^{\frac{2\alpha + 2d}{2 \gamma\alpha + d}}$
		and $\frac{(\gamma+1)\alpha}{d}\geq  1$. 
		We now have $a > \frac{2\gamma\alpha +d }{2\alpha + 2d} $. 
		There holds
		\begin{align*}
			\frac{n_{\mathrm{P}} + \lambda^2 n_{\mathrm{Q}}}{\delta^{-(2 \alpha+ 2d)/ d}} \geq &    n_{\mathrm{Q}}\cdot \delta^{ \frac{4\gamma\alpha - 2 d+ 2 \alpha + 2 d}{d}} \\
			\asymp &  
			n_{\mathrm{Q}}\cdot
			\left(\left(\frac{n_{\mathrm{Q}}}{\log(n_{\mathrm{P}} +  n_{\mathrm{Q}})}\right)^{\frac{2\alpha + 2d}{2 \gamma\alpha + d}}\right)^{- \frac{2\gamma\alpha }{2\alpha + 2 d}}\\
			=& \left(\frac{n_{\mathrm{Q}}}{\log (n_{\mathrm{P}} +  n_{\mathrm{Q}})}\right)^{\frac{d}{2\gamma\alpha + d}}\cdot \varepsilon^2 \cdot  \frac{ \log (n_{\mathrm{P}} +  n_{\mathrm{Q}}) }{\varepsilon^2} \gtrsim \frac{ \log (n_{\mathrm{P}} +  n_{\mathrm{Q}}) }{\varepsilon^2}
		\end{align*}
		for sufficiently large $n_{\mathrm{P}} +  n_{\mathrm{Q}}$,
		which implies the first conclusion. 
		For the second conclusion, we want 
		\begin{align*}
			n_{\mathrm{P}} +  n_{\mathrm{P}}^{a \frac{2\alpha + d}{2\gamma\alpha + d}} \lesssim n_{\mathrm{P}}^{a \frac{2\gamma\alpha }{2\gamma\alpha + d}} + n_{\mathrm{P}}^{1 +  a\frac{(\gamma + 1)\alpha + d}{2\gamma\alpha + d}} + n_{\mathrm{P}}^{a \frac{2(\gamma + 1)\alpha + d}{2\gamma\alpha + d} }. 
		\end{align*}
		The left-hand side is dominated by the term 
		$n_{\mathrm{P}}^{a \frac{2\alpha + d}{2\gamma\alpha + d}}$. 
		Thus, the conclusion holds since there holds
		$a \frac{2\alpha + d}{2\gamma\alpha + d} < a \frac{2(\gamma + 1)\alpha + d}{2\gamma\alpha + d}$ for all $a, \gamma, \alpha$, and $d$. 
		For the last statement, the same procedure yields that it suffices to show
		\begin{align*}
			n_{\mathrm{P}}^{\frac{1}{2}}  \lesssim \left(n_{\mathrm{P}}^{1-a\frac{d}{2\gamma\alpha +  d}} + n_{\mathrm{P}}^{ a - a\frac{2\gamma \alpha }{2\gamma\alpha + d}} \right)\cdot \varepsilon. 
		\end{align*}
		Thus, the conclusion holds since the following conditions hold:
		\begin{align*}
			\frac{1}{2} \leq 1 - \frac{d}{2\alpha + 2d}<  1 - a \frac{d }{2\gamma\alpha + d}, \;\;\;\text{ and }\;\;\ n_{\mathrm{P}}^{\frac{\alpha}{2\alpha + 2d}}\cdot \varepsilon  \gtrsim 1. 
		\end{align*}

		\textbf{Case ($\mathbf{iv}$)}: 
		$ \frac{n_{\mathrm{P}} \varepsilon^2}{\log (n_{\mathrm{P}} + n_{\mathrm{Q}})} <  \left(\frac{n_{\mathrm{Q}}}{\log (n_{\mathrm{P}} + n_{\mathrm{Q}})}\right)^{\frac{2\alpha + 2d}{2 \gamma\alpha + d}}$
		and $\frac{(\gamma+1)\alpha}{d}<  1$. 
		Compared to case ($iii$), the value is $\lambda$ is multiplied by a factor $\delta^{(\frac{(\gamma + 1)\alpha}{d} - 1)\wedge 0} >1 $. 
		Thus, the first conclusion holds since the value of the left-hand side increases. 
		The third conclusion holds since the value of the left-hand side decreases. 
		For the second conclusion, by replacing all terms with equivalent order of $n_{\mathrm{P}}$, it suffices to prove
		\begin{align*}
			n_{\mathrm{P}} +  n_{\mathrm{P}}^{a \frac{-2\gamma \alpha + 3d}{2\gamma\alpha + d}} \lesssim n_{\mathrm{P}}^{a \frac{2\gamma\alpha }{2\gamma\alpha + d}} + n_{\mathrm{P}}^{1 +  a - a \frac{2\gamma \alpha }{2\gamma\alpha + d}} + n_{\mathrm{P}}^{2 a  - a\frac{4\gamma \alpha - d}{2\gamma\alpha + d} }. 
		\end{align*}
		We have $n_{\mathrm{P}} \lesssim n_{\mathrm{P}}^{1 +   \frac{ad }{2\gamma\alpha + d}} = n_{\mathrm{P}}^{1 +  a - a \frac{2\gamma \alpha }{2\gamma\alpha + d}}$ and
		$n_{\mathrm{P}}^{a \frac{-2\gamma \alpha + 3d}{2\gamma\alpha + d}} \lesssim n_{\mathrm{P}}^{ \frac{  3ad}{2\gamma\alpha + d}} = n_{\mathrm{P}}^{2 a  - a\frac{4\gamma \alpha - d}{2\gamma\alpha + d} }$. 
	\end{proof}

	\subsection{Proofs of Results in Section \ref{sec:proofoferroranalysis}}
	
	
	\begin{proof}\textbf{of Lemma \ref{lem:laplacebound}}\;\;
		The conclusion follows from (2.18) in \cite{wainwright2019high}.  
	\end{proof}

	In the subsequent proof, we define the empirical measure $\mathrm{D} := \frac{1}{n} \sum_{i=1}^n \delta_{(X_i,Y_i)}$ given samples $D = \{(X_1, Y_1), \cdots, (X_n,Y_n)\}$, where $\delta$ is the Dirac function. 
	To conduct our analysis, we need to introduce the following fundamental descriptions of \textit{covering number} which enables an approximation of an infinite set by finite subsets.
	\begin{definition}[Covering Number]
		Let $(\mathcal{X}, d)$ be a metric space and $A \subset \mathcal{X}$. For $\varepsilon>0$, the $\varepsilon$-covering number of $A$ is denoted as 
		\begin{align*}
			\mathcal{N}(A, d, \varepsilon) 
			:= \min \biggl\{ n \geq 1 : \exists x_1, \ldots, x_n \in \mathcal{X} \text{ such that } A \subset \bigcup^n_{i=1} B(x_i, \varepsilon) \biggr\},
		\end{align*}
		where $B(x, \varepsilon) := \{ x' \in \mathcal{X} : d(x, x') \leq \varepsilon \}$.
	\end{definition}
	
	To prove Lemma \ref{VCIndex}, we need the following fundamental lemma concerning the VC dimension of random partitions in Section \ref{sec:partition}. To this end, let $p \in \mathbb{N}$ be fixed and $\pi_p$ be a partition of $\mathcal{X}$ with the number of splits $p$ and $\pi_{(p)}$ denote the collection of all partitions $\pi_p$.

	\begin{lemma}\label{VCIndex}
		Let $\tilde{\mathcal{A}}$ be the collection of all cells $\times_{i=1}^d [a_i,b_i]$ in $\mathbb{R}^d$. For all $0<\varepsilon<1$, there exists a universal constant $C$ such that
		\begin{align*}
			\mathcal{N}(\eins_{\tilde{\mathcal{A}}}, \|\cdot\|_{L_1(Q)}, \varepsilon)\leq C(2d+1)(4e)^{2d+1}(1/\varepsilon)^{2d}.
		\end{align*}
	\end{lemma}
	
	\begin{proof}\textbf{of Lemma \ref{VCIndex}}\;\;
		The result follows directly from Theorem 9.2 in \cite{Kosorok2008introduction}.
	\end{proof}
	
	Before we proceed, we list the well-known Bernstein's inequality that will be used frequently in the proofs.

	\begin{lemma}[Bernstein's inequality] \label{lem::bernstein}
		Let $B>0$ and $\sigma>0$ be real numbers, and $n\geq 1$ be an integer. Furthermore, let $\xi_1,\ldots,\xi_n$ be independent random variables satisfying $\mathbb{E}_{P} \xi_i=0$, $\|\xi_i\|_\infty\leq B$, and $\mathbb{E}_{P} \sum_{i=1}^n \xi_i^2\leq \sigma^2$ for all $i=1,\ldots,n$. Then for all $\tau>0$, we have
		\begin{align*}
			P\biggl(\frac{1}{n}\sum_{i=1}^n\xi_i\geq \sqrt{\frac{2\sigma^2\tau}{n^2}}+\frac{2B\tau}{3n}\biggr)\leq e^{-\tau}.
		\end{align*}
	\end{lemma}
	
	\begin{proof}\textbf{of Lemma \ref{lem:cellbound}}\;\;
		Since $\mathrm{P}$ and $\mathrm{Q}$ share the same marginal distribution, it is equivalent to prove
		\begin{align*}
			& \left|\frac{1}{n_{\mathrm{P}} + n_{\mathrm{Q}}}\left(\sum_{i=1}^{n_{\mathrm{P}} }  \mathbf{1}\{X_i \in A_{(p)}^{j}\} + \lambda  \sum_{i=1}^{ n_{\mathrm{Q}}}  \mathbf{1}\{X_i \in A_{(p)}^{j}\} \right) - \frac{n_{\mathrm{P}} + \lambda n_{\mathrm{Q}}}{n_{\mathrm{P}} + n_{\mathrm{Q}}}\int_{A_{(p)}^{j}}d\mathrm{P}_X(x')  \right|\\ \lesssim & \sqrt{\frac{ (n_{\mathrm{P}} + \lambda^2 n_{\mathrm{Q}}) \log ( n_{\mathrm{P}} + n_{\mathrm{Q}} )}{2^p(n_{\mathrm{P}} + n_{\mathrm{Q}})^2}}
		\end{align*}
		holds with probability $\mathrm{P}^{n_{\mathrm{P}} + n_{\mathrm{Q}}}$ at least $1-1/(n_{\mathrm{P}} + n_{\mathrm{Q}})^2$ for $X_i \sim \mathrm{P}_X$. 
		Let 
		$\tilde{\mathcal{A}}$ be the collection of all cells $\times_{i=1}^d [a_i,b_i]$ in $\mathbb{R}^d$.
		Applying Lemma \ref{VCIndex} with $(\mathrm{D}_X+\mathrm{P}_X)/2$, there exists an $\varepsilon$-net $\{\tilde{A}^k\}_{k=1}^K\subset \tilde{\mathcal{A}}$ with 
		\begin{align}\label{equ::Kcover}
			K\leq C(2d+1)(4e)^{2d+1}(1/\varepsilon)^{2d}
		\end{align}
		such that
		for any $j\in \mathcal{I}_p$, there exist some $k\in \{1,\ldots,K\}$ such that
		\begin{align*}
			\|\eins\{x\in A_{(p)}^j\}-\eins\{x\in \tilde{A}^k\}\|_{L_1((\mathrm{D}_X+\mathrm{P}_X)/2)}\leq \varepsilon,
		\end{align*}
		Since 
		\begin{align*}
			&\|\eins\{x\in A_{(p)}^j\}-\eins\{x\in\tilde{A}^k\}\|_{L_1((\mathrm{D}_X+\mathrm{P}_X)/2)}\\
			= & 1/2\cdot \|\eins\{x\in A_{(p)}^j\}-\eins\{x\in \tilde{A}^k\}\|_{L_1(\mathrm{D}_X)}+1/2\cdot \|\eins\{x\in A_{(p)}^j\}-\eins\{x\in \tilde{A}^k\}\|_{L_1(\mathrm{P}_X)},
		\end{align*}
		we get
		\begin{align}\label{equ::einsapj}
			\|\eins\{x\in A_{(p)}^j\}-\eins\{x\in \tilde{A}^k\}\|_{L_1(\mathrm{D}_X)}\leq 2\varepsilon, \quad  \|\eins\{x\in A_{(p)}^j\}-\eins\{x\in \tilde{A}^k\}\|_{L_1(\mathrm{P}_X)}\leq 2\varepsilon.
		\end{align}
		Consequently, by the definition of the covering number and the triangle inequality, for any $j\in \mathcal{I}_p$, there holds
		\begin{align*}
			&\biggl|\frac{1}{n_{\mathrm{P}} + n_{\mathrm{Q}}}\left(\sum_{i=1}^{n_{\mathrm{P}} }\eins\{x\in A_{(p)}^j\}(X_i) + \lambda \sum_{i=1}^{ n_{\mathrm{Q}}}\eins\{x\in A_{(p)}^j\}(X_i)\right)- \int_{\tilde{A}_{(p)}^j}d\mathrm{P}_X(x') \biggr|\\
			\leq &
			\biggl|\frac{1}{n_{\mathrm{P}} + n_{\mathrm{Q}}} \left( \sum_{i=1}^{n_{\mathrm{P}} } \eins\{x\in \tilde{A}^k\}(X_i) + \lambda \sum_{i=1}^{ n_{\mathrm{Q}}} \eins\{x\in \tilde{A}^k\}(X_i)\right) - \int_{\tilde{A}^k}d\mathrm{P}_X(x')\biggr| \\ & + \|\eins\{x\in A_{(p)}^j\}-\eins\{x\in \tilde{A}^k\}\|_{L_1(\mathrm{D}_X)}
			+ \|\eins\{x\in A_{(p)}^j\}-\eins\{x\in \tilde{A}^k\}\|_{L_1(\mathrm{P}_X)}
			\\ \leq& \biggl|\frac{1}{n_{\mathrm{P}} + n_{\mathrm{Q}}} \left( \sum_{i=1}^{n_{\mathrm{P}} } \eins\{x\in \tilde{A}^k\}(X_i) + \lambda \sum_{i=1}^{ n_{\mathrm{Q}}} \eins\{x\in \tilde{A}^k\}(X_i)\right) -\int_{\tilde{A}^k}d\mathrm{P}_X(x')\biggr| + 4\varepsilon.
		\end{align*}
		Therefore, we get
		\begin{align}\label{equ::Isupjip}
			& \sup_{j\in \mathcal{I}_p} \,\biggl|\frac{1}{n_{\mathrm{P}} + n_{\mathrm{Q}}}\left( \sum_{i=1}^{n_{\mathrm{P}} } \eins\{x\in A_{(p)}^j\}(X_i) + \lambda \sum_{i=1}^{ n_{\mathrm{Q}}} \eins\{x\in A_{(p)}^j\}(X_i)\right)- \int_{\tilde{A}_{(p)}^j}d\mathrm{P}_X(x') \biggr|\\
			\leq & \sup_{1\leq k\leq K} \,\biggl|\frac{1}{n_{\mathrm{P}} + n_{\mathrm{Q}}} \left(\sum_{i=1}^{n_{\mathrm{P}} } \eins\{x\in\tilde{A}^k\}(X_i) + \lambda \sum_{i=1}^{ n_{\mathrm{Q}}} \eins\{x\in\tilde{A}^k\}(X_i)\right)-\int_{\tilde{A}^k}d\mathrm{P}_X(x')\biggr| + 4\varepsilon.
		\end{align}
		For any fixed $1\leq k\leq K$, let the random variable $\xi_i$ be defined by $\xi_i:=\eins\{X_i\in\tilde{A}^k\}-\int_{\tilde{A}^k}d\mathrm{P}_X(x')$ for $i = 1,\cdots, n_{\mathrm{P}}$ and $\xi_i:= \lambda (\eins\{X_i\in\tilde{A}^k\}-\int_{\tilde{A}^k}d\mathrm{P}_X(x'))$ for $i = n_{\mathrm{P}}+1,\cdots, n_{\mathrm{P}} + n_{\mathrm{Q}}$.
		Then we have $\mathbb{E}_{\mathrm{P}_X}\xi_i=0$, $\|\xi\|_\infty\leq 1\vee \lambda $, and $\mathbb{E}_{\mathrm{P}_X} \sum_{i = 1}^{n_{\mathrm{P}}+ n_{\mathrm{Q}}}\xi_i^2\leq (n_{\mathrm{P}}+ \lambda^2 n_{\mathrm{Q}})\int_{\tilde{A}^k}d\mathrm{P}_X(x')$. 
		According to Assumption \ref{asp:boundedmarginal}, there holds $\mathbb{E}_{\mathrm{P}_X} \sum_{i = 1}^{n_{\mathrm{P}}+ n_{\mathrm{Q}}} \xi_i^2\leq (n_{\mathrm{P}}+ \lambda^2 n_{\mathrm{Q}}) \cdot \overline{c}\cdot 2^{-p}$. 
		Applying Bernstein's inequality in Lemma \ref{lem::bernstein}, we obtain
		\begin{align*}
			\biggl|\frac{1}{n_{\mathrm{P}} + n_{\mathrm{Q}}} \sum_{i=1}^{n_{\mathrm{P}} + n_{\mathrm{Q}}} \eins\{X_i\in \tilde{A}^k\}-
			\frac{n_{\mathrm{P}} + \lambda n_{\mathrm{Q}}}{n_{\mathrm{P}} + n_{\mathrm{Q}}} 
			\int_{\tilde{A}^k}d\mathrm{P}_X(x')\biggr|\leq & \sqrt{\frac{\overline{c}\; 2^{1-p} \tau (n_{\mathrm{P}} + \lambda^2 n_{\mathrm{Q}})}{(n_{\mathrm{P}} + n_{\mathrm{Q}})^2} } \\ & +\frac{2\tau \cdot 1\vee \lambda \cdot  \log (n_{\mathrm{P}} + n_{\mathrm{Q}})}{3(n_{\mathrm{P}} + n_{\mathrm{Q}})}
		\end{align*}
		with probability $\mathrm{P}^{n_{\mathrm{P}} + n_{\mathrm{Q}}}$ at least $1-2e^{-\tau}$. 
		Then the union bound together with the covering number estimate \eqref{equ::Kcover} implies that
		\begin{align*}
			& \sup_{1\leq k\leq K}\, \biggl|\frac{1}{n_{\mathrm{P}} + n_{\mathrm{Q}}} \sum_{i=1}^{n_{\mathrm{P}} + n_{\mathrm{Q}}} \eins\{X_i\in\tilde{A}^k\}- \frac{n_{\mathrm{P}} + \lambda n_{\mathrm{Q}}}{n_{\mathrm{P}} + n_{\mathrm{Q}}} 
			\int_{\tilde{A}^k}d\mathrm{P}_X(x')\biggr|\\
			\leq & \sqrt{\frac{\overline{c}\cdot 2^{1-p} (n_{\mathrm{P}} + \lambda^2 n_{\mathrm{Q}})(\tau+\log (2K))}{(n_{\mathrm{P}} + n_{\mathrm{Q}})^2}} +\frac{2(\tau+\log (2K)) \cdot 1\vee \lambda \cdot \log (n_{\mathrm{P}} + n_{\mathrm{Q}})}{3(n_{\mathrm{P}} + n_{\mathrm{Q}})}
		\end{align*}
		with probability $\mathrm{P}^{n_{\mathrm{P}} + n_{\mathrm{Q}}}$ at least $1-e^{-\tau}$. Let $\tau=2\log (n_{\mathrm{P}} + n_{\mathrm{Q}})$ and $\varepsilon=1/(n_{\mathrm{P}} + n_{\mathrm{Q}})$. Then for any $n_{\mathrm{P}} + n_{\mathrm{Q}}> N_1:=(2C)\wedge (2d+1)\wedge (4e)$, we have $\tau+\log (2K)=2\log n+ \log (2C)+\log (2d+1)+(2d+1)\log (4e)+ 2d\log n\leq (4d+5)\log (n_{\mathrm{P}} + n_{\mathrm{Q}})$. Therefore, we have
		\begin{align*}
			& \sup_{1\leq k\leq K}\, \biggl|\frac{1}{n_{\mathrm{P}} + n_{\mathrm{Q}}} \sum_{i=1}^{n_{\mathrm{P}} + n_{\mathrm{Q}}} \eins\{X_i \in \tilde{A}^k\}-
			\frac{n_{\mathrm{P}} + \lambda n_{\mathrm{Q}}}{n_{\mathrm{P}} + n_{\mathrm{Q}}} 
			\int_{\tilde{A}^k}d\mathrm{P}_X(x')\biggr|
			\\
			\leq & \sqrt{\frac{\overline{c}\cdot 2^{1-p} ( n_{\mathrm{P}} + \lambda^2 n_{\mathrm{Q}})(4d+5)\log (n_{\mathrm{P}} + n_{\mathrm{Q}})}{(n_{\mathrm{P}} + n_{\mathrm{Q}})^2}} +\frac{2(4d+5) \cdot 1\vee \lambda \cdot  \log (n_{\mathrm{P}} + n_{\mathrm{Q}})}{3(n_{\mathrm{P}} + n_{\mathrm{Q}})}
		\end{align*}
		with probability $\mathrm{P}^{n_{\mathrm{P}} + n_{\mathrm{Q}}}$ at least $1-1/{(n_{\mathrm{P}} + n_{\mathrm{Q}})}^2$. This together with \eqref{equ::Isupjip} yields that
		\begin{align*}
			& \sup_{j\in \mathcal{I}_p} \,\biggl|\frac{1}{n_{\mathrm{P}} + n_{\mathrm{Q}}}\sum_{i=1}^{n_{\mathrm{P}} + n_{\mathrm{Q}}} \eins\{x\in A_{(p)}^j\}- 
			\frac{n_{\mathrm{P}} + \lambda n_{\mathrm{Q}}}{n_{\mathrm{P}} + n_{\mathrm{Q}}} 
			\int_{\tilde{A}_{(p)}^j}d\mathrm{P}_X(x') \biggr|\\
			\leq & \sqrt{\frac{\overline{c}\cdot 2^{1-p}(4d+5)(n_{\mathrm{P}} + \lambda^2 n_{\mathrm{Q}})\log (n_{\mathrm{P}} + n_{\mathrm{Q}})}{(n_{\mathrm{P}} + n_{\mathrm{Q}})^2 }} +\frac{2(4d+5) \cdot 1\vee \lambda \cdot \log (n_{\mathrm{P}} + n_{\mathrm{Q}})}{3(n_{\mathrm{P}} + n_{\mathrm{Q}})}+\frac{4}{n_{\mathrm{P}} + n_{\mathrm{Q}}}.
		\end{align*}
		We now compare to see which term is dominating. 
		Apply Lemma \ref{lem:boudnofnlamdasq} (the first conclusion), we have 
		\begin{align*}
			\frac{2^{-p} (n_{\mathrm{P}} + \lambda^2 n_{\mathrm{Q}})\cdot  \log (n_{\mathrm{P}} +  n_{\mathrm{Q}})}{ (1\vee \lambda^2) \cdot \log^2 (n_{\mathrm{P}} +  n_{\mathrm{Q}})}
			\gtrsim
			\frac{\delta^{-\frac{ 2\alpha + d }{d }} }{\varepsilon^{2}\cdot ( 1\vee \lambda^2)}.
		\end{align*}
		When $\gamma \leq \frac{d}{\alpha} - 1$ and $\gamma \leq \frac{d}{2\alpha}$,
		\begin{align*}
			\frac{\delta^{-\frac{ 2\alpha + d }{d }} }{\varepsilon^{2}\cdot ( 1\vee \lambda^2)} = \varepsilon^{-2}\cdot \delta^{-\frac{ 2\alpha + d }{d } - \frac{(\gamma - 1)\alpha }{d}} > 
			\varepsilon^{-2}\cdot \delta^{-\frac{ 2\alpha + 3d }{2d }}\geq &
			\varepsilon^{-2}\cdot (n_{\mathrm{P}}\cdot \varepsilon^2)^{\frac{2\alpha +3 d}{4\alpha + 4d}}\\
			= & (n_{\mathrm{P}}\cdot \varepsilon^{-\frac{4\alpha + 2d}{2\alpha + 3d}})^{\frac{2\alpha + 3d}{4\alpha + 4d}} \gtrsim 1. 
		\end{align*}
		When $\gamma \leq \frac{d}{\alpha} - 1$ and $\gamma > \frac{d}{2\alpha}$,
		\begin{align*}
			\frac{\delta^{-\frac{ 2\alpha + d }{d }} }{\varepsilon^{2}\cdot ( 1\vee \lambda^2)} = \varepsilon^{-2}\cdot \delta^{-\frac{ 2\alpha + d }{d }} \geq \varepsilon^{-2}\cdot (n_{\mathrm{P}}\cdot \varepsilon^2)^{\frac{2\alpha + d}{2\alpha + 2d}} = (n_{\mathrm{P}}\cdot \varepsilon^{-\frac{2d}{2\alpha + d}})^{\frac{2\alpha + d}{2\alpha + 2d}} \gtrsim 1. 
		\end{align*}
		When $\gamma > \frac{d}{\alpha} - 1$ and $\gamma \geq 1$, the derivation is the same as the last situation.
		When $\gamma > \frac{d}{\alpha} - 1$ and $\gamma < 1$,
		\begin{align*}
			\frac{\delta^{-\frac{ 2\alpha + d }{d }} }{\varepsilon^{2}\cdot ( 1\vee \lambda^2)} = \varepsilon^{-2}\cdot \delta^{-\frac{ 2\alpha + d }{d } - \frac{2 \gamma\alpha - d}{d}} \geq 
			\varepsilon^{-2}\cdot (n_{\mathrm{P}}\cdot \varepsilon^2)^{\frac{2\alpha}{ d}}
			=  (n_{\mathrm{P}}\cdot \varepsilon^{-\frac{4\alpha-2d}{2\alpha}})^{\frac{2\alpha }{d}} \gtrsim 1. 
		\end{align*}
		These derivations imply that the first term is dominating the upper bound,  
		which implies the desired result. 
	\end{proof}

	\begin{proof}\textbf{of Lemma \ref{lem:cellboundY}}\;\;
		For notation simplicity, let $(X_i, Y_i) \sim \mathrm{P}$ for $i = 1, \cdots, n_{\mathrm{P}}$ and $(X_i, Y_i) \sim \mathrm{Q}$ for $i = n_{\mathrm{P}} + 1, \cdots, n_{\mathrm{P}} + n_{\mathrm{Q}}$. 
		Let $\tilde{Y}_i = Y_i$ for $i = 1, \cdots, n_{\mathrm{P}}$ and $\tilde{Y}_i = \lambda \cdot Y_i$ for $i = n_{\mathrm{P}} + 1, \cdots, n_{\mathrm{P}} + n_{\mathrm{Q}}$. 
		Let $\eta_{\mathrm{P}, \mathrm{Q}}^*(x) = \left(n_{\mathrm{P}}\eta_{\mathrm{P}}^*(x) +  \lambda n_{\mathrm{Q}}\eta_{\mathrm{Q}}^*(x)\right)/(n_{\mathrm{P}} + \lambda n_{\mathrm{Q}})$. 
		Let 
		$\tilde{\mathcal{A}}$ be the collection of all cells $\times_{i=1}^d [a_i,b_i]$ in $\mathbb{R}^d$.
		Then there exists an $\varepsilon$-net $\{\tilde{A}^k\}_{k=1}^K\subset \tilde{\mathcal{A}}$ with $K$ bounded by \eqref{equ::Kcover} such that 
		for any $j\in \mathcal{I}_p$,
		\eqref{equ::einsapj} holds
		for some $k\in \{1,\ldots,K\}$. 
		Consequently, by the definition of the covering number and the triangle inequality, for any $j\in \mathcal{I}_p$, there holds
		\begin{align}
			&\biggl|\sum_{i=1}^{n_{\mathrm{P}} + n_{\mathrm{Q}}}\eins\{X_i\in A_{(p)}^j\}\tilde{Y}_i -\int_{A_{(p)}^j}\eta_{{\mathrm{P}} , {\mathrm{Q}}}^*(x') d\mathrm{P}_X(x')\biggr|\nonumber\\
			\leq & \biggl|\sum_{i=1}^{n_{\mathrm{P}} + n_{\mathrm{Q}}} \eins\{X_i\in\tilde{A}^k\}\tilde{Y}_i -\int_{\tilde{A}^k} \eta_{{\mathrm{P}} , {\mathrm{Q}}}^*(x') d\mathrm{P}_X(x')\biggr| \nonumber\\
			& +
			\int_{\mathbb{R}^d} \bigl|\eins\{x'\in A_{(p)}^j\}-\eins\{x'\in\tilde{A}^k\}\bigr| \bigl|\eta_{{\mathrm{P}} , {\mathrm{Q}}}^*(x') \bigr| d\mathrm{P}_X(x') +\sum_{i=1}^{n_{\mathrm{P}} + n_{\mathrm{Q}}} \bigl|\eins\{X_i\in\tilde{A}^k\} - \eins\{X_i\in A_{(p)}^j\}\bigr| \nonumber\\
			\leq & \biggl|\sum_{i=1}^{n_{\mathrm{P}} + n_{\mathrm{Q}}} \eins\{X_i\in \tilde{A}^k\}\tilde{Y}_i -\int_{\tilde{A}^k} \eta_{{\mathrm{P}} , {\mathrm{Q}}}^*(x') d\mathrm{P}_X(x')\biggr|\nonumber\\
			&+\max_{1\leq i\leq n} \cdot \|\eins\{x\in A_{(p)}^j\}-\eins\{x\in \tilde{A}^k\}\|_{L_1(\mathrm{D}_X)} +  \|\eins\{x\in A_{(p)}^j\}-\eins\{x\in\tilde{A}^k\}\|_{L_1(\mathrm{P}_X)}\nonumber\\
			\leq & \biggl|\sum_{i=1}^{n_{\mathrm{P}} + n_{\mathrm{Q}}} \eins\{X_i\in \tilde{A}^k\}\tilde{Y}_i -\int_{\tilde{A}^k} \eta_{{\mathrm{P}} , {\mathrm{Q}}}^*(x') d\mathrm{P}_X(x')\biggr|+4\varepsilon.\label{equ::Isupjip1}
		\end{align}
		where the last inequality follow from the condition $\mathcal{Y} = \{0,1\}$.
		
		For any fixed $1\leq k\leq K$, let the random variable $\tilde{\xi}_i$ be defined by $\tilde{\xi}_i:=\eins\{X_i\in \tilde{A}^k\}\tilde{Y}_i-\int_{\tilde{A}^k} \eta^*_{\mathrm{P}}(x')\, d\mathrm{P}_X(x')$ for $i = 1,\cdots, n_{\mathrm{P}}$ and $\tilde{\xi}_i:=\eins\{X_i\in \tilde{A}^k\}\tilde{Y}_i- \lambda \int_{\tilde{A}^k} \eta^*_{\mathrm{Q}}(x')\, d\mathrm{P}_X(x')$ for $i = n_{\mathrm{P}} + 1,\cdots, n_{\mathrm{P}} + n_{\mathrm{Q}}$.
		Then we have $\mathbb{E}\tilde{\xi}_i=0$, $\|\xi\|_\infty\leq 1\vee \lambda$, and $\mathbb{E}\sum_{i=1}^{n_{\mathrm{P}} + n_{\mathrm{Q}}}\tilde{\xi}_i^2\leq  (n_{\mathrm{P}} + \lambda^2  n_{\mathrm{Q}})\int_{\tilde{A}^k}d\mathrm{P}_X(x')$. According to Assumption \ref{asp:boundedmarginal}, there holds 
		$\mathbb{E} \sum_{i=1}^{n_{\mathrm{P}} + n_{\mathrm{Q}}} \tilde{\xi}_i^2\leq  \overline{c}\cdot (n_{\mathrm{P}} + \lambda^2 n_{\mathrm{Q}})\cdot  2^{-p}$.
		Applying Bernstein's inequality in Lemma \ref{lem::bernstein}, we obtain
		\begin{align*}
			& \biggl| \frac{1}{ n_{\mathrm{P}} + n_{\mathrm{Q}} }\sum_{i=1}^{n_{\mathrm{P}} + n_{\mathrm{Q}}}\eins\{X_i\in A_{(p)}^j\}\tilde{Y}_i -
			\frac{n_{\mathrm{P}} + \lambda n_{\mathrm{Q}}}{ n_{\mathrm{P}} + n_{\mathrm{Q}} }
			\int_{A_{(p)}^j}\eta_{{\mathrm{P}} , {\mathrm{Q}}}^*(x') d\mathrm{P}_X(x')\biggr| \\
			\leq & \sqrt{\frac{ \overline{c}\cdot 2^{1-p} \tau (n_{\mathrm{P}} + \lambda^2 n_{\mathrm{Q}})}{(n_{\mathrm{P}} + n_{\mathrm{Q}})^2 }}    +  \frac{2\tau \cdot 1\vee \lambda \cdot \log (n_{\mathrm{P}} + n_{\mathrm{Q}})}{3(n_{\mathrm{P}} + n_{\mathrm{Q}})}
		\end{align*}
		with probability $\mathrm{P}^{n_{\mathrm{P}} + n_{\mathrm{Q}}}$ at least $1-2e^{-\tau}$. 
		Similar to the proof of Lemma \ref{lem:cellbound}, one can show that for any $n\geq N_1$, there holds
		\begin{align*}
			& \sup_{1\leq k\leq K} \biggl| \frac{1}{ n_{\mathrm{P}} + n_{\mathrm{Q}} } \sum_{i=1}^{n_{\mathrm{P}} + n_{\mathrm{Q}}}\eins\{X_i\in A_{(p)}^j\}\tilde{Y}_i -
			\frac{n_{\mathrm{P}} + \lambda n_{\mathrm{Q}}}{ n_{\mathrm{P}} + n_{\mathrm{Q}} }
			\int_{A_{(p)}^j}\eta_{{\mathrm{P}} , {\mathrm{Q}}}^*(x') d\mathrm{P}_X(x')\biggr| \\ \leq  &\sqrt{\frac{ \overline{c}\cdot 2^{1-p} \tau (n_{\mathrm{P}} + \lambda^2 n_{\mathrm{Q}}) }{(n_{\mathrm{P}} + n_{\mathrm{Q}})^2}}  +\frac{2\tau \cdot 1\vee \lambda \cdot \log (n_{\mathrm{P}} + n_{\mathrm{Q}})}{3(n_{\mathrm{P}} + n_{\mathrm{Q}})}
		\end{align*}
		with probability $\mathrm{P}^{n_{\mathrm{P}}+n_{\mathrm{Q}}}$ at least $1-1/(n_{\mathrm{P}} + n_{\mathrm{Q}})^2$. This together with \eqref{equ::Isupjip1} yields that
		\begin{align*}
			&\biggl| \frac{1 }{ n_{\mathrm{P}} +  n_{\mathrm{Q}} } 
			\sum_{i=1}^{n_{\mathrm{P}} + n_{\mathrm{Q}}}\eins\{X_i\in A_{(p)}^j\}\tilde{Y}_i -
			\frac{n_{\mathrm{P}} + \lambda  n_{\mathrm{Q}} }{ n_{\mathrm{P}} +  n_{\mathrm{Q}} } \int_{A_{(p)}^j}\eta_{{\mathrm{P}} , {\mathrm{Q}}}^*(x') d\mathrm{P}_X(x')\biggr| \\\leq& \sqrt{\frac{\overline{c}\cdot 2^{1-p}(4d+5) ( n_{\mathrm{P}} + \lambda^2 n_{\mathrm{Q}} )\log (n_{\mathrm{P}} + n_{\mathrm{Q}})}{(n_{\mathrm{P}} + n_{\mathrm{Q}})^2 }} +\frac{2(4d+5) \cdot 1\vee\lambda\cdot  \log (n_{\mathrm{P}} + n_{\mathrm{Q}})}{3(n_{\mathrm{P}} + n_{\mathrm{Q}})}+\frac{4}{n_{\mathrm{P}} + n_{\mathrm{Q}}}.
		\end{align*}
		This implies the desired result by analogous derivations as in the last part of the proof of Lemma \ref{lem:cellbound}.
	\end{proof}

	\begin{proof}\textbf{of Lemma \ref{lem::treeproperty}}\;\;
		According to the max-edge partition rule, when the depth of the tree $i$ is a multiple of dimension $d$, each cell of the tree partition is a high-dimensional cube with a side length $2^{-i/d}$.
		On the other hand, when the depth of the tree $p$ is not a multiple of dimension $d$, we consider the max-edge tree partition with depth $\lfloor i/d\rfloor$ and $\lceil i/d\rceil$, whose corresponding side length of the higher dimensional cube is $2^{-\lfloor i/d\rfloor}$ and $2^{-\lceil i/d \rceil}$. Note that in the splitting procedure of max-edge partition, the side length of each sub-rectangle decreases monotonically with the increase of $p$, so the side length of a random tree partition cell is between $2^{-\lceil i/d \rceil}$ and $2^{-\lfloor i/d \rfloor}$. This implies that
		\begin{align*}
			\sqrt{d}\cdot 2^{-\lceil i/d \rceil}\leq \mathrm{diam}(A_{(i)}^j)\leq \sqrt{d}\cdot 2^{-\lfloor i/d \rfloor}
		\end{align*}
		Since $i/d-1\leq \lfloor i/d\rfloor\leq \lceil i/d\rceil \leq i/d+1$, we immediately get $2^{-1}\sqrt{d}\cdot 2^{-i/d}\leq \mathrm{diam}(A_{
			(i)
		}^j)\leq 2\sqrt{d} \cdot 2^{-i/d}$. 
	\end{proof}

	\subsection{Proof of Theorem \ref{thm:rangeparameter}}

	\begin{proof}\textbf{of Theorem \ref{thm:rangeparameter}}\;\;
		We let $\tilde{f}^{\pi, \mathcal{M}}_{\mathrm{P}, \mathrm{Q}} := (\tilde{f}^{\pi}_{\mathrm{P}, \mathrm{Q}} \cdot \eins_{[0,1]^d}) \circ \mathcal{M}$ and $X\sim \mathrm{P}_X$.
		The excess risk can be decomposed by 
		\begin{align*}
			\mathcal{R}_{\mathrm{P}}\left(\tilde{f}^{\pi, \mathcal{M}}_{\mathrm{P}, \mathrm{Q}}\right)-\mathcal{R}_{\mathrm{P}}^*  = &\int_{[0,1]^d} L(x,y, \tilde{f}^{\pi, \mathcal{M}}_{\mathrm{P}, \mathrm{Q}}(x))d\mathrm{P}(x,y) - \int_{[0,1]^d} L(x,y, f^{*}(x))d\mathrm{P}(x,y) \\
			+ & \int_{[0,1]^{dc}} L(x,y, \tilde{f}^{\pi, \mathcal{M}}_{\mathrm{P}, \mathrm{Q}}(x))d\mathrm{P}(x,y) - \int_{[0,1]^{dc}} L(x,y, f^{*}(x))d\mathrm{P}(x,y)\\
			\lesssim & \; \delta^{\frac{\alpha(\beta + 1)}{d}} +  \mathbb{P}(\mathcal{M}(X)\notin \times_{j=1}^d[0,1])
		\end{align*}
		where $\delta$ is defined in Theorem \ref{thm:utility}.
		Note that
		\begin{align*}
			\mathbb{P}(\mathcal{M}(X)\notin \times_{j=1}^d[0,1]) =
			\mathbb{P}(X\notin \times_{j=1}^d[\widehat{a^j} , \widehat{b^j}]]) \leq  \sum_{j=1}^d \mathbb{P}(X^j\notin [\widehat{a^j},\widehat{b^j}]).
		\end{align*}
		Denote the CDF of $X^j$ as $F^j$. Since
		\begin{align*}
			\mathbb{P}\left(\left| F^j(\widehat{a^j}) -F^j(a^j)\right|>\frac{2\log n_{\mathrm{Q}}}{n_{\mathrm{Q}}}\right) & =\mathrm{P}\left(  F^j(\widehat{a^j} ) >\frac{2\log n_{\mathrm{Q}}}{n_{\mathrm{Q}}}\right) \\
			& = \left(1 - \frac{2\log n_{\mathrm{Q}}}{ n_{\mathrm{Q}}}\right)^{n_{\mathrm{Q}}} = \frac{1}{n_{\mathrm{Q}}^2}
		\end{align*}
		and similar result holds for $b^j$, we have $\mathbb{P}(X^j\notin [\widehat{a^j},\widehat{b^j}]) \lesssim \frac{\log n_{\mathrm{Q}}}{n_{\mathrm{Q}}}$
		with probability $1- 1/n_{\mathrm{Q}}^2$. 
		Thus, there holds $\mathbb{P}(X\notin \times_{j=1}^d[\widehat{a^j},\widehat{b^j}])) \lesssim \frac{\log n_{\mathrm{Q}}}{n_{\mathrm{Q}}}$
		with probability $1- d/n_{\mathrm{Q}}^2$. 
	\end{proof}

	\subsection{Proof of Proposition \ref{prop:lepskibound}}

	\begin{proof}\textbf{of Proposition \ref{prop:lepskibound}}\;\;
		We prove the conclusion for $k = p_0$ first. 
		When $k = p_0$ and $x\in A^j_{(k)}$, the term $\tilde{\eta}^{\pi_k}_{\mathrm{P}, \mathrm{Q}}(x)-\mathbb{E}_{Y|X}\left[\widehat{\eta}^{\pi_k}_{\mathrm{P}, \mathrm{Q}}(x) \right]$ can be written as 
		\begin{align*}
			\frac{1}{\sum_{i=1}^{n_{\mathrm{P}}} \tilde{U}^{\mathrm{P} j}_i + \lambda\sum_{i=1}^{n_{\mathrm{Q}}}{U}^{\mathrm{Q} j}_i}\cdot \biggl(
			\overbrace{\sum_{i = 1}^{n_{\mathrm{P}}}\left(V_i^{\mathrm{P}j} - \mathbb{E}_{Y^{\mathrm{P}}|X^{\mathrm{P}}}[V_i^{\mathrm{P}j}]\right)}^{(I)}
			+ \overbrace{\lambda\sum_{i = 1}^{n_{\mathrm{Q}}}\left(V_i^{\mathrm{Q}j} - \mathbb{E}_{Y^{\mathrm{Q}}|X^{\mathrm{Q}}}[V_i^{\mathrm{Q}j}]\right)}^{(II)}\\	
			+ \underbrace{\frac{4}{\varepsilon}\sum_{i = 1}^{n_{\mathrm{P}}} \xi_i^j - 
				\frac{\sum_{i=1}^{n_{\mathrm{P}}} \mathbb{E}_{Y^{\mathrm{P}}|X^{\mathrm{P}}}[V_i^{\mathrm{P}j}] + \lambda\sum_{i=1}^{n_{\mathrm{Q}}}\mathbb{E}_{Y^{\mathrm{Q}}|X^{\mathrm{Q}}}[V_i^{\mathrm{Q}j}]}{\sum_{i=1}^{n_{\mathrm{P}}} {U}^{\mathrm{P} j}_i + \lambda\sum_{i=1}^{n_{\mathrm{Q}}}{U}^{\mathrm{Q} j}_i}  \cdot \frac{4}{\varepsilon}
				\sum_{i = 1}^{n_{\mathrm{P}}} \zeta_i^j}_{(III)}  
			\biggr).
		\end{align*}
		
		We bound the three parts separately. 
		Since $ \mathbb{E}_{Y^{\mathrm{P}}|X^{\mathrm{P}}}[V_i^{\mathrm{P}j}]  \leq U_i^{\mathrm{P}j}$, there holds 
		\begin{align*}
			\left|\frac{\sum_{i=1}^{n_{\mathrm{P}}} \mathbb{E}_{Y^{\mathrm{P}}|X^{\mathrm{P}}}[V_i^{\mathrm{P}j}] + \lambda\sum_{i=1}^{n_{\mathrm{Q}}}\mathbb{E}_{Y^{\mathrm{Q}}|X^{\mathrm{Q}}}[V_i^{\mathrm{Q}j}]}{\sum_{i=1}^{n_{\mathrm{P}}} {U}^{\mathrm{P} j}_i + \lambda\sum_{i=1}^{n_{\mathrm{Q}}}{U}^{\mathrm{Q} j}_i} \right| \leq 1. 
		\end{align*}
		Applying Lemma \ref{lem:laplacebound}, $(III)$ is bounded by
		\begin{align}\label{equ:singleboundfornoise}
			|(III)|\leq  2 \sqrt{ 6\varepsilon^{-2} n_{\mathrm{P}} \log (n_{\mathrm{P}} + n_{\mathrm{Q}})}
		\end{align}
		with probability  $\mathrm{R}$ at least $1 - 1 / (n_{\mathrm{P}} + n_{\mathrm{Q}})^2$ for all $j$.
		The proofs for $(I)$ and $(II)$ are analogous to the proof of Lemma \ref{lem:cellboundY} which utilizes Bernstein's inequality with the VC dimension of possible squares on $\mathbb{R}^d$. 
		$(I)$  is an independent sum of Bernoulli random variables with at most $\sum_{i=1}^{n_{\mathrm{P}}}U_i^{\mathrm{P}j}$ non-zero elements.
		Thus by Bernstein's inequality (Lemma \ref{lem::bernstein}), we have
		\begin{align*}
			|(I)| \leq \sqrt{\frac{\sum_{i=1}^{n_{\mathrm{P}}}U_i^{\mathrm{P}j}\cdot \tau }{2}}
		\end{align*}
		with probability $\mathrm{P}_{Y^{\mathrm{P}} | X^{\mathrm{P}}}^{n_{\mathrm{P}}}$ at least $1-e^{-\tau}$. 
		The order of $p_0$ yields that $p_0 \leq \log (n_{\mathrm{P}}\varepsilon^2 + n_{\mathrm{Q}})$. 
		Since there are at most $2^{p_0+1}$ cells in the union of all $\pi_k$, there holds 
		\begin{align}\label{equ:singleboundforp}
			|(I)| \leq \sqrt{\frac{\sum_{i=1}^{n_{\mathrm{P}}}U_i^{\mathrm{P}j}\cdot (\tau + (p_0+1)\log 2) }{2}} \leq &\sqrt{\sum_{i=1}^{n_{\mathrm{P}}}U_i^{\mathrm{P}j}(2\log (n_{\mathrm{P}} + n_{\mathrm{Q}}) + \log \varepsilon )}\\
			\leq & \sqrt{3\sum_{i=1}^{n_{\mathrm{P}}}U_i^{\mathrm{P}j}\log (n_{\mathrm{P}} + n_{\mathrm{Q}})}\nonumber
		\end{align}
		with probability  $\mathrm{P}_{Y^{\mathrm{P}} | X^{\mathrm{P}}}^{n_{\mathrm{P}}}$ at least $1- 1/ (n_{\mathrm{P}} + n_{\mathrm{Q}})^2$. 
		By Lemma \ref{lem:laplacebound}, for all $j$, with probability $\mathrm{R}$ at least $1 - 1/(n_{\mathrm{P}} + n_{\mathrm{Q}})^2$, we have
		\begin{align*}
			|\sum_{i=1}^{n_{\mathrm{P}}}\tilde{U}_i^{\mathrm{P}j} - \sum_{i=1}^{n_{\mathrm{P}}}{U}_i^{\mathrm{P}j}|\leq  2\sqrt{3\varepsilon^{-2} n_{\mathrm{P}} \log (n_{\mathrm{P}} + n_{\mathrm{Q}})}.
		\end{align*}
		If $\sum_{i=1}^{n_{\mathrm{P}}}\tilde{U}_i^{\mathrm{P}j} \geq 8 \varepsilon^{-2} n_{\mathrm{P}}$, 
		there holds 
		\begin{align*}
			\sum_{i=1}^{n_{\mathrm{P}}}\tilde{U}_i^{\mathrm{P}j} \geq 8 \varepsilon^{-2} n_{\mathrm{P}} -  2\sqrt{3\varepsilon^{-2} n_{\mathrm{P}} \log (n_{\mathrm{P}} + n_{\mathrm{Q}})} \geq 6 \varepsilon^{-2} n_{\mathrm{P}} ,
		\end{align*}
		as well as 
		\begin{align*}
			\frac{4}{3}\sum_{i=1}^{n_{\mathrm{P}}}\tilde{U}_i^{\mathrm{P}j} \geq  &\sum_{i=1}^{n_{\mathrm{P}}}{U}_i^{\mathrm{P}j} -  2\sqrt{3\varepsilon^{-2} n_{\mathrm{P}} \log (n_{\mathrm{P}} + n_{\mathrm{Q}})}  + \frac{1}{3} \sum_{i=1}^{n_{\mathrm{P}}}\tilde{U}_i^{\mathrm{P}j} \\
			\geq &  \sum_{i=1}^{n_{\mathrm{P}}}{U}_i^{\mathrm{P}j} - \sqrt{3\sum_{i=1}^{n_{\mathrm{P}}}\tilde{U}_i^{\mathrm{P}j}\log (n_{\mathrm{P}} + n_{\mathrm{Q}})}+ \frac{1}{3} \sum_{i=1}^{n_{\mathrm{P}}}\tilde{U}_i^{\mathrm{P}j} \geq  \sum_{i=1}^{n_{\mathrm{P}}}{U}_i^{\mathrm{P}j}
		\end{align*}
		for sufficiently large $n_{\mathrm{P}}$.
		In this case, we have 
		\begin{align}
			|(I)|  +  |(III)| \leq & \sqrt{3\sum_{i=1}^{n_{\mathrm{P}}}U_i^{\mathrm{P}j}\log (n_{\mathrm{P}} + n_{\mathrm{Q}})} + 2 \sqrt{ 6\varepsilon^{-2} n_{\mathrm{P}} \log (n_{\mathrm{P}} + n_{\mathrm{Q}})}\nonumber
			\\
			\leq &  2\sqrt{\sum_{i=1}^{n_{\mathrm{P}}}\tilde{U}_i^{\mathrm{P}j}\log (n_{\mathrm{P}} + n_{\mathrm{Q}})} +  2 \sqrt{  \sum_{i=1}^{n_{\mathrm{P}}}\tilde{U}_i^{\mathrm{P}j} \log (n_{\mathrm{P}} + n_{\mathrm{Q}})} \nonumber
			\\
			= & 4 \sqrt{  \sum_{i=1}^{n_{\mathrm{P}}}\tilde{U}_i^{\mathrm{P}j} \log (n_{\mathrm{P}} + n_{\mathrm{Q}})} .\label{equ:boundoflepski121}
		\end{align}
		If $\sum_{i=1}^{n_{\mathrm{P}}}\tilde{U}_i^{\mathrm{P}j} < 8 \varepsilon^{-2} n_{\mathrm{P}}$, 
		there holds 
		\begin{align*}
			\sum_{i=1}^{n_{\mathrm{P}}}{U}_i^{\mathrm{P}j}  \leq \sum_{i=1}^{n_{\mathrm{P}}}\tilde{U}_i^{\mathrm{P}j} + 2\sqrt{3\varepsilon^{-2} n_{\mathrm{P}} \log (n_{\mathrm{P}} + n_{\mathrm{Q}})} \leq 8 \varepsilon^{-2} n_{\mathrm{P}} + \frac{8}{3}\varepsilon^{-2} n_{\mathrm{P}}
		\end{align*}
		for sufficiently large $n_{\mathrm{P}}$.
		In this case, we get 
		\begin{align}\nonumber
			|(I)|  +  |(III)| \leq & \sqrt{3\sum_{i=1}^{n_{\mathrm{P}}}U_i^{\mathrm{P}j}\log (n_{\mathrm{P}} + n_{\mathrm{Q}})} + 2 \sqrt{ 6\varepsilon^{-2} n_{\mathrm{P}} \log (n_{\mathrm{P}} + n_{\mathrm{Q}})}\\
			\leq & 2\sqrt{ 8 \varepsilon^{-2} n_{\mathrm{P}}\log (n_{\mathrm{P}} + n_{\mathrm{Q}})} + 2 \sqrt{ 8\varepsilon^{-2} n_{\mathrm{P}} \log (n_{\mathrm{P}} + n_{\mathrm{Q}})}\nonumber
			\\
			= & 4 \sqrt{ 8\varepsilon^{-2} n_{\mathrm{P}} \log (n_{\mathrm{P}} + n_{\mathrm{Q}})}.\label{equ:boundoflepski122}
		\end{align}
		Combinging \eqref{equ:boundoflepski121} and \eqref{equ:boundoflepski122}, we get 
		\begin{align}\label{equ:boundoflepski12}
			|(I)|  +  |(III)| \leq 4 \sqrt{ \left(8\varepsilon^{-2} n_{\mathrm{P}}\right)\vee \left( \sum_{i=1}^{n_{\mathrm{P}}}\tilde{U}_i^{\mathrm{P}j} \right) \log (n_{\mathrm{P}} + n_{\mathrm{Q}})}.
		\end{align}
		Similarly to \eqref{equ:singleboundforp}, we have 
		\begin{align}\label{equ:singleboundforq}
			|(II)| \leq \lambda \sqrt{\frac{\sum_{i=1}^{n_{\mathrm{P}}}U_i^{\mathrm{Q}j}\cdot (\tau + (p_0+1)\log 2) }{2}} \leq &\lambda \sqrt{\sum_{i=1}^{n_{\mathrm{P}}}U_i^{\mathrm{Q}j}(2\log (n_{\mathrm{P}} + n_{\mathrm{Q}}) + \log \varepsilon)}\\
			\leq & \lambda \sqrt{3\sum_{i=1}^{n_{\mathrm{P}}}U_i^{\mathrm{Q}j}\log (n_{\mathrm{P}} + n_{\mathrm{Q}}) }\nonumber
		\end{align}
		with probability  $\mathrm{Q}_{Y^{\mathrm{Q}} | X^{\mathrm{Q}}}^{n_{\mathrm{Q}}}$ at least $1- 1/ (n_{\mathrm{P}} + n_{\mathrm{Q}})^2$.
		Combining \eqref{equ:boundoflepski12} and \eqref{equ:singleboundforq}, we get
		\begin{align*}
			|(I) + (II) + (III)| \leq & \left(4 \sqrt{ \left(8\varepsilon^{-2} n_{\mathrm{P}}\right)\vee \left( \sum_{i=1}^{n_{\mathrm{P}}}\tilde{U}_i^{\mathrm{P}j} \right)}
			+ \lambda \sqrt{3\sum_{i=1}^{n_{\mathrm{P}}}U_i^{\mathrm{Q}j} }\right) \cdot\sqrt{\log (n_{\mathrm{P}} + n_{\mathrm{Q}})}
			\\ 
			\leq & \sqrt{4C_1\left(8\varepsilon^{-2} n_{\mathrm{P}}\right)\vee \left( \sum_{i=1}^{n_{\mathrm{P}}}\tilde{U}_i^{\mathrm{P}j} \right) + C_1 \lambda^2\sum_{i=1}^{n_{\mathrm{P}}}U_i^{\mathrm{Q}j}} .
		\end{align*}
		with probability $\mathrm{P}_{Y^{\mathrm{P}} | X^{\mathrm{P}}}^{n_{\mathrm{P}}} \otimes \mathrm{Q}_{Y^{\mathrm{Q}} | X^{\mathrm{Q}}}^{n_{\mathrm{Q}}}\otimes \mathrm{R}$ at least $1-3 / (n_{\mathrm{P}} + n_{\mathrm{Q}})^2$.
		Here, we define constant $C_1$ for notation simplicity 
		\begin{align*}
			C_1  = 8\log (n_{\mathrm{P}} + n_{\mathrm{Q}}).
		\end{align*}
		We now prove the conclusion for $k<p_0$ with a slight modification on the above statement.
		For the term $(II)$, summing among all nodes with the same ancestor yields the sum of $2^{p_0 - k} \cdot n_{\mathrm{P}}$ Laplace random variables, for which we have
		\begin{align*}
			|(I) + (II) + (III)| \leq \sqrt{4C_1\left(2^{p_0 - k + 3}\varepsilon^{-2} n_{\mathrm{P}}\right)\vee \left( \sum_{j'}\sum_{i=1}^{n_{\mathrm{P}}}\tilde{U}_i^{\mathrm{P}j'} \right) + C_1 \lambda^2\sum_{j'}\sum_{i=1}^{n_{\mathrm{P}}}U_i^{\mathrm{Q}j'}}.
		\end{align*}
		The proof is completed by bringing the bound into the decomposition in the beginning.
	\end{proof}

	\subsection{Proof of Theorem \ref{thm:utility3}}
	
	In the proof, depth $p_0$, $p$, $k$ are all integers.
	Without loss of generality, we only write $p_0$, $p$, $k$ equal some order which will only cause difference in constants.

	\begin{proof}\textbf{of Theorem \ref{thm:utility3}}\;\;
		In the following proof, we give several version of upper bounds of $\left|\tilde{\eta}^{prune}_{\mathrm{P}, \mathrm{Q}}(x)  - \mathbb{E}_{Y|X}\left[\widehat{\eta}^{prune}_{\mathrm{P}, \mathrm{Q}}(x) \right]\right| \lesssim b$ for each fixed $x\in \Omega_n$. 
		Then, we consider the region $\Omega_n$ which is 
		\begin{align*}
			\Omega_n := \left\{x \in \mathcal{X} \bigg|  \left|\eta_{\mathrm{P}}^*(x) - \frac{1}{2}\right| \geq c_{\Omega}\cdot b\right\}.
		\end{align*}
		For some $c_{\Omega}$, by techniques analogous to the proof of Theorem \ref{thm:utility}, we can show that all $x$ in $\Omega$ is assigned to the same label as the Bayes classifier.
		For each fixed $x$, let $A^{j_k}$ denote the partition cell in $\pi_k$ that contains $x$. 
		We will discuss the stopping time of the pruning process and the resulting point-wise error. 
		We prove the conclusion under two cases respectively.

		\textbf{Case (i)}: $\left( \frac{n_{\mathrm{P}} \varepsilon^2}{\log (n_{\mathrm{P}} + n_{\mathrm{Q}})}\right) \lesssim \left(\frac{n_{\mathrm{Q}}}{\log (n_{\mathrm{P}} + n_{\mathrm{Q}})}\right)^{\frac{2\alpha + 2d}{2 \gamma\alpha + d}}$. 
		In this case, $\delta^{-1} \lesssim n_{\mathrm{Q}}^{\frac{d}{2\gamma\alpha + d}} \lesssim n_{\mathrm{Q}} \leq \lfloor\frac{d}{2+2 d} \log _2\left(n_{\mathrm{P}} \varepsilon^2+n_{\mathrm{Q}}^{\frac{2+2 d}{d}}\right)\rfloor$.
		Without loss of generality, we let ${\eta}_{\mathrm{P}}^{*}(x)-\frac{1}{2}\geq 0$.
		Following the analysis in Section \ref{sec:derivationpruning}, the final statistics will satisfy
		\begin{align*}
			\frac{|\tilde{\eta}_{\mathrm{P}, \mathrm{Q}}^{\pi_k}(x) - \frac{1}{2}|}{r_k^j}  = & \sqrt{\frac{\left(\sum_{j'}\sum_{i=1}^{n_{\mathrm{P}}}\tilde{U}_i^{\mathrm{P}j'}\right)^2}{u_k} \left(\tilde{\eta}_{\mathrm{P}}^{\pi_k}(x) - \frac{1}{2}\right)^2 + \sum_{j^{\prime}} \sum_{i=1}^{n_{\mathrm{Q}}} U_i^{\mathrm{Q} j^{\prime}} \left(\widehat{\eta}_{\mathrm{Q}}^{\pi_k}(x) - \frac{1}{2}\right)^2 }\\
			\geq & \sqrt{\sum_{j^{\prime}} \sum_{i=1}^{n_{\mathrm{Q}}} U_i^{\mathrm{Q} j^{\prime}} } \left|\widehat{\eta}_{\mathrm{Q}}^{\pi_k}(x) - \frac{1}{2}\right|.
		\end{align*}
		By Proposition \ref{prop:lepskibound}, there exists a constant $\overline{C}_p$ such that 
		\begin{align}
			\left|\widehat{\eta}_{ \mathrm{Q}}^{\pi_k}(x)-\mathbb{E}_{Y|X} \widehat{\eta}_{ \mathrm{Q}}^{\pi_k}(x) \right|
			\leq   \overline{C}_p \sqrt{\frac{\log (n_{\mathrm{P}} + n_{\mathrm{Q}})}{\sum_{j^{\prime}} \sum_{i=1}^{n_{\mathrm{Q}}} U_i^{\mathrm{Q} j^{\prime}} }} \label{equ:lepskibound1}
		\end{align}
		with probability $\mathrm{P}_{Y^{\mathrm{P}} | X^{\mathrm{P}}}^{n_{\mathrm{P}}} \otimes \mathrm{Q}_{Y^{\mathrm{Q}} | X^{\mathrm{Q}}}^{n_{\mathrm{Q}}}\otimes \mathrm{R}$ at least  $1-3/(n_{\mathrm{P}} + n_{\mathrm{Q}})^2$.
		Consider the region
		\begin{align*}
			\Omega_n := \left\{x \in \mathcal{X} \bigg|  \left|\eta_{\mathrm{P}}^*(x) - \frac{1}{2}\right| \geq C_{\Omega}\cdot \delta^{\alpha / d} \right\}.
		\end{align*}
		Let $2^{-k} =  \delta$. 
		For each $x\in\Delta_n$, by Lemma \ref{lem:boundingofapproximationerror}, there exist $\underline{C}_{p1}$, $\underline{C}_{p2}$ such that,
		\begin{align}\label{equ:lepskibound2}
			\min_{x\in A^{j_k}} \mathbb{E}_{Y|X} \widehat{\eta}_{\mathrm{Q}}^{\pi_k}(x) - \frac{1}{2}  \geq  \underline{C}_p^1 \left(\mathbb{E}_{Y|X} \widehat{\eta}_{ \mathrm{Q}}^{\pi_k}(x) - \frac{1}{2} \right)\geq \underline{C}_{p1} (C_{\Omega} - \underline{C}_{p2}) \delta^{\gamma \alpha / d}.
		\end{align}
		Also, by Lemma \ref{lem:cellbound}, there exists constant $\underline{C}_{p3}$ such that
		\begin{align}\label{equ:lepskibound3}
			\sum_{j^{\prime}} \sum_{i=1}^{n_{\mathrm{Q}}} U_i^{\mathrm{Q} j^{\prime}} \geq \underline{C}_{p3} \delta \cdot n_{\mathrm{Q}}
		\end{align}
		with probability $\mathrm{P}_{ X^{\mathrm{P}}}^{n_{\mathrm{P}}} \otimes \mathrm{Q}_{X^{\mathrm{Q}}}^{n_{\mathrm{Q}}}$ at least $1 - 1 / (n_{\mathrm{P}} + n_{\mathrm{Q}})^2$.
		\eqref{equ:lepskibound2} and \eqref{equ:lepskibound3} together lead to 
		\begin{align}\nonumber
			\sqrt{\sum_{j^{\prime}} \sum_{i=1}^{n_{\mathrm{Q}}} U_i^{\mathrm{Q} j^{\prime}}} \left(\mathbb{E}_{Y|X} \widehat{\eta}_{\mathrm{Q}}^{\pi_k}(x) - \frac{1}{2}\right) \geq &  \underline{C}_{p1} \underline{C}_{p3}^{1/2}(C_{\Omega} - \underline{C}_{p2})  \sqrt{n_{\mathrm{Q}}} \cdot \delta^{\frac{2\gamma \alpha + d}{2d}}\\ 
			\geq &  2^{\frac{2\gamma \alpha + d}{2d}} \underline{C}_{p1} \underline{C}_{p3}^{1/2}(C_{\Omega} - \underline{C}_{p2}) \sqrt{\log (n_{\mathrm{P}} + n_{\mathrm{Q}})}.
			\label{equ:lepskibound4}
		\end{align}
		For sufficiently large constant $C_{\Omega}$, \eqref{equ:lepskibound4} yields that the choice of $2^{-k} =  \delta$ leads to 
		\begin{align}\label{equ:lepskibound5}
			\sqrt{\sum_{j^{\prime}} \sum_{i=1}^{n_{\mathrm{Q}}} U_i^{\mathrm{Q} j^{\prime}}} \left(\mathbb{E}_{Y|X} \widehat{\eta}_{\mathrm{Q}}^{\pi_k}(x) - \frac{1}{2}\right) \geq 3 \overline{C}_p\sqrt{\log (n_{\mathrm{P}} + n_{\mathrm{Q}})}. 
		\end{align}
		Combining \eqref{equ:lepskibound1} and \eqref{equ:lepskibound5}, there holds the following 
		\begin{align*}
			& \sqrt{\sum_{j^{\prime}} \sum_{i=1}^{n_{\mathrm{Q}}} U_i^{\mathrm{Q} j^{\prime}} } \left(\widehat{\eta}_{\mathrm{Q}}^{\pi_k}(x) - \frac{1}{2}\right)  \\
			\geq & \sqrt{\sum_{j^{\prime}} \sum_{i=1}^{n_{\mathrm{Q}}} U_i^{\mathrm{Q} j^{\prime}} } \left(\mathbb{E}_{Y|X} \widehat{\eta}_{ \mathrm{Q}}^{\pi_k}(x) - \frac{1}{2} \right) - \sqrt{\sum_{j^{\prime}} \sum_{i=1}^{n_{\mathrm{Q}}} U_i^{\mathrm{Q} j^{\prime}} } \left|\widehat{\eta}_{\mathrm{Q}}^{\pi_k}(x)-\mathbb{E}_{Y|X} \widehat{\eta}_{ \mathrm{Q}}^{\pi_k}(x) \right|\\
			> & \overline{C}_p\sqrt{\log (n_{\mathrm{P}} + n_{\mathrm{Q}})}. 
		\end{align*}
		for the choice of $2^k = \delta^{-1}$.
		Thus, for $x\in \Omega$ and $x\in A^j_{(p_0)}$, the terminating condition will be triggered at least at $2^k = \delta^{-1}$, indicating that $2^{k_j}\geq \delta^{-1}$ with probability $\mathrm{P}^{n_{\mathrm{P}}} \otimes \mathrm{Q}^{n_{\mathrm{Q}}}\otimes \mathrm{R}$ at least $1 - 4 / (n_{\mathrm{P}} + n_{\mathrm{Q}})^2$.
		Thus, in this case, we have
		\begin{align}	\label{equ:comparisonofsignalstrengthandvariance1}
			\mathbb{E}_{Y|X}\left[\widehat{\eta}^{\pi_{k_j}}_{\mathrm{P}, \mathrm{Q}}(x)\right] - \frac{1}{2} =  \mathbb{E}_{Y|X}\left[\widehat{\eta}^{\pi_{k_j}}_{\mathrm{P}, \mathrm{Q}}(x)\right] - \tilde{\eta}^{\pi_{k_j}}_{\mathrm{P}, \mathrm{Q}}(x) + \tilde{\eta}^{\pi_{k_j}}_{\mathrm{P}, \mathrm{Q}}(x) - \frac{1}{2} \geq r_k^j - r_k^j \geq 0, 
		\end{align}
		which implies that the signs of $\mathbb{E}_{Y|X}\left[\widehat{\eta}^{\pi_{k_j}}_{\mathrm{P}, \mathrm{Q}}(x)\right]$ and $ \tilde{\eta}^{\pi_{k_j}}_{\mathrm{P}, \mathrm{Q}}(x)$ are identical. 
		By steps analogous to proof of Proposition \ref{lem:boundingofapproximationerror}, there holds $\mathbb{E}_{Y|X}\left[\widehat{\eta}^{\pi_{k_j}}_{\mathrm{P}, \mathrm{Q}}(x)\right] - \frac{1}{2} > 0$ for $x\in \Omega_n$.
		Thus we show that, for each $x\in \Omega_n$ and $x\in A^j_{(p_0)}$, the label is assigned correctly as the Bayes classifier. 
		Therefore, by Assumption \ref{asp:margin}, we have 
		\begin{align*}
			\mathcal{R}_{\mathrm{P}} (\tilde{f}^{prune}_{\mathrm{P}, \mathrm{Q}}) - \mathcal{R}_{\mathrm{P}}^* = & \int_{\mathcal{X}} \left|\eta^*_{\mathrm{P}}(x) -\frac{1}{2}\right| \eins\left\{\tilde{f}^{\pi}_{\mathrm{P}, \mathrm{Q}}(x) \neq f^*_{\mathrm{P}}(x)\right\} 
			d\mathrm{P}_X(x)
			\\ 
			=& \int_{\Omega_n^c} \left|\eta^*_{\mathrm{P}}(x) -\frac{1}{2}\right| d\mathrm{P}_X(x) \lesssim \delta^{(\beta+1)\alpha /d}. 
		\end{align*}

		\textbf{Case (ii)} $\left( \frac{n_{\mathrm{P}} \varepsilon^2}{\log (n_{\mathrm{P}} + n_{\mathrm{Q}})}\right) \gtrsim \left(\frac{n_{\mathrm{Q}}}{\log (n_{\mathrm{P}} + n_{\mathrm{Q}})}\right)^{\frac{2\alpha + 2d}{2 \gamma\alpha + d}}$. 
		We first note that, with
		$n_{\mathrm{P}}^{-\frac{\alpha}{2\alpha + 2d}}\lesssim \varepsilon \lesssim n_{\mathrm{P}}^{\frac{d}{4 + 2d}}$, there holds $2^{p_0-k+3} \varepsilon^{-2} n_{\mathrm{P}} \geq \sum_{j^{\prime}} \sum_{i=1}^{n_{\mathrm{P}}} \tilde{U}_i^{\mathrm{P} j^{\prime}}$ with high probability, i.e. Algorithm \ref{alg:prunedeta} will determine true in first if block. 
		Since $\varepsilon \lesssim n_{\mathrm{P}}^{\frac{d}{4 + 2d}}$, there holds $2^{p_0} \geq \left(n_{\mathrm{P}}\varepsilon^2\right)^{\frac{d}{2 + 2d}}\gtrsim \varepsilon^2$. 
		By Lemma \ref{lem:cellbound}, when $\lambda = 0$, we have $	\frac{1}{n_{\mathrm{P}}} \sum_{j^{\prime}} \sum_{i=1}^{n_{\mathrm{P}}} {U}_i^{\mathrm{P} j^{\prime}}  \asymp 2^{-k}$. 
		Together, we have
		\begin{align}\label{equ:determinofifblock1}
			2^{p_0-k+3} \varepsilon^{-2} n_{\mathrm{P}} \gtrsim \sum_{j^{\prime}} \sum_{i=1}^{n_{\mathrm{P}}} {U}_i^{\mathrm{P} j^{\prime}}. 
		\end{align}
		For the Laplace noise, there holds
		\begin{align}\label{equ:determinofifblock2}
			2^{p_0-k} \varepsilon^{-2} n_{\mathrm{P}} \gtrsim \sqrt{2^{p_0-k} \varepsilon^{-2} n_{\mathrm{P}} \log (n_{\mathrm{P}} + n_{\mathrm{Q}})}  \gtrsim \frac{4}{\varepsilon}\sum_{j^{\prime}} \sum_{i=1}^{n_{\mathrm{P}}} {\zeta}_i^{\mathrm{P} j^{\prime}}
		\end{align}
		due to Lemma \ref{lem:laplacebound}. 
		Adding \eqref{equ:determinofifblock1} and \eqref{equ:determinofifblock2} shows $2^{p_0-k+3} \varepsilon^{-2} n_{\mathrm{P}} \geq \sum_{j^{\prime}} \sum_{i=1}^{n_{\mathrm{P}}} \tilde{U}_i^{\mathrm{P} j^{\prime}}$ holds with probability $1 - 2 / (n_{\mathrm{P}} + n_{\mathrm{Q}})^2$.
		Under this case, if $k_j > \lfloor \frac{d}{2 + 2d} \log_2 (n_{\mathrm{P}} \varepsilon^2 ) \rfloor $ for all $j\in\mathcal{I}_{p_0}$, we end up with $\tilde{\eta}_{\mathrm{P}, \mathrm{Q}}^{prune} = \sum_j s_k^j \cdot \eins\{A_{(p_0)}^j\}$. 
		We define 
		\begin{align*}
			\Omega_n:=\left\{x \in \mathcal{X}|| \eta_{\mathrm{P}}^*(x)-\frac{1}{2} | \geq C_{\Omega} \cdot \delta^{\alpha (\alpha + d) / d (1 + d)}\right\}
		\end{align*}
		for some $C_{\Omega}$ large enough. 
		Assume $ \eta_{\mathrm{P}}^*(x) > 1 / 2$. 
		If $x'$ belongs to the same $A^{k_j}$ as $x$, then 
		\begin{align*}
			\eta_{\mathrm{P}}^*(x') - \frac{1}{2} \geq  \eta_{\mathrm{P}}^*(x)- \frac{1}{2} - C' 2^{- k_j \alpha / d}  \geq \eta_{\mathrm{P}}^*(x)- \frac{1}{2} - C' \delta^{\alpha (\alpha + d) / d (1 + d)} \geq 0
		\end{align*}
		for some $C'$. 
		Analogous to \eqref{equ:comparisonofsignalstrengthandvariance1}, all $x\in\Omega_n$ are correctly classified and thus 
		\begin{align*}
			\mathcal{R}_{\mathrm{P}} (\tilde{f}^{prune}_{\mathrm{P}, \mathrm{Q}}) - \mathcal{R}_{\mathrm{P}}^* = & \int_{\mathcal{X}} \left|\eta^*_{\mathrm{P}}(x) -\frac{1}{2}\right| \eins\left\{\tilde{f}^{\pi}_{\mathrm{P}, \mathrm{Q}}(x) \neq f^*_{\mathrm{P}}(x)\right\} 
			d\mathrm{P}_X(x)
			\\ 
			=& \int_{\Omega_n^c} \left|\eta^*_{\mathrm{P}}(x) -\frac{1}{2}\right| d\mathrm{P}_X(x) \lesssim \delta^{(\beta+1)\alpha (\alpha + d) /d ( 1+ d)}. 
		\end{align*}
		If there exists some $j \in \mathcal{I}_{p_0}$ such that the termination is triggered, i.e. we will end up with $\tilde{f}_{\mathrm{P}}^{\pi_p}$, the proof becomes analogous to that of Theorem \ref{thm:utility}. 
		Specifically, the trade off \eqref{equ:signofetadpwithcomega} will have a larger approximation error $ \delta^{\alpha (\alpha + d) / d (1 + d)}$ and smaller sample error as well as privacy error. 
		Thus, the final rate is dominated by the approximation error, which becomes $\delta^{(\beta+1)\alpha (\alpha + d) /d ( 1+ d)}$. 
		In Case (ii), the overall incidents happen with probability $\mathrm{P}^{n_{\mathrm{P}}} \otimes \mathrm{Q}^{n_{\mathrm{Q}}}\otimes \mathrm{R}$ at least $1 - 4 / (n_{\mathrm{P}} + n_{\mathrm{Q}})^2$.

	\end{proof}

	\section{Experiment Details}\label{app:experimentdetails}

	\subsection{Real Datasets Pre-processing}
	\label{sec:realdatasetpreprocess}
	
	We present additional information on the data sets including data source and pre-processing details.
	The data sets are sourced from \texttt{UCI Machine Learning Repository} \citep{Dua:2019} and \texttt{Kaggle}. 
	
	\texttt{Anonymity}: 
	The Anonymity data set \citep{anonymity} contains poll results about Internet privacy from a large number of people, including answers about whether the interviewee believes United States law provides reasonable privacy protection for Internet users or their conservativeness about Internet privacy. 
	We use other variables to predict whether the interviewee has ever tried to mask his/her identity when using the Internet. 
	We randomly select 50 samples as the public data. 
	
	\texttt{Diabetes}:
	This data set from \cite{diabetes} contains a diverse range of health-related attributes, meticulously collected to aid in the development of predictive models for identifying individuals at risk of diabetes.
	We randomly select 200 samples as the public data. 
	
	\texttt{Email}: 
	The data set is used for spam or not-spam classification \citep{email}.
	For each row, the count of each word (column) in that email (row) is stored in the respective cells. Thus, information regarding all emails is stored in a compact data frame.
	We randomly select 200 samples as the public data.

	\texttt{Employ}:
	The data set \citep{employ} was collected from different agencies in the Philippines consisting of mock interview performances of students. 
	The final task is to decide whether the student got employed. 
	We randomly select 200 samples as the public data.

	\texttt{Rice}:
	The Rice data set \citep{ricedata} contains morphological features obtained for two grains of rice to classify them. 
	We randomly select 300 samples as the public data.

	\texttt{Census}:
	This data set is the well-known census data \citep{census} and is used to predict whether income exceeds 50,000 per year. 
	We select the samples whose native country is the U.S. to be public data and others to be private data.

	\texttt{Election}:
	This data set from \citep{election} contains county-level demographics and whether or not Bill Clinton won each county during the 1992 U.S. presidential election. 
	The goal of this data set is to predict if Clinton won a county.
	We use samples whose state is TX as public data and the rest as private data. 
	
	\texttt{Employee}:
	The Employee data \citep{employee} contains information about employees in a company, including their educational backgrounds, work history, and employment-related factors. 
	We use these variables to predict whether the employee leaves or not.
	We select the sample to be public if the employee joined the company before 2012.

	\texttt{Jobs}:
	The data \citep{jobs}contains a comprehensive collection of information regarding job applicants and their respective employability scores to predict whether the applicant has been hired.
	We select the samples whose native country is the U.S. to be public data and others to be private data. 
	
	\texttt{Landcover}:
	The data is from \cite{johnson2016integrating}, which contains Landsat time-series satellite imagery information on given pixels and their corresponding land cover class labels (farm, forest, water, etc.) obtained from multiple sources. We classify whether the image comes from a farm or a forest.
	Within this data set, there are two kinds of label sources: labels obtained from Open-StreetMap and labels that are accurately labeled by experts.
	We take the latter as public data and the former as private data.

	\texttt{Taxi}:
	The Taxi data set was obtained from the Differential Privacy Temporal Map Challenge \citep{taxi}, which aims to develop algorithms that preserve data utility while guaranteeing individual privacy protection. 
	The data set contains quantitative and categorical information about taxi trips in Chicago, including time, distance, location, payment, and service provider. These features include the identification number of taxis ($taxi\_id$), time of each trip ($seconds$), the distance of each trip ($miles$), index of the zone where the trip starts ($pca$), index of the zone where the trip ends ($dca$), service provider ($company$), the method used to pay for the trip ($payment\_type$), amount of tips ($tips$) and fares ($fare$).
	We use the other variable to predict whether the passenger is paying by credit card or by cash. 
	We select data from one company as public data and others as private data. 
	

	\bibliography{ref}

\end{document}